\documentclass[11pt]{article}
\usepackage[preprint]{acl}
\usepackage{times}
\usepackage{latexsym}
\usepackage[T1]{fontenc}
\usepackage[utf8]{inputenc}
\usepackage{microtype}
\usepackage{inconsolata}
\usepackage{graphicx}
\usepackage{booktabs}
\usepackage{array}
\usepackage{amsmath}
\usepackage{amssymb}
\usepackage{amsthm}
\usepackage{bm}
\usepackage{multirow}
\usepackage{enumitem}
\usepackage{float}
\usepackage{tikz}
\usetikzlibrary{arrows.meta,positioning,fit,calc,patterns,decorations.pathreplacing}
\usepackage{pgfplots}
\usepgfplotslibrary{groupplots}
\pgfplotsset{compat=1.17}

\makeatletter\g@addto@macro\UrlBreaks{\do\-}\makeatother

\newcommand{\fromcode}[1]{\textcolor{green!45!black}{#1}}
\newcommand{\fitted}[1]{\textcolor{red!75!black}{#1}}
\newfloat{algorithm}{tbp}{loa}
\floatname{algorithm}{Algorithm}

\theoremstyle{plain}
\newtheorem{theorem}{Theorem}
\newtheorem{lemma}{Lemma}
\theoremstyle{remark}
\newtheorem*{remark}{Remark}

\newcommand{\gib}{\,GiB}
\newcommand{\bytes}{\,\mathrm{bytes}}

\newcommand{\system}{\mbox{DisaggKD}}
\newcommand{\sqlk}{\mbox{SQ-L+K}}
\newcommand{\ballk}{\mbox{Bal-L+K}}
\newcommand{\dt}{d_{t}}
\newcommand{\ds}{d_{s}}
\newcommand{\Vloc}{V_{\mathrm{loc}}}
\newcommand{\ind}[1]{\hspace{1.6em}\parbox[t]{5.0cm}{\raggedright #1}}
\newcolumntype{L}[1]{>{\raggedright\arraybackslash}p{#1}}

\title{DisaggKD: Efficient Knowledge Distillation through Teacher–Student Disaggregation and a Fused Loss}

\author{Ruixuan Tu\thanks{Work was done during the internship at SB Intuitions Corp.} \\
  The University of Tokyo, RIKEN AIP \\
  \texttt{turx2003@gmail.com}
  \And
  Yuyang Dong \\
  SB Intuitions Corp. \\
  \texttt{yuyang.dong@sbintuitions.co.jp}}

\begin{document}
\maketitle

\begin{abstract}
Knowledge distillation (KD) frameworks place teacher and student on the same GPUs, so per-GPU memory limits the teachers that fit. \textbf{\system{}} separates them across nodes: teacher replicas stream hidden states to the student, a KD loss over row windows and vocabulary shards equals the full-vocabulary loss in exact arithmetic without forming the full-vocabulary logits, and a prefetch ring paces producer and trainer. Closed-form memory and traffic models locate the memory limits of each layout. On the Qwen3.5-35B-A3B to 9B pair, the disaggregated layout matches the LM loss of the colocated run with the same loss and is faster than that run; our fused loss also lets a colocated run train on as few nodes at faster speed, while the split is memory efficient when the student spans several nodes. On the larger pairs, colocated runs with a fused loss still run OOM where the split trains, because teacher weights and student states fill the budget. Code will be released upon acceptance.
\end{abstract}

\section{Introduction}

Knowledge distillation (KD) trains a student on its teacher's output distribution \citep{hinton2015distilling}, and at pretraining scale SlimQwen uses it to recover the quality of compressed students \citep{tang2026slimqwen}.

The 3 open frameworks we ran place the teacher on the student's GPUs, keeping its weights beside the student's weights, gradients, and optimizer state \citep{megatronbridge,modelopt} or offloading it between phases \citep{nemorl,zhang2026kdflow}, as KDFlow optionally does. Colocated, the 2 models share 1 per-GPU memory budget, and the KD loss adds a logits tensor over the full vocabulary for every token the student processes. That budget is the limit: with ModelOpt's loss the colocated stack on the Qwen3.5 pair at 3 nodes runs out of memory (OOM) where the separated one fits (Figure~\ref{fig:overview}), and our fused loss moves it but not the weights of larger teachers (Table~\ref{tab:pairs}).

We separate the 2 roles: the teacher runs forward-only on its own nodes, so no GPU holds the layers of both models. The teacher sends its final hidden states, which are narrower than its vocabulary-wide logits, and the student recovers the logits with a copy of the teacher's output head. A fused kernel then evaluates the KD loss over vocabulary shards from per-row statistics that combine across tensor-parallel (TP) ranks, without forming a full-vocabulary tensor, so trainer memory stays low.

Whether the split pays depends on per-GPU memory and on NVLink and InfiniBand traffic, so we derive closed-form estimates of both.

Our contributions are:
\begin{itemize}[itemsep=2pt, topsep=3pt, leftmargin=*]
\item \textbf{Estimation models} (\S\ref{sec:estimation}) of per-GPU memory (\S\ref{sec:memory}) and of traffic inside the student, inside the teacher, and between them (\S\ref{sec:movement}), for both layouts (\S\ref{sec:predictions}); the memory model is fitted on measured peaks and OOM events, the traffic model checked against node counters (Appendices~\ref{app:memory},~\ref{app:traffic}).
\item \textbf{A disaggregated KD algorithm}, \system{} (\S\ref{sec:algorithm}): teacher replicas stream hidden states over NVIDIA Collective Communications Library (NCCL) point-to-point (P2P) links (\S\ref{sec:roles}), the student evaluates a fused loss over row windows and vocabulary shards, equal to the full-vocabulary loss in exact arithmetic (\S\ref{sec:loss}, Appendix~\ref{app:loss-equiv}), and a prefetch ring paces both (\S\ref{sec:pacing}).
\item \textbf{A measured evaluation} (\S\ref{sec:experiments}). On the Qwen3.5 pair the language-modeling (LM) losses of the 2 layouts match, the disaggregated run is faster than the colocated runs with ModelOpt's loss balancer and with the same loss, and a colocated run with our fused loss is as fast or faster on as few nodes but holds more memory per GPU when the student spans several nodes (\S\ref{sec:results}); on 3 larger pairs the disaggregated layout trains where the colocated stack runs OOM, also with a fused loss, and is faster on Llama.
\end{itemize}

\begin{figure}[t]
\centering
\begin{tikzpicture}[>={Stealth[length=5pt]}]
\begin{axis}[
  width=6.0cm, height=4.1cm, scale only axis, font=\small,
  ybar stacked, bar width=0.8cm, ymin=0, ymax=192,
  ytick={0,79.11,160}, yticklabels={0,79,160},
  ylabel={GiB per GPU},
  axis line style={line width=0.5pt}, tick style={line width=0.5pt},
  ymajorgrids, grid style={gray!25, line width=0.3pt},
  symbolic x coords={col3,colf,ours_pr,ours_tr}, xtick=data,
  xticklabels={ModelOpt, fused-W, producer, trainer},
  enlarge x limits=0.15, clip=false,
  legend style={at={(0.5,1.04)}, anchor=south, draw=none, inner sep=1pt, column sep=4pt,
    legend columns=3, legend cell align=left},
  legend image post style={scale=0.9},
]
\addplot[fill=blue!35, draw=black!50, line width=0.3pt] table[x=cell, y=S, col sep=comma] {figs/mem_overview.csv};
\addplot[fill=cyan!25, draw=black!50, line width=0.3pt] table[x=cell, y=A, col sep=comma] {figs/mem_overview.csv};
\addplot[fill=orange!45, draw=black!50, line width=0.3pt] table[x=cell, y=T, col sep=comma] {figs/mem_overview.csv};
\addplot[fill=red!40, draw=black!50, line width=0.3pt] table[x=cell, y=K, col sep=comma] {figs/mem_overview.csv};
\addplot[fill=gray!35, draw=black!50, line width=0.3pt] table[x=cell, y=O, col sep=comma] {figs/mem_overview.csv};
\addplot+[stack plots=false, only marks, mark=diamond*, mark size=3pt, black, mark options={fill=white, draw=black, line width=0.7pt}]
  coordinates {(col3,88.0) (colf,78.9) (ours_pr,39.4) (ours_tr,57.1)};
\legend{$S$, $A$, teacher, $\kappa$, other, measured}
\draw[dashed, line width=0.8pt, black!85] (rel axis cs:0,0.4120) -- (rel axis cs:1,0.4120);
\node[anchor=south east, black!85] at (rel axis cs:1,0.4120) {$M$};
\draw[->, line width=1.2pt] ($(axis cs:ours_pr,110)+(0.2cm,0cm)$) -- ($(axis cs:ours_tr,110)-(0.2cm,0cm)$);
\node[anchor=south] at ($(axis cs:ours_pr,110)!0.5!(axis cs:ours_tr,110)$) {$\bm{H}_t^{(q)}$};
\draw[black!60, line width=0.5pt] ($(axis cs:col3,0)+(-0.5cm,-0.88cm)$) -- ($(axis cs:colf,0)+(0.5cm,-0.88cm)$);
\node[anchor=north] at ($(axis cs:col3,0)!0.5!(axis cs:colf,0)+(0cm,-0.92cm)$) {colocated};
\draw[black!60, line width=0.5pt] ($(axis cs:ours_pr,0)+(-0.5cm,-0.88cm)$) -- ($(axis cs:ours_tr,0)+(0.5cm,-0.88cm)$);
\node[anchor=north] at ($(axis cs:ours_pr,0)!0.5!(axis cs:ours_tr,0)+(0cm,-0.92cm)$) {disaggregated};
\end{axis}
\end{tikzpicture}
\caption{Per-GPU memory of the Qwen3.5-35B-A3B to 9B pair on 3 nodes: bars stack student states $S$, activations and buffers $A$, teacher weights, logits buffers (colocated) or the memory $\kappa$ retained by the streamed loss (trainer), and other memory; diamonds are measured peaks (ModelOpt bar: allocated plus requested bytes at the OOM error). The first colocated bar (ModelOpt's loss) is the model's demand; the fused-W bar (Table~\ref{tab:pairs}; fused loss, model cross-entropy fusion off) splits a measured peak, with $\kappa$ the remainder. Trainer and producer nodes hold separate stacks and exchange hidden-state slots $\bm{H}_t^{(q)}$ over NCCL P2P.}
\label{fig:overview}
\end{figure}
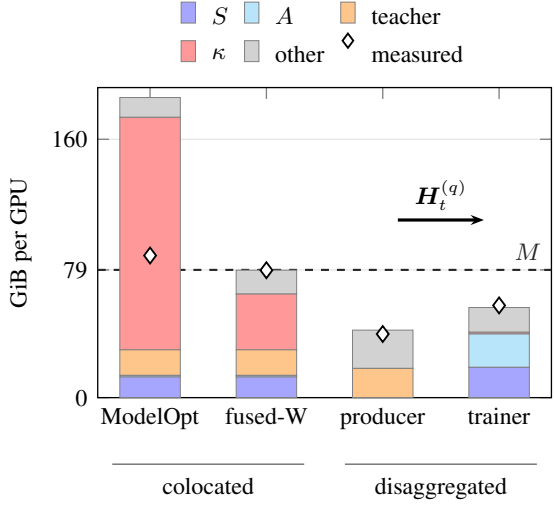

\section{Related Work}
\label{sec:related}

\paragraph{KD for language models.}
At LLM scale the usual objective is the forward KL to the teacher's next-token distribution, used alone or with hidden-state losses in compression pipelines and model releases \citep{muralidharan2024compact,bercovich2025puzzle,gemmateam2024gemma2,nvidia2025nemotronnano2,liu2026ministral3} or studied as a mixture with the LM loss \citep{tang2026slimqwen,peng2025pretraining,busbridge2025distillation,lu2026strong}. Caching the teacher's top-$K$ logits stores $K$ values and indices per token \citep{ryskulov2026topk}; streaming avoids it.

\paragraph{KD frameworks.}
Table~\ref{tab:systems} compares the open frameworks we ran against \system{}; all 3 keep teacher and student on the same GPUs in the modes we ran: Megatron-Bridge ModelOpt \citep{megatronbridge,modelopt}, KDFlow \citep{zhang2026kdflow}, and NeMo-RL \citep{nemorl}. Like KDFlow and Maestro \citep{yuan2026maestro}, we send hidden states instead of logits, but stream the states from teacher replicas by NCCL P2P to every tensor-parallel rank of the student under a prefetch ring, with a vocabulary-sharded loss and memory and traffic models for both placements (\S\ref{sec:estimation}, \S\ref{sec:algorithm}).

\paragraph{Disaggregation and fused losses.}
Separating phases that contend for one device has paid off in LLM serving, with prefill and decode on separate pools \citep{zhong2024distserve,patel2024splitwise,qin2025mooncake}. RLHF systems place the actor, reference, and reward models flexibly \citep{sheng2025hybridflow} and search plans with a cost estimator \citep{mei2025real}; we separate the teacher of KD. On the student side, chunked and fused losses avoid materializing the full logits tensor in cross-entropy \citep{wijmans2025cutlosses,hsu2025liger} and in offline KD \citep{ryskulov2026topk}.

\section{Estimation Models}
\label{sec:estimation}

Whether a teacher--student pair can be distilled on a given set of nodes depends on 3 per-GPU quantities: peak memory, bytes moved over the intra-node fabric (NVLink), and bytes moved over the inter-node fabric (InfiniBand). We model these quantities for each layout at its optimal setting.

\paragraph{Notation.}
The teacher has $D_t$ dense and $E_t$ routed-expert parameters ($P_t=D_t{+}E_t$), $\ell_t$ layers, and hidden width $\dt$. The student has $D_s$, $E_s$, $\ell_s$, and $\ds$. Both share the vocabulary size $V$. A layout assigns the student tensor (TP), pipeline (PP), context (CP), data (DP), and expert (EP) parallel degrees $(\mathrm{tp},\mathrm{pp},\mathrm{cp},\mathrm{dp},\mathrm{ep})$ on $N$ nodes of 8 GPUs, each with budget $M{=}79.11$\gib{}. A sequence has $L$ tokens and an iteration processes $\mathrm{gbs}$ sequences, so $N_{\mathrm{tok}}=\mathrm{gbs}\cdot L$, each GPU handles $\tau=N_{\mathrm{tok}}/(\mathrm{dp}\,\mathrm{cp})$ tokens, and each DP replica runs $n_{\mathrm{mb}}=\mathrm{gbs}/\mathrm{dp}$ microbatches of one sequence. Activations take $b{=}2$ bytes per element, so $x$ elements occupy $x\,b\bytes$, and $\Vloc=V/\mathrm{tp}$ is the vocabulary shard of one TP rank (Appendix~\ref{app:glossary}).

\paragraph{Optimal setting.}
For a layout, let $\mathcal{G}$ be its legal parallelism grids (degrees divide layer, head, expert, and sequence counts and the world size). The optimal setting is the node-minimal grid $g^\star=\arg\min_{g\in\mathcal{G}} N(g)$ whose memory constraint holds, ties broken toward the smaller per-GPU peak. In the colocated layout the constraint couples both models to the student grid; in the disaggregated layout it splits into trainer and producer inequalities, and $R$ follows from the rate constraint.

\subsection{Memory}
\label{sec:memory}

A student rank stores bf16 weights, fp32 gradients, and a distributed AdamW state, which is sharded over the DP$\times$CP group of $\mathrm{dp}\cdot\mathrm{cp}$ ranks for dense parameters and over the expert data-parallel group of $\mathrm{edp}=\mathrm{dp}\,\mathrm{cp}\,\mathrm{tp}/\mathrm{ep}$ ranks for routed experts:
\begin{equation}
\label{eq:statics}
\begin{aligned}
S(g)&=(\fromcode{2}{+}\fromcode{4})\,P_{\mathrm{loc}}\bytes\\
&\quad+\fromcode{12}\Bigl(\tfrac{D_s}{\mathrm{tp}\,\mathrm{pp}\,\mathrm{dp}\,\mathrm{cp}}
 +\tfrac{E_s}{\mathrm{pp}\,\mathrm{ep}\,\mathrm{edp}}\Bigr)\bytes,\\[3pt]
P_{\mathrm{loc}}&=\tfrac{D_s}{\mathrm{tp}\,\mathrm{pp}}+\tfrac{E_s}{\mathrm{pp}\,\mathrm{ep}},\quad
\mathrm{edp}=\tfrac{\mathrm{dp}\,\mathrm{cp}\,\mathrm{tp}}{\mathrm{ep}}
\end{aligned}
\end{equation}
where $P_{\mathrm{loc}}$ counts the parameters of one rank (for the trainer, of the last pipeline stage, which holds the loss; the formula gives the stage average and Appendix~\ref{app:memory} adds the output layer $V\ds/\mathrm{tp}$) and routed experts are sharded over $\mathrm{pp}\cdot\mathrm{ep}$ only; coefficients are \fromcode{green} when fixed by the code, \fitted{red} when fitted to measurements, and black when set as an allowance (Table~\ref{tab:modules}).

\paragraph{Colocated.}
The colocated teacher is built at the student's grid in bf16, so its shard is $T_{\mathrm{col}}=\fromcode{2}\bigl(\tfrac{D_t}{\mathrm{pp}\,\mathrm{tp}}+\tfrac{E_t}{\mathrm{pp}\,\mathrm{ep}}\bigr)\bytes$ per rank. The ModelOpt logits-KL loss runs eagerly at the last pipeline stage. At its high-water mark the loss holds 7 fp32-equivalent logits-sized buffers $\in\mathbb{R}^{L/\mathrm{cp}\times V/\mathrm{tp}}$ (enumerated in Appendix~\ref{app:memory}), 1 for the loss-averaging all-reduce, and 1.5 for the backward pass:
\begin{equation}
\label{eq:colbuf}
\kappa_{\mathrm{col}}(g)=(\fitted{7}{+}1{+}\fitted{1.5})\cdot\frac{L}{\mathrm{cp}}\cdot\frac{V}{\mathrm{tp}}\cdot \fromcode{4}\bytes.
\end{equation}
The peak at the last pipeline rank is
\begin{equation}
\label{eq:colocate}
M_{\mathrm{col}}(g)=S+T_{\mathrm{col}}+\kappa_{\mathrm{col}}+A_{\mathrm{col}}+\fitted{\sigma_{\mathrm{col}}}\le M,
\end{equation}
where $A_{\mathrm{col}}=\fromcode{1}\tfrac{\ell_s}{\mathrm{pp}}\tfrac{L}{\mathrm{cp}\,\mathrm{tp}}\ds\, b\bytes$ holds the stored layer inputs of the 1 microbatch in flight under full activation recomputation (mixture-of-experts (MoE) students add the dispatch term of Table~\ref{tab:modules}), and $\fitted{\sigma_{\mathrm{col}}}{=}12.27$\gib{} is the device peak minus the torch-allocated peak, the mean of the 2 cp4 runs (Table~\ref{tab:modules}). The peak is the larger of this loss point and the teacher's forward-pass point. For the Qwen3.5 pair the logits buffers decide the grid: at tp 2 the loss alone needs 144\gib{} at $\mathrm{cp}{=}1$, so $\mathrm{cp}>1$ is required. The \ballk{} recipe (ModelOpt's balancer) holds 1 more.

\paragraph{Disaggregated.}
The 2 roles have separate budgets. The trainer's peak is on the loss stage, which holds the student stack \eqref{eq:statics}, the streamed-loss retention $\kappa_{\mathrm{str}}$ of \S\ref{sec:loss}, and backward-pass memory:
\begin{equation}
\label{eq:trainer}
M_{\mathrm{tr}}(g)=S+\kappa_{\mathrm{str}}+A_{\mathrm{str}}+\fitted{\sigma_{\mathrm{str}}}\le M,
\end{equation}
with $\kappa_{\mathrm{str}}$ of \eqref{eq:kstr}, which holds no full-row logits, the activation and temporary memory $A_{\mathrm{str}}$ (reference-cycle garbage, loss-backward and recompute buffers, and an expert-layer block), and the fitted remainder $\fitted{\sigma_{\mathrm{str}}}$, the mean of 8 trainers (Table~\ref{tab:modules}). The producer holds the teacher's checkpoint as a shard (routed experts in fp8 or bf16, the rest bf16 but Nemotron's shared experts; Appendix~\ref{app:fp8}), a fitted remainder $\varphi_p$ (forward-pass workspace, CUDA context, allocator), and NCCL buffers $\omega_p$:
\begin{equation}
\label{eq:producer}
\begin{aligned}
M_{\mathrm{pr}}(\mathrm{tp}_p,\mathrm{ep}_p)={}&\Bigl(\fromcode{2}\tfrac{D_t}{\mathrm{tp}_p}+\beta_e\tfrac{E_t}{\mathrm{tp}_p\,\mathrm{ep}_p}\Bigr)\bytes\\
&+\fitted{\varphi_p}+\omega_p\le M,
\end{aligned}
\end{equation}
where $\beta_e$ is 1 for fp8 routed experts and 2 for bf16 and $\fitted{\varphi_p}$ is calibrated (Table~\ref{tab:modules}). The producers keep pace with the trainer when $R\,\rho_{\mathrm{rep}}\ge N_{\mathrm{tok}}/T_{\mathrm{iter}}$ ($\rho_{\mathrm{rep}}$: replica forward rate, $T_{\mathrm{iter}}$: iteration time).

\begin{table*}[t]
\centering
\footnotesize
\setlength{\tabcolsep}{4pt}
\renewcommand{\arraystretch}{1.05}
\begin{tabular}{L{2.1cm}L{6.6cm}L{6.2cm}}
\toprule
 & Colocated (4 nodes, tp2/pp4/cp4/dp1) & Disaggregated (3 nodes: 2 trainer + 1 producer) \\
\midrule
\textbf{Memory}\newline per GPU (GiB) &
$S{+}T_{\mathrm{col}}{+}\kappa_{\mathrm{col}}{+}A_{\mathrm{col}}{+}\sigma_{\mathrm{col}}\le M$ \newline
11.5 + 15.8 + 36.0 + 0.25 + 12.3 $\approx$ predicted 75.8 (fit), measured 76.1 &
trainer $S{+}\kappa_{\mathrm{str}}{+}A_{\mathrm{str}}{+}\sigma_{\mathrm{str}}\le M$: predicted 55.8 (fit), measured 57.1 \newline
producer $\bigl(2\tfrac{D_t}{\mathrm{tp}_p}{+}\tfrac{E_t}{\mathrm{tp}_p}\bigr){+}\varphi_p{+}\omega_p\le M$: predicted 41.8 (fit), measured 39.4 \\
\bottomrule
\end{tabular}

\vspace{1pt}
\begin{tabular}{L{2.9cm}L{4.1cm}L{2.7cm}L{4.7cm}}
\toprule
GB per GPU (s at the limit) & Expression & Colocated & Disaggregated \\
\midrule
Student, NVLink & $\Phi^{\mathrm{tp}}_s+\Phi^{\mathrm{cp}}_s+\Phi^{\mathrm{dp}}_s$ & 701.7 (1.56\,s) & 1179.5 (2.62\,s) \\
Teacher, NVLink & $\Phi^{\mathrm{tp}}_t+\Phi^{\mathrm{cp}}_t+\Phi^{\mathrm{ep}}_t$ & 191.0 (0.42\,s) & 2405 (5.34\,s), producer \\
Student, InfiniBand & $\Phi^{\mathrm{pp}}_s$ per direction & 4.29 (0.09\,s) & 4.29 (0.09\,s) \\
Teacher, InfiniBand & $\Phi^{\mathrm{pp}}_t$ per direction & 2.15 (0.04\,s) & 0 \\
Transfer, InfiniBand & $\Phi^{\mathrm{xfer}}$ per last-stage rank & 0 & 4.29 (0.09\,s), 34.4 in total \\
NVLink, measured & all dimensions & 932.1 (2.07\,s) & 1360 (3.02\,s), trainer; 2485 (5.52\,s), producer \\
Last-stage node, Gbps received & & 4.56 calculated, 4.61 measured & 7.23 calculated, 7.35 measured \\
\bottomrule
\end{tabular}
\caption{Estimation models per layout at its optimal setting for the Qwen3.5-35B-A3B to 9B pair ($L{=}32768$, $\mathrm{gbs}{=}128$; other dimensions in Appendix~\ref{app:traffic}). Memory rows: predictions against measured device peaks (GiB). Traffic rows (\eqref{eq:move}, \eqref{eq:xfer}): GB sent by a GPU per iteration and, in parentheses, the time at the nominal limits (NVLink 450\,GB/s per GPU and direction; InfiniBand 400\,Gb/s per GPU); measured NVLink: counter bytes per GPU, iterations 3 to 12 (colocated) and 3 to 10 (disaggregated).}
\label{tab:estimates}
\end{table*}

\subsection{Data movement}
\label{sec:movement}

A GPU moves bytes for 3 reasons: the parallelism of the student, that of the teacher's forward pass (on the student's grid when colocated, on the producer's span when disaggregated), and the teacher-to-student transfer. For a model $x\in\{t,s\}$ with hidden width $d_x$, TP all-reduces \citep{shoeybi2019megatron}, pipeline activation transfers, and the student's gradient and parameter collectives send per GPU and iteration
\begin{equation}
\label{eq:move}
\begin{aligned}
\Phi^{\mathrm{tp}}_x&=\nu_x\,\tfrac{\ell_x}{\mathrm{pp}}\,\tfrac{2(\mathrm{tp}-1)}{\mathrm{tp}}\,\tau\, d_x b\bytes,\\
\Phi^{\mathrm{pp}}_x&=n_{\mathrm{mb}}\,\tfrac{L}{\mathrm{cp}}\,\tfrac{d_x b}{\mathrm{tp}}\bytes,\\
\Phi^{\mathrm{dp}}_s&=\tfrac{\mathrm{dp}\,\mathrm{cp}-1}{\mathrm{dp}\,\mathrm{cp}}\,P_{\mathrm{loc}}\,(4{+}2)\bytes,
\end{aligned}
\end{equation}
where $\nu_t{=}2$ and $\nu_s{=}6$ are 2 all-reduces per layer times the teacher's 1 forward pass and the student's 3 passes (forward, recompute, backward), and $\Phi^{\mathrm{pp}}_x$ is per stage boundary and direction. CP exchanges and EP all-to-all are in Table~\ref{tab:traffic}. Megatron's rank order (tp, cp, ep, dp, pp) with 8 GPUs per node decides the link: a dimension of degree $n$ whose ranks are $\Delta$ apart spans $n\Delta$ ranks, so the dimension stays on NVLink if $n\Delta\le8$ (one node) and otherwise uses InfiniBand at its node boundaries, so each dimension that spans nodes adds inter-node traffic. The transfer sends each slot once to every last-stage rank of the receiving group,
\begin{equation}
\label{eq:xfer}
\Phi^{\mathrm{xfer}}=n_{\mathrm{mb}}\,\tfrac{L}{\mathrm{cp}}\,\dt\, b\bytes
\end{equation}
per last-stage rank, or $\mathrm{gbs}\cdot\mathrm{tp}\cdot L\,\dt\, b\bytes$ in total; the lm\_head ($\dt\, V\, b\bytes$) is written once to a file.

\subsection{Predictions at the optimal settings}
\label{sec:predictions}

Table~\ref{tab:estimates} evaluates the models for the Qwen3.5-35B-A3B to 9B pair. At 4 nodes 2 grids with tp$\le$2 fit the budget and the colocated optimum has the lowest predicted peak; at 3 nodes no tp$\le$2 grid is predicted under the budget for ModelOpt's loss. The disaggregated configuration uses the (2+1) split of Table~\ref{tab:runs}, the most GPU-efficient split we ran (Table~\ref{tab:gpusec}), and both roles fit.

With ModelOpt's loss the colocated layout is capacity-bound: its logits buffers force $\mathrm{cp}\ge4$ and hence $\mathrm{dp}{=}1$ on 4 nodes at tp$\le$2. On the Qwen3.5 pair neither layout is bandwidth-bound: at the link limits the busiest GPU needs under 6 seconds of NVLink and 0.2 seconds of InfiniBand per iteration, and a hidden-state row carries $2\dt$ bytes where a logits row carries $2V$, 8$\times$ fewer for the Llama pair and 121$\times$ fewer for the Qwen3.5 pair.

\paragraph{Validation.}
Table~\ref{tab:validation} compares the models with measurements on the Qwen3.5 pair; Table~\ref{tab:accuracy} summarizes the memory accuracy on all 5 pairs and the traffic accuracy on the Qwen3.5 pair (Appendix~\ref{app:memory}). The colocated model classifies 17 of the 18 configurations outside its fit and 5 of the 6 that set constants correctly, and predicts the peaks of 9 runs outside its fit within 3.0\gib{}; the 2 errors are the Qwen3.5 \ballk{} fit row, at the budget, and the Nemotron run on 4 nodes (OOM in the teacher's forward pass). The trainer estimate is within 4.0\gib{} on 13 runs of the fitted families and the producer estimate within 8.4\gib{} on its 8 fitting runs. Inter-node traffic is within 2\% of the counters (Table~\ref{tab:nodes}), NVLink within 15\% (Table~\ref{tab:accuracy}).

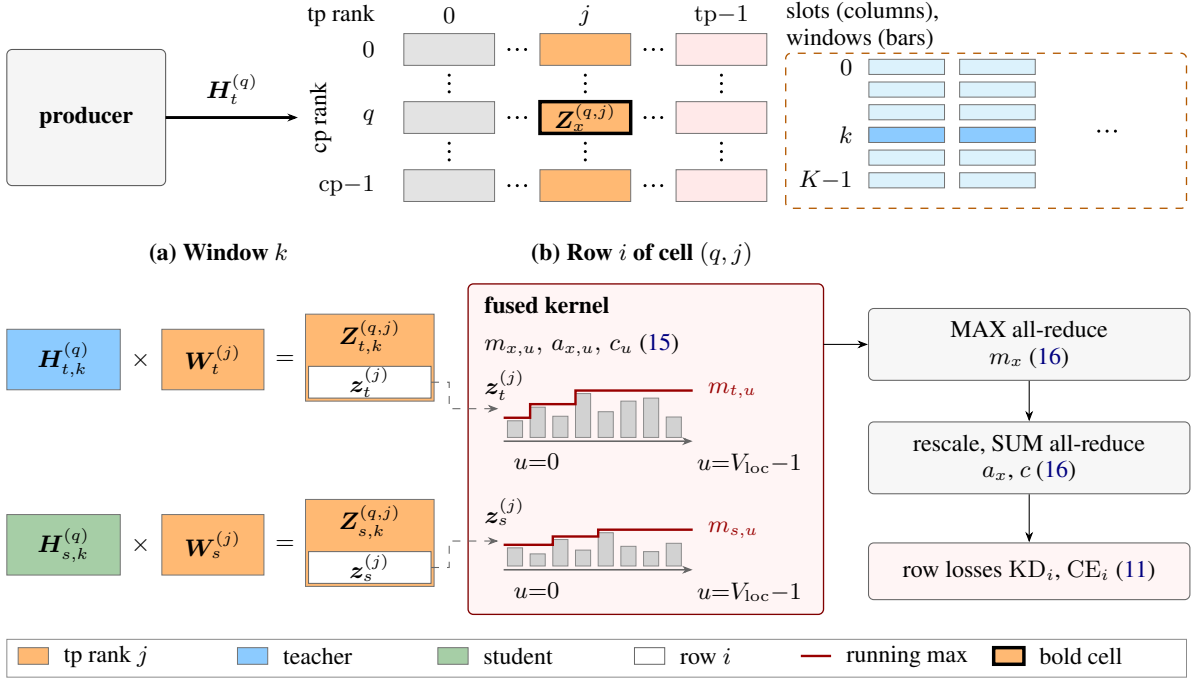
\begin{figure*}[t]
\centering
\begingroup\medmuskip=2mu plus 1mu minus 1mu\thickmuskip=3mu plus 1mu minus 1mu
\begin{tikzpicture}[font=\small, >={Stealth[length=4pt]},
  blk/.style={draw=black!65, rounded corners=2pt, inner sep=2pt, align=center},
  band/.style={draw=black!55, line width=0.3pt, inner sep=0pt},
  rk0/.style={fill=gray!22}, rkj/.style={fill=orange!55}, rkL/.style={fill=pink!35},
  htk/.style={fill=blue!45!cyan!45}, hsk/.style={fill=green!45!black!35},
  proc/.style={draw=black!60, rounded corners=2pt, fill=gray!6, inner sep=2pt, align=center},
  hdr/.style={font=\small\bfseries, anchor=base},
  red/.style={red!60!black}]

\begin{scope}[yshift=0.3cm]
\draw[blk, fill=gray!8] (0.0,5.25) rectangle (2.1,7.05);
\node[font=\small\bfseries] at (1.05,6.15) {producer};
\draw[->, line width=1.1pt] (2.1,6.15) -- (3.85,6.15);
\node[anchor=south] at (2.98,6.2) {$\bm{H}_t^{(q)}$};

\node[anchor=east] at (5.0,7.50) {tp rank};
\foreach \x/\lab in {5.85/{$0$}, 7.65/{$j$}, 9.45/{$\mathrm{tp}{-}1$}}
  \node at (\x,7.50) {\lab};
\node[rotate=90] at (4.15,6.15) {cp rank};
\foreach \y/\lab in {7.05/{$0$}, 6.15/{$q$}, 5.25/{$\mathrm{cp}{-}1$}}
  \node[anchor=east] at (5.0,\y) {\lab};
\foreach \y in {7.05, 6.15, 5.25} {
  \draw[band, rk0] (5.25,\y-0.21) rectangle +(1.2cm,0.42cm);
  \draw[band, rkj] (7.05,\y-0.21) rectangle +(1.2cm,0.42cm);
  \draw[band, rkL] (8.85,\y-0.21) rectangle +(1.2cm,0.42cm);
  \foreach \d in {-1,0,1} {\fill (6.75+0.11*\d,\y) circle (0.02); \fill (8.55+0.11*\d,\y) circle (0.02);}
}
\foreach \y in {6.60, 5.70}
  \foreach \x in {5.85, 7.65, 9.45}
    \foreach \d in {-1,0,1} \fill (\x,\y+0.11*\d) circle (0.02);
\draw[draw=black, line width=1.3pt, rkj] (7.05,5.94) rectangle +(1.2cm,0.42cm);
\node at (7.65,6.15) {$\bm{Z}_x^{(q,j)}$};

\draw[dashed, draw=orange!70!black, rounded corners=2pt, line width=0.5pt] (10.30,4.95) rectangle (15.6,7.0);
\node[anchor=south west, align=left, inner sep=0pt] at (10.30,7.05) {slots (columns),\\windows (bars)};
\foreach \x in {11.40,12.60} \foreach \q in {0,...,5} \draw[band, fill=cyan!12] (\x,6.72-0.30*\q) rectangle +(1.0cm,0.2cm);
\draw[band, htk] (11.40,6.72-0.30*3) rectangle +(1.0cm,0.2cm);
\draw[band, htk] (12.60,6.72-0.30*3) rectangle +(1.0cm,0.2cm);
\foreach \d in {-1,0,1} \fill (14.55+0.11*\d,5.97) circle (0.02);
\node[anchor=east] at (11.32,6.82) {$0$};
\node[anchor=east] at (11.32,5.92) {$k$};
\node[anchor=east] at (11.32,5.32) {$K{-}1$};

\end{scope}

\node[hdr] at (2.8,4.56) {(a) Window $k$};
\draw[band, htk] (0.0,2.85) rectangle +(1.5cm,0.8cm);   \node at (0.75,3.25) {$\bm{H}_{t,k}^{(q)}$};
\node at (1.77,3.25) {$\times$};
\draw[band, rkj] (2.05,2.85) rectangle +(1.35cm,0.8cm);  \node at (2.725,3.25) {$\bm{W}_t^{(j)}$};
\node at (3.68,3.25) {$=$};
\draw[band, rkj] (3.95,2.70) rectangle +(1.7cm,1.15cm);  \node at (4.80,3.55) {$\bm{Z}_{t,k}^{(q,j)}$};
\draw[fill=white, draw=black!60, line width=0.3pt] (3.99,2.74) rectangle +(1.62cm,0.42cm);  \node at (4.80,2.95) {$\bm{z}_t^{(j)}$};
\draw[band, hsk] (0.0,0.40) rectangle +(1.5cm,0.8cm);   \node at (0.75,0.80) {$\bm{H}_{s,k}^{(q)}$};
\node at (1.77,0.80) {$\times$};
\draw[band, rkj] (2.05,0.40) rectangle +(1.35cm,0.8cm);  \node at (2.725,0.80) {$\bm{W}_s^{(j)}$};
\node at (3.68,0.80) {$=$};
\draw[band, rkj] (3.95,0.25) rectangle +(1.7cm,1.15cm);  \node at (4.80,1.10) {$\bm{Z}_{s,k}^{(q,j)}$};
\draw[fill=white, draw=black!60, line width=0.3pt] (3.99,0.29) rectangle +(1.62cm,0.42cm);  \node at (4.80,0.50) {$\bm{z}_s^{(j)}$};

\node[hdr] at (8.40,4.56) {(b) Row $i$ of cell $(q,j)$};
\draw[rounded corners=2pt, draw=red!50!black, fill=red!4, line width=0.5pt] (6.10,-0.12) rectangle (10.80,4.25);
\node[anchor=north west, font=\small\bfseries] at (6.18,4.22) {fused kernel};
\node[anchor=north west] at (6.18,3.72) {$m_{x,u},\,a_{x,u},\,c_u$ \eqref{eq:running}};

\node[anchor=north west] at (6.18,3.25) {$\bm{z}_t^{(j)}$};
\foreach \q/\h in {0/0.22, 1/0.40, 2/0.28, 3/0.58, 4/0.34, 5/0.48, 6/0.52, 7/0.27}
  \draw[fill=gray!35, draw=black!45, line width=0.2pt] (6.62+0.30*\q,2.22) rectangle +(0.20cm,\h cm);
\draw[red, line width=0.9pt] (6.57,2.48) -- (6.92,2.48) -- (6.92,2.66) -- (7.52,2.66) -- (7.52,2.84) -- (9.07,2.84);
\node[anchor=west, red] at (9.14,2.84) {$m_{t,u}$};
\draw[->, line width=0.6pt, black!60] (6.57,2.17) -- (9.07,2.17);
\node[anchor=north west] at (6.57,2.14) {$u{=}0$};
\node[anchor=north east] at (10.65,2.14) {$u{=}V_{\mathrm{loc}}{-}1$};

\node[anchor=north west] at (6.18,1.55) {$\bm{z}_s^{(j)}$};
\foreach \q/\h in {0/0.24, 1/0.16, 2/0.35, 3/0.21, 4/0.44, 5/0.26, 6/0.19, 7/0.30}
  \draw[fill=gray!35, draw=black!45, line width=0.2pt] (6.62+0.30*\q,0.52) rectangle +(0.20cm,\h cm);
\draw[red, line width=0.9pt] (6.57,0.80) -- (7.22,0.80) -- (7.22,0.91) -- (7.82,0.91) -- (7.82,1.00) -- (9.07,1.00);
\node[anchor=west, red] at (9.14,1.00) {$m_{s,u}$};
\draw[->, line width=0.6pt, black!60] (6.57,0.47) -- (9.07,0.47);
\node[anchor=north west] at (6.57,0.44) {$u{=}0$};
\node[anchor=north east] at (10.65,0.44) {$u{=}V_{\mathrm{loc}}{-}1$};

\node[proc, minimum width=4.25cm, minimum height=0.95cm] (mx) at (13.52,3.45) {MAX all-reduce\\$m_x$ \eqref{eq:combine}};
\node[proc, minimum width=4.25cm, minimum height=0.95cm] (sm) at (13.52,1.95) {rescale, SUM all-reduce\\$a_x$, $c$ \eqref{eq:combine}};
\node[proc, fill=red!4, minimum width=4.25cm, minimum height=0.75cm] (loss) at (13.52,0.45) {row losses $\mathrm{KD}_i$, $\mathrm{CE}_i$ \eqref{eq:rowloss}};
\draw[->] (10.80,3.45) -- (mx.west);
\draw[->] (mx.south) -- (sm.north);
\draw[->] (sm.south) -- (loss.north);

\draw[->, dashed, black!60] (5.61,2.95) -- (5.85,2.95) -- (5.85,2.60) -- (6.50,2.60);
\draw[->, dashed, black!60] (5.61,0.50) -- (5.85,0.50) -- (5.85,0.85) -- (6.50,0.85);

\draw[draw=black!45, line width=0.3pt] (0.0,-0.95) rectangle (15.6,-0.45);
\draw[band, rkj] (0.15,-0.78) rectangle +(0.40cm,0.22cm);
\node[anchor=west] at (0.62,-0.67) {tp rank $j$};
\draw[band, htk] (3.05,-0.78) rectangle +(0.40cm,0.22cm);
\node[anchor=west] at (3.52,-0.67) {teacher};
\draw[band, hsk] (5.70,-0.78) rectangle +(0.40cm,0.22cm);
\node[anchor=west] at (6.17,-0.67) {student};
\draw[fill=white, draw=black!60, line width=0.3pt] (8.30,-0.78) rectangle +(0.40cm,0.22cm);
\node[anchor=west] at (8.77,-0.67) {row $i$};
\draw[red, line width=0.9pt] (10.50,-0.67) -- (10.90,-0.67);
\node[anchor=west] at (10.97,-0.67) {running max};
\draw[draw=black, line width=1.3pt, rkj] (13.05,-0.78) rectangle +(0.40cm,0.22cm);
\node[anchor=west] at (13.52,-0.67) {bold cell};
\end{tikzpicture}
\endgroup
\caption{The disaggregated algorithm (symbols as in \S\ref{sec:roles} and \S\ref{sec:loss}; $\cdots$: omitted ranks). \emph{Top}: the producer sends slot $\bm{H}_t^{(q)}$ to every tp rank of its cp rank; cell $(q,j)$ of the receiver group (bold) holds the logits block $\bm{Z}_x^{(q,j)}$, and its prefetch ring holds up to 4 slots of $K$ windows. (a)~Window $k$: $\bm{H}_{x,k}^{(q)}\bm{W}_x^{(j)}$ gives $\bm{Z}_{x,k}^{(q,j)}$; row $i$ (white) is what the kernel scans. (b)~The kernel walks the row in tiles with a running maximum (red); tp ranks combine the statistics by all-reduce.}
\label{fig:algorithm}
\end{figure*}

\section{The Disaggregated KD Algorithm}
\label{sec:algorithm}

\subsection{Roles and hidden-state transfer}
\label{sec:roles}

\system{} is built on Megatron-Bridge \citep{megatronbridge} and splits the 2 models across nodes: no GPU or node hosts both models' layers (Figure~\ref{fig:algorithm}). The \emph{trainer} is a standard training run on $N_{\mathrm{stu}}$ nodes with grid $(\mathrm{tp},\mathrm{pp},\mathrm{cp},\mathrm{dp},\mathrm{ep})$. The \emph{producer} runs $R$ replicas of the teacher on $N_{\mathrm{prod}}$ nodes, forward only, each spanning $\mathrm{tp}_p$ GPUs (times an EP degree $\mathrm{ep}_p$) chosen by \eqref{eq:producer} independently of the trainer.

For every microbatch the producer emits the teacher's final hidden states instead of its logits. CP divides a sequence's rows, one per token, over the $\mathrm{cp}$ ranks of the student. The \emph{slot} of cp rank $q$ is the matrix $\bm{H}_t^{(q)}\in\mathbb{R}^{r\times\dt}$ of its $r$ rows, so a sequence has $L=\mathrm{cp}\cdot r$ rows. The $\mathrm{tp}\times\mathrm{cp}$ GPUs of the student's last pipeline stage in one DP replica, the \emph{last-stage ranks}, form a \emph{receiver group}, and the producer sends $\bm{H}_t^{(q)}$ to every tp rank of cp rank $q$. We call the last-stage rank with cp rank $q$ and tp rank $j$ the cell $(q,j)$ (Figure~\ref{fig:algorithm}, top).

The teacher's lm\_head $\bm{W}_t\in\mathbb{R}^{\dt\times V}$ is not part of the producer's forward path. The producer writes this matrix once to a file that the trainers read. The vocabulary is sharded over the tp ranks: tp rank $j$ keeps the column shards $\bm{W}_t^{(j)}\in\mathbb{R}^{\dt\times\Vloc}$ and $\bm{W}_s^{(j)}\in\mathbb{R}^{\ds\times\Vloc}$ of the teacher's and the student's lm\_head, with $\Vloc=V/\mathrm{tp}$, identical on every cp rank. With $\bm{H}_s^{(q)}\in\mathbb{R}^{r\times\ds}$ the student's own final hidden states of the same rows, cell $(q,j)$ therefore evaluates the logits block $\bm{Z}_x^{(q,j)}=\bm{H}_x^{(q)}\bm{W}_x^{(j)}\in\mathbb{R}^{r\times\Vloc}$ for $x\in\{t,s\}$ (teacher, student): the logits $\bm{Z}_x\in\mathbb{R}^{L\times V}$ are partitioned over the receiver group, cp ranks owning different rows and tp ranks different columns. Statistics are reduced over the tp ranks of one cp rank only; cp ranks exchange no row statistics inside the loss. Appendix~\ref{app:pipeline} gives the row layout, protocol, and schedule.

\begin{table*}[t]
\centering
\footnotesize
\setlength{\tabcolsep}{2.5pt}
\renewcommand{\arraystretch}{1.0}
\begin{tabular}{L{2.9cm}L{2.5cm}L{2.2cm}L{3.1cm}L{2.2cm}L{2.0cm}}
\toprule
 & ModelOpt & ModelOpt fused-W & KDFlow & NeMo-RL & \system{} (ours) \\
\midrule
Teacher and student on & same GPUs & same GPUs & same GPUs & same GPUs & separate nodes \\
Idle teacher weights & kept & kept & kept (optional offload) & offloaded & lm\_head only \\
Teacher forward & in the student's step & in the student's step & phases & 1 phase per step & concurrent, ahead \\
Teacher output & logits, $V/\mathrm{tp}$ columns & logits, $V/\mathrm{tp}$ columns & hidden states & logits, $V$ columns & hidden states \\
KL taken over & all $V$ columns & all $V$ columns & all $V$ columns & top 8192 columns & all $V$ columns \\
Largest logits tensor & $L/\mathrm{cp}\times V/\mathrm{tp}$ & $L/\mathrm{cp}\times V/\mathrm{tp}$ & tokens $\times V$, chunks optional & step tokens $\times V$ & $C\times V/\mathrm{tp}$ \\
\midrule
\multicolumn{6}{l}{\emph{Qwen3.5 pair, 4 nodes (ModelOpt: \ballk{} / KD-only)}} \\
Tokens per step ($10^3$) & 4194 & 4194 & 4183 & 1047 & 4129 \\
Tokens/s per GPU & 741 / 1450 & 1303$^\ddagger$ & 390 & 645 & \textbf{1846} \\
\midrule
\multicolumn{6}{l}{\emph{Qwen3.5 pair, 3 nodes}} \\
Tokens per step ($10^3$) & OOM & 4129 & $-$ & $-$ & 4194 \\
Tokens/s per GPU & OOM & \textbf{2389} & $-$ & $-$ & 2352 \\
\midrule
\multicolumn{6}{l}{\emph{Qwen3-32B pair, 1 node (2 for \system{}); ModelOpt: \ballk{} / KD-only}} \\
Tokens per step ($10^3$) & 524 & $-$ & 515 & 524 & 524 \\
Tokens/s per GPU & 3500 / 3570 & $-$ & \textbf{6570} & 1780 & 2426 \\
\bottomrule
\end{tabular}
\caption{\system{} against ModelOpt \citep{megatronbridge,modelopt}, ModelOpt with the fused-W loss of \system{} (Appendix~\ref{app:stacks}), KDFlow \citep{zhang2026kdflow}, and NeMo-RL \citep{nemorl}. Measured rows are our runs (Table~\ref{tab:stacks}); NeMo-RL on Qwen3.5 ran OOM at 4.19M tokens per step; $^\ddagger$: fused, cross-entropy fusion on; $-$: not run. Bold: best tokens/s per GPU in each pair and node count.}
\label{tab:systems}
\end{table*}

\begin{table}[t]
\centering
\footnotesize
\setlength{\tabcolsep}{3pt}
\renewcommand{\arraystretch}{1.0}
\begin{tabular}{rL{1.9cm}L{1.6cm}L{1.9cm}}
\toprule
Nodes & \system{} (ours) & ModelOpt & ModelOpt fused-W \\
\midrule
\multicolumn{4}{l}{\emph{Qwen3.5: 35B/3B to 9B}} \\
2 & 284 & $-$ & 147 \\
3 & 74 & OOM & 72 \\
4 & 70 & 91 to 177 & 101$^\ddagger$ \\
\midrule
\multicolumn{4}{l}{\emph{Qwen3-235B: 235B/22B to 30B/3B}} \\
3 & 393 & OOM & OOM \\
8 & 109 & 75 & $-$ \\
\midrule
\multicolumn{4}{l}{\emph{Nemotron: 560B/55B to 30B/3B}} \\
4 & 630 & OOM & OOM \\
8 & 316 & OOM & $-$ \\
\midrule
\multicolumn{4}{l}{\emph{Llama: 405B to 70B, dense}} \\
6 & 818 & $-$ & OOM \\
8 & $-$ & 484 & $-$ \\
16 & 186 & 301 & $-$ \\
\bottomrule
\end{tabular}
\caption{Outcome of each layout by node count: median s/it or OOM; $-$: not run. ModelOpt fused-W (Appendix~\ref{app:stacks}): the fused loss of \system{} with the model's cross-entropy fusion off ($^\ddagger$: on); \system{} forms no model cross-entropy. Entries pool the grids and losses run (Tables~\ref{tab:brackets}, \ref{tab:runs}, and~\ref{tab:otherpairs}a).}
\label{tab:pairs}
\end{table}

\subsection{Windowed, fused loss over vocabulary shards}
\label{sec:loss}

The KD objective is a sum over rows, and within a row the coupling along the vocabulary axis is the reductions of \eqref{eq:stats}. The trainer therefore evaluates a slot in $K$ \emph{windows} $k=0,\dots,K{-}1$ of $C$ consecutive rows each (Algorithm~\ref{alg:step}, Appendix~\ref{app:pipeline}), so $r=K\cdot C$ ($C$ does not change the loss in exact arithmetic). Let $\bm{H}_{x,k}^{(q)}\in\mathbb{R}^{C\times d_x}$ be window $k$ of $\bm{H}_x^{(q)}$. Cell $(q,j)$ computes
\begin{equation}
\label{eq:window}
\bm{Z}_{x,k}^{(q,j)}=\bm{H}_{x,k}^{(q)}\,\bm{W}_x^{(j)}\in\mathbb{R}^{C\times\Vloc},\quad x\in\{t,s\},
\end{equation}
the window-$k$ rows of its block $\bm{Z}_x^{(q,j)}$. Fix a row $i$ of the window; we leave $i$ implicit in $\bm{z}_x$, $m_x$, $a_x$ and $c$, which all refer to row $i$. Its logits are $\bm{z}_x=[\bm{z}_x^{(0)},\dots,\bm{z}_x^{(\mathrm{tp}-1)}]\in\mathbb{R}^V$, where $\bm{z}_x^{(j)}\in\mathbb{R}^{\Vloc}$ is the $i$th row of $\bm{Z}_{x,k}^{(q,j)}$ with entries $z_{x,u}^{(j)}$, $u=0,\dots,\Vloc{-}1$ (Figure~\ref{fig:algorithm}a), and no rank forms $\bm{z}_x$. With $p_x=\mathrm{softmax}(\bm{z}_x)$ over the $V$ columns and label $y_i$, the trainer minimizes
\begin{equation}
\label{eq:objective}
\begin{aligned}
\mathcal{L}&=\sum_i\bigl[(1-\lambda)\,\mathrm{CE}_i+\lambda\,\mathrm{KD}_i\bigr],\\
\mathrm{CE}_i&=-\log p_{s,y_i},\quad \mathrm{KD}_i=-\sum_v p_{t,v}\log p_{s,v},
\end{aligned}
\end{equation}
with KD weight $\lambda$. $\mathrm{KD}_i$ is the soft cross-entropy, which has the gradient of $\mathrm{KL}(p_t\,\|\,p_s)$.

\raggedbottom
TP rank $j$ holds only $\bm{z}_x^{(j)}$, so the rank cannot normalize a softmax alone, and gathering $\bm{z}_x$ would undo the sharding. The rank need not gather. Define, for the full row,
\begin{equation}
\label{eq:stats}
\begin{aligned}
m_x&=\max_v z_{x,v},\qquad a_x=\sum_v e^{z_{x,v}-m_x},\\
c&=\sum_v e^{z_{t,v}-m_t}\,z_{s,v}.
\end{aligned}
\end{equation}
Then $p_{x,v}=e^{z_{x,v}-m_x}/a_x$, and the row losses are
\begin{equation}
\label{eq:rowloss}
\mathrm{KD}_i=\mathrm{lse}_s-\frac{c}{a_t},\quad \mathrm{CE}_i=\mathrm{lse}_s-z_{s,y_i},
\end{equation}
with $\mathrm{lse}_s=\log\sum_v e^{z_{s,v}}=m_s+\log a_s$,
so 5 numbers $(m_t,m_s,a_t,a_s,c)$ replace the 2 rows of logits. Each is a maximum or a sum over columns, so tp rank $j$ computes its local versions $m_x^{(j)}$, $a_x^{(j)}$, $c^{(j)}$, that is, \eqref{eq:stats} over the entries of $\bm{z}_x^{(j)}$ only. A fused kernel evaluates the statistics with 1 thread block per row, walking the entries in tiles with the online-softmax recurrence \citep{milakov2018online}; the tp ranks then combine the results by a MAX all-reduce followed, after each rank rescales its sums to the global maximum, by a SUM all-reduce (Appendix~\ref{app:loss-equiv}, \eqref{eq:running}, \eqref{eq:combine}).

The local statistics are taken relative to a maximum because $e^{z}$ overflows for large logits while $e^{z-m}\le1$. The maximum of a row is unknown until all entries have been read, and a separate pass to find the maximum would read the logits 2 times, so the fused kernel (Figure~\ref{fig:algorithm}b) reads the entries of $\bm{z}_t^{(j)}$ and $\bm{z}_s^{(j)}$ once and after entry $u$ keeps the running maximum $m_{x,u}$ and the sums $a_{x,u}$ and $c_u$ taken relative to the maximum; when the maximum rises, the sums are rescaled to the new one by a factor of at most 1. In the forward scan a rank holds at most 1 matrix $\in\mathbb{R}^{C\times\Vloc}$ per distribution, never one $\in\mathbb{R}^{r\times V}$, and the loss equals the full-vocabulary loss in exact arithmetic (Theorem~\ref{thm:exact}).

\begin{theorem}[Exactness]
\label{thm:exact}
For every window size $C$ dividing $r$, every grouping of the entries read per step, every $\mathrm{tp}$ dividing $V$, and every $\mathrm{cp}$ with $2\,\mathrm{cp}$ dividing $L$, the losses obtained from the scanned statistics and their combination over the tp ranks through \eqref{eq:rowloss} are the $\mathrm{CE}_i$ and $\mathrm{KD}_i$ of \eqref{eq:objective}, and the backward pass returns the gradient of the full-vocabulary softmax. Hence, in exact arithmetic, $\mathcal{L}$ and its gradients do not depend on these choices.
\end{theorem}
The proof is in Appendix~\ref{app:loss-equiv}.
\flushbottom

With $\mu=\min(\mathrm{pp},2)$ microbatches in flight, the trainer retains
\begin{equation}
\label{eq:kstr}
\kappa_{\mathrm{str}}=\bigl(\mu\,r(\dt{+}\ds)\,b+\dt\Vloc\,(b{+}\fromcode{1}\,\mathbf{1}_{\mathrm{fp8}})\bigr)\bytes
\end{equation}
(hidden-state rows and the teacher's lm\_head shard, with a 1-byte fp8 copy when the teacher-logit product runs in fp8, $\mathbf{1}_{\mathrm{fp8}}{=}1$), where a colocated loss holds 9.5 fp32-equivalent buffers of $r\,\Vloc$ logits (Eq.~\ref{eq:colbuf}).

\subsection{Prefetch and pacing}
\label{sec:pacing}

Slots are received into a ring of up to 4 slots ahead within the iteration (Appendix~\ref{app:pipeline}). Iteration time is the larger of the trainer's compute time and $N_{\mathrm{tok}}/(R\rho_{\mathrm{rep}})$, so a trainer that outruns its replicas is \emph{producer-bound} (Table~\ref{tab:rate}).

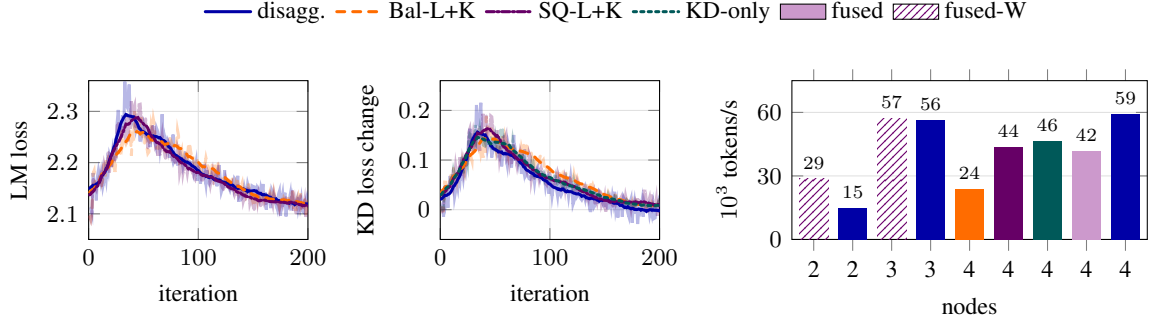
\begin{figure*}[!t]
\centering
\begin{tikzpicture}
\pgfplotsset{
  compat=1.17,
  every axis/.append style={font=\small},
  parity/.style={
    width=2.9cm, height=2.1cm, scale only axis,
    grid=major, grid style={gray!25, line width=0.3pt},
    axis line style={line width=0.4pt}, tick style={line width=0.4pt},
    every axis plot/.append style={line cap=round},
    unbounded coords=jump, xmin=0, xmax=200, xtick={0,100,200},
    xlabel={iteration},
  },
  raw/.style={line width=0.35pt, opacity=0.28, forget plot},
  oursn/.style={blue!65!black, line width=1.15pt},
  colol/.style={orange!85!red, line width=1.15pt, dashed},
  colok/.style={teal!70!black, line width=1.15pt, dash pattern=on 1.2pt off 1.6pt},
  colos/.style={violet!80!black, line width=1.15pt, dash pattern=on 3pt off 1pt on 0.8pt off 1pt},
}
\begin{groupplot}[group style={group size=3 by 1, horizontal sep=1.75cm}]
\nextgroupplot[parity, ylabel={LM loss},
  ymin=2.05, ymax=2.36, ytick={2.1,2.2,2.3}, yticklabel style={/pgf/number format/fixed, /pgf/number format/precision=1},
  legend to name=legendparity, legend columns=6,
  legend style={draw=none, font=\small, /tikz/every even column/.append style={column sep=1.5pt}}]
\addplot[raw, oursn] table[x=it, y=lm, col sep=comma] {figs/p1_ours4n.csv};
\addplot[raw, colol] table[x=it, y=lm, col sep=comma] {figs/p1_colo4n.csv};
\addplot[raw, colos] table[x=it, y=lm, col sep=comma] {figs/p1_colo4n_sq.csv};
\addplot[oursn] table[x=it, y=lm_s, col sep=comma] {figs/p1_ours4n.csv};
\addlegendentry{disagg.}
\addplot[colol] table[x=it, y=lm_s, col sep=comma] {figs/p1_colo4n.csv};
\addlegendentry{Bal-L+K}
\addplot[colos] table[x=it, y=lm_s, col sep=comma] {figs/p1_colo4n_sq.csv};
\addlegendentry{SQ-L+K}
\addlegendimage{colok}
\addlegendentry{KD-only}
\addlegendimage{area legend, fill=violet!40, draw=none}
\addlegendentry{fused}
\addlegendimage{area legend, pattern=north east lines, pattern color=violet!85!black, draw=none}
\addlegendentry{fused-W}
\nextgroupplot[parity, ylabel={KD loss change},
  ymin=-0.06, ymax=0.26, ytick={0,0.1,0.2}, yticklabel style={/pgf/number format/fixed, /pgf/number format/precision=1}]
\addplot[raw, oursn] table[x=it, y=dkd, col sep=comma] {figs/p1_ours4n.csv};
\addplot[raw, colol] table[x=it, y=dkd, col sep=comma] {figs/p1_colo4n.csv};
\addplot[raw, colos] table[x=it, y=dkd, col sep=comma] {figs/p1_colo4n_sq.csv};
\addplot[raw, colok] table[x=it, y=dkd, col sep=comma] {figs/p1_colo4n_kd.csv};
\addplot[oursn] table[x=it, y=dkd_s, col sep=comma] {figs/p1_ours4n.csv};
\addplot[colol] table[x=it, y=dkd_s, col sep=comma] {figs/p1_colo4n.csv};
\addplot[colos] table[x=it, y=dkd_s, col sep=comma] {figs/p1_colo4n_sq.csv};
\addplot[colok] table[x=it, y=dkd_s, col sep=comma] {figs/p1_colo4n_kd.csv};
\nextgroupplot[width=4.7cm, height=2.1cm, scale only axis,
  ybar, bar shift=0pt, bar width=0.38cm, ymin=0, ymax=75, ytick={0,30,60},
  ylabel={$10^3$ tokens/s}, xlabel={nodes},
  axis line style={line width=0.4pt}, tick style={line width=0.4pt},
  ymajorgrids, grid style={gray!25, line width=0.3pt},
  symbolic x coords={n2f,n2d,n3f,n3d,n4b,n4s,n4k,n4u,n4d}, xtick={n2f,n2d,n3f,n3d,n4b,n4s,n4k,n4u,n4d},
  xticklabels={2,2,3,3,4,4,4,4,4},
  enlarge x limits=0.07, nodes near coords, nodes near coords align={vertical},
  every node near coord/.append style={font=\scriptsize, /pgf/number format/fixed, /pgf/number format/precision=0},
  ]
\addplot[pattern=north east lines, pattern color=violet!85!black, draw=none] coordinates {(n2f,28.6)};
\addplot[fill=blue!65!black, draw=none] coordinates {(n2d,14.8)};
\addplot[pattern=north east lines, pattern color=violet!85!black, draw=none] coordinates {(n3f,57.3)};
\addplot[fill=blue!65!black, draw=none] coordinates {(n3d,56.4)};
\addplot[fill=orange!85!red, draw=none] coordinates {(n4b,23.7)};
\addplot[fill=violet!80!black, draw=none] coordinates {(n4s,43.5)};
\addplot[fill=teal!70!black, draw=none] coordinates {(n4k,46.4)};
\addplot[fill=violet!40, draw=none] coordinates {(n4u,41.7)};
\addplot[fill=blue!65!black, draw=none] coordinates {(n4d,59.1)};
\end{groupplot}
\node[anchor=south] at ($(group c1r1.north west)!0.5!(group c3r1.north east)+(0,0.5cm)$) {\pgfplotslegendfromname{legendparity}};
\end{tikzpicture}
\caption{Loss parity and throughput on the Qwen3.5 pair. Colocated runs are ModelOpt with the loss in the legend; fused and fused-W use the fused loss of \system{} (hatched: fused-W). (a)~LM loss of the 4-node runs (thin: per iteration, thick: 17-iteration mean). (b)~KD loss minus its first-3-iteration mean. (c)~Median total throughput by node count (2-node, fused, and fused-W bars: 15 iterations).}
\label{fig:parity}
\end{figure*}

\section{Experiments}
\label{sec:experiments}

\subsection{Setup}
\label{sec:setup}

\begin{table}[!tb]
\centering
\footnotesize
\setlength{\tabcolsep}{2.6pt}
\begin{tabular}{@{}l|rr|r|rr@{}}
\toprule
 & \multicolumn{2}{c|}{s/it, seed} & Peak & \multicolumn{2}{c}{MMLU, seed} \\
Loss (nodes) & 42 & 43 & & 42 & 43 \\
\midrule
\multicolumn{6}{@{}l}{\emph{Colocated (ModelOpt)}} \\
KD-only (4) & 90.5 & 97.7 & 63.4 & 67.97 & 69.22 \\
\ballk{} (4) & 176.8 & 172.1 & 67.2 & 67.51 & 69.78 \\
\sqlk{} (4) & 96.3 & -- & 59.9 & -- & -- \\
\sqlk{} fused (4) & 100.6$^\dagger$ & -- & 40.7 & -- & -- \\
\sqlk{} fused-W (3) & 72.0$^\dagger$ & -- & 64.1 & -- & -- \\
\midrule
\multicolumn{6}{@{}l}{\emph{Disaggregated (ours)}} \\
\sqlk{} (3) & 74.3 & -- & 39.2 & \textbf{70.14} & -- \\
\sqlk{} (4) & \textbf{69.9} & 61.9$^\ddagger$ & \textbf{33.0} & 68.52 & \textbf{70.09} \\
\bottomrule
\end{tabular}
\caption{Primary pair at iteration 200, 1 run per cell; --: not run; $^\dagger$: 15 iterations (\sqlk{}: 110.3); $^\ddagger$: 30 iterations without cache release (Appendix~\ref{app:cache}; 99.4 with it); W: cross-entropy fusion off, batch 126; peak: torch GiB; best bold.}
\label{tab:main}
\end{table}

\paragraph{Pairs.}
We use 5 open-weight, same-tokenizer pairs (Tables~\ref{tab:systems} and~\ref{tab:pairs}) from the Qwen3.5 \citep{qwen2026qwen35}, Qwen3 \citep{yang2025qwen3}, Nemotron 3 \citep{nvidia2025nemotron3,nvidia2026nemotron3ultra}, and Llama 3.1 \citep{grattafiori2024llama3} families. The Qwen3.5 pair (35B-A3B teacher, 9B student) is primary: every colocated loss and disaggregated split runs on it, at several node counts; the Qwen3-32B pair fits 1 node and serves the comparison with other stacks.

\paragraph{Loss and its position.}
\emph{Offline} distillation stores teacher outputs as generated sequences \citep{kim2016sequence} or truncated top-K logits \citep{ryskulov2026topk}. \emph{On-policy} distillation (OPD) trains the student on its own samples and matches the teacher's token-level distributions on them \citep{agarwal2024onpolicy,gu2024minillm,yang2025qwen3}. Our objective \eqref{eq:objective} is neither. Its loss and $\lambda$ schedule follow SlimQwen \citep{tang2026slimqwen} without its multi-token-prediction (MTP) terms (\sqlk{}, SlimQwen's LM+KD mixture): the full-vocabulary forward Kullback--Leibler (KL) divergence between the teacher's and the student's next-token distributions, in its soft cross-entropy form, mixed with the LM loss as $(1-\lambda)\,\mathcal{L}_{\mathrm{LM}}+\lambda\,\mathcal{L}_{\mathrm{KD}}$ and $\lambda$ decaying from 1 to 0.75, where $\lambda{=}1$ gives the KD-only variant. ModelOpt and KDFlow minimize the same full-vocabulary forward KL, and NeMo-RL takes the KL over the top 8192 columns. ModelOpt's loss balancer (\ballk{}, its LM+KD balancer) rescales the KD term to the LM loss instead of weighting by $\lambda$, which gives the KD term a gradient weight of 0.89 (Table~\ref{tab:kdshare}).

\paragraph{Recipe and colocated stacks.}
All runs use nodes of 8 80\,GB H100 GPUs (79.11\gib{} usable), connected by NVLink and InfiniBand, and train on FineWeb-Edu \citep{penedo2024fineweb}. The ModelOpt runs are Megatron-Bridge runs with the KD-only (native), \ballk{}, and \sqlk{} losses; the stacks differ in loss weighting, learning-rate floor, teacher precision, and, at 4 nodes, global batch. KDFlow and NeMo-RL were run on the primary pair and on Qwen3-32B to Qwen3-4B, on the corpus packed to one sample length (Table~\ref{tab:stacks}). In most DisaggKD runs the trainer computes the teacher's logits with an fp8 matrix product (Appendix~\ref{app:fp8}). MMLU \citep{hendrycks2021mmlu} is scored with FlexEval \citep{flexeval} (Appendix~\ref{app:runs}).

\subsection{Results on the Qwen3.5 pair}
\label{sec:results}

\paragraph{Loss parity.}
Figure~\ref{fig:parity}a compares the LM loss of the 4-node disaggregated run with the 4-node colocated runs: their 17-iteration moving averages differ by 0.008 nats on average from the \sqlk{} run and 0.012 from the \ballk{} run (Table~\ref{tab:parity}); the runs read different document orders (Table~\ref{tab:dataorder}). The disaggregated runs and the colocated \sqlk{} run log soft cross-entropy and the other ModelOpt runs KL, which differ by the teacher's entropy, 2.00 nats in the logs, so Figure~\ref{fig:parity}b subtracts the initial level.

\paragraph{Speedup.}
At 4 nodes and 32768 tokens the disaggregated run is faster than the colocated \ballk{}, \sqlk{}, and KD-only runs at seed 42 (Tables~\ref{tab:main} and~\ref{tab:rate}); it uses an fp8 teacher and the colocated runs bf16 (Table~\ref{tab:fp8colo}). \ballk{} adds 86\,s per iteration to the KD-only run, 81\% of the gap, and the traced loss stage gains about 50\,s of event time in collectives, copy operators, and stream synchronizations (Table~\ref{tab:speedup}a); \sqlk{} adds 6\,s. The colocated stack runs the teacher forward pass (16\,s, estimated) and the student step in sequence; the disaggregated stack runs them concurrently (\S\ref{sec:pacing}). A colocated run with the fused loss of \system{} (\sqlk{} fused) peaks lower and, with the model's cross-entropy fusion off (fused-W), trains at a similar speed at 3 nodes and faster at 2 nodes, where \system{} splits 1+1 (Table~\ref{tab:pairs}); \system{} holds less memory per GPU at 3 nodes.

\paragraph{Quality.}
With seed 42 the 4-node disaggregated student is 1.0 point above the colocated \ballk{} student and 0.6 above KD-only, and the 3-node student is 1.6 above the 4-node one; with seed 43 each student with a run gains 1.3 to 2.3 points, so layout differences are of the size of run-to-run differences (Table~\ref{tab:seeds}).

\subsection{Other pairs}
\label{sec:others}

\paragraph{Qwen3-32B.}
Splitting a pair that fits 1 node colocated doubles the GPU count, which halves the rate per GPU at equal step time. The rest of KDFlow's 2.7$\times$ lead, 1.35$\times$, is the step: our trainer-limited 13.5\,s (Table~\ref{tab:rate}) against its 9.8\,s and ModelOpt's 18.4\,s (Table~\ref{tab:stacks}).

\paragraph{Qwen3-235B.}
Of the colocated grids, the 3 colocated runs of Table~\ref{tab:brackets} that fit memory use ep 16 or more, where the all-to-all dispatcher corrupts the student's forward pass, so training gives not-a-number (NaN) values or diverges (Table~\ref{tab:moeprobe}). With the all-gather dispatcher the colocated stack trains the pair on 8 nodes faster than \system{} (75\,s with \ballk{}, against 109\,s; Table~\ref{tab:pairs}): the disaggregated pace is the forward time of its 4 producer replicas (duty 1.0, Table~\ref{tab:rate}).

\paragraph{Nemotron.}
Memory and 2 NaN sources keep the colocated stack from training Nemotron: the teacher runs OOM at 4 nodes, and the 2 8-node grids that fit memory (ep 16) give NaN, in the forward loss with the all-to-all dispatcher and in the gradient norm with either. A Nano-teacher probe gives NaN at ep 16, not at ep 4 (Table~\ref{tab:nemoprobe}).

\paragraph{Llama.}
The model predicts every colocated tp$\le$8 grid over the budget, so the colocated runs use tp 32; 6 nodes have no grid at a global batch of 128 (Table~\ref{tab:brackets}). At 16 nodes disaggregated (tp 8) is faster than colocated (tp 32; Table~\ref{tab:otherpairs}a).

\section{Conclusion}

Memory and traffic models predict the colocated memory limits. On Qwen3.5 \system{} matches the colocated LM loss and is faster than colocated runs with the same loss, and the fused loss of \system{} lets a colocated run train on as few nodes at a similar or faster speed; \system{} is memory efficient when the student spans several nodes. On the larger pairs colocated runs with a fused loss still run OOM where \system{} trains.

\section*{Limitations}

DisaggKD needs separate teacher nodes, so where a colocated layout already fits the pair it uses more GPUs. Each framework runs under its own recipe; aligning hyperparameters across frameworks, isolating the cause of the NaN in the colocated Nemotron student, and fully sharded data parallelism (FSDP) at pp 1 are outside this paper. The estimation models are fitted to the parallelism ranges and sequence lengths we ran, and some of their terms rest on few runs. We study where the teacher runs and how its loss is computed, not how to raise downstream quality; the runs are short, with 2 seeds on 1 pair, on 1 node type and 1 English corpus.

\bibliography{refs}

\begin{thebibliography}{38}
\providecommand{\natexlab}[1]{#1}

\bibitem[{Agarwal et~al.(2024)Agarwal, Vieillard, Zhou, Stanczyk, Ramos~Garea,
  Geist, and Bachem}]{agarwal2024onpolicy}
Rishabh Agarwal, Nino Vieillard, Yongchao Zhou, Piotr Stanczyk, Sabela
  Ramos~Garea, Matthieu Geist, and Olivier Bachem. 2024.
\newblock \href
  {https://proceedings.iclr.cc/paper_files/paper/2024/file/5be69a584901a26c521c2b51e40a4c20-Paper-Conference.pdf}
  {On-policy distillation of language models: Learning from self-generated
  mistakes}.
\newblock In \emph{International Conference on Learning Representations},
  volume 2024, pages 21246--21263.

\bibitem[{Bercovich et~al.(2025)Bercovich, Ronen, Abramovich, Ailon, Assaf,
  Dabbah, Galil, Geifman, Geifman, Golan, Haber, Karpas, Koren, Levy,
  Molchanov, Mor, Moshe, Nabwani, Puny, Rubin, Schen, Shahaf, Tropp, Argov,
  Zilberstein, and El-Yaniv}]{bercovich2025puzzle}
Akhiad Bercovich, Tomer Ronen, Talor Abramovich, Nir Ailon, Nave Assaf,
  Mohammed Dabbah, Ido Galil, Amnon Geifman, Yonatan Geifman, Izhak Golan,
  Netanel Haber, Ehud~Dov Karpas, Roi Koren, Itay Levy, Pavlo Molchanov, Shahar
  Mor, Zach Moshe, Najeeb Nabwani, Omri Puny, and 7 others. 2025.
\newblock \href {https://proceedings.mlr.press/v267/bercovich25a.html} {Puzzle:
  Distillation-based {NAS} for inference-optimized {LLM}s}.
\newblock In \emph{Proceedings of the 42nd International Conference on Machine
  Learning}, volume 267 of \emph{Proceedings of Machine Learning Research},
  pages 3806--3830. PMLR.

\bibitem[{Busbridge et~al.(2025)Busbridge, Shidani, Weers, Ramapuram, Littwin,
  and Webb}]{busbridge2025distillation}
Dan Busbridge, Amitis Shidani, Floris Weers, Jason Ramapuram, Etai Littwin, and
  Russell Webb. 2025.
\newblock \href {https://proceedings.mlr.press/v267/busbridge25a.html}
  {Distillation scaling laws}.
\newblock In \emph{Proceedings of the 42nd International Conference on Machine
  Learning}, volume 267 of \emph{Proceedings of Machine Learning Research},
  pages 5977--6045. PMLR.

\bibitem[{{Gemma Team}(2024)}]{gemmateam2024gemma2}
{Gemma Team}. 2024.
\newblock \href {https://arxiv.org/abs/2408.00118} {{Gemma} 2: Improving open
  language models at a practical size}.
\newblock \emph{Preprint}, arXiv:2408.00118.

\bibitem[{Gu et~al.(2024)Gu, Dong, Wei, and Huang}]{gu2024minillm}
Yuxian Gu, Li~Dong, Furu Wei, and Minlie Huang. 2024.
\newblock \href
  {https://proceedings.iclr.cc/paper_files/paper/2024/file/8ac015d409635f196f9e3e9dcfb9a94e-Paper-Conference.pdf}
  {{MiniLLM}: Knowledge distillation of large language models}.
\newblock In \emph{International Conference on Learning Representations},
  volume 2024, pages 32694--32717.

\bibitem[{Hendrycks et~al.(2021)Hendrycks, Burns, Basart, Zou, Mazeika, Song,
  and Steinhardt}]{hendrycks2021mmlu}
Dan Hendrycks, Collin Burns, Steven Basart, Andy Zou, Mantas Mazeika, Dawn
  Song, and Jacob Steinhardt. 2021.
\newblock \href {https://openreview.net/forum?id=d7KBjmI3GmQ} {Measuring
  massive multitask language understanding}.
\newblock In \emph{International Conference on Learning Representations}.

\bibitem[{Hinton et~al.(2015)Hinton, Vinyals, and Dean}]{hinton2015distilling}
Geoffrey Hinton, Oriol Vinyals, and Jeff Dean. 2015.
\newblock \href {https://arxiv.org/abs/1503.02531} {Distilling the knowledge in
  a neural network}.
\newblock \emph{Preprint}, arXiv:1503.02531.

\bibitem[{Hsu et~al.(2025)Hsu, Dai, Kothapalli, Song, Tang, Zhu, Shimizu,
  Sahni, Ning, and Chen}]{hsu2025liger}
Pin-Lun Hsu, Yun Dai, Vignesh Kothapalli, Qingquan Song, Shao Tang, Siyu Zhu,
  Steven Shimizu, Shivam Sahni, Haowen Ning, and Yanning Chen. 2025.
\newblock \href {https://arxiv.org/abs/2410.10989} {{Liger} {Kernel}: Efficient
  {Triton} kernels for {LLM} training}.
\newblock \emph{Preprint}, arXiv:2410.10989.

\bibitem[{Kim and Rush(2016)}]{kim2016sequence}
Yoon Kim and Alexander~M. Rush. 2016.
\newblock \href {https://doi.org/10.18653/v1/D16-1139} {Sequence-level
  knowledge distillation}.
\newblock In \emph{Proceedings of the 2016 Conference on Empirical Methods in
  Natural Language Processing}, pages 1317--1327, Austin, Texas. Association
  for Computational Linguistics.

\bibitem[{Korthikanti et~al.(2023)Korthikanti, Casper, Lym, McAfee, Andersch,
  Shoeybi, and Catanzaro}]{korthikanti2023reducing}
Vijay~Anand Korthikanti, Jared Casper, Sangkug Lym, Lawrence McAfee, Michael
  Andersch, Mohammad Shoeybi, and Bryan Catanzaro. 2023.
\newblock \href
  {https://proceedings.mlsys.org/paper_files/paper/2023/file/80083951326cf5b35e5100260d64ed81-Paper-mlsys2023.pdf}
  {Reducing activation recomputation in large transformer models}.
\newblock In \emph{Proceedings of Machine Learning and Systems}, volume~5,
  pages 341--353. Curan.

\bibitem[{Liu et~al.(2024)Liu, Zaharia, and Abbeel}]{liu2024ringattention}
Hao Liu, Matei Zaharia, and Pieter Abbeel. 2024.
\newblock \href
  {https://proceedings.iclr.cc/paper_files/paper/2024/file/1119587863e78451f080da2a768c4935-Paper-Conference.pdf}
  {{RingAttention} with blockwise transformers for near-infinite context}.
\newblock In \emph{International Conference on Learning Representations},
  volume 2024, pages 3992--4008.

\bibitem[{{Llama Team}(2024)}]{grattafiori2024llama3}
{Llama Team}. 2024.
\newblock \href {https://arxiv.org/abs/2407.21783} {The {Llama} 3 herd of
  models}.
\newblock \emph{Preprint}, arXiv:2407.21783.

\bibitem[{Lu and Liu(2026)}]{lu2026strong}
Taiming Lu and Zhuang Liu. 2026.
\newblock \href {https://arxiv.org/abs/2605.23857} {Strong teacher not needed?
  {O}n distillation in {LLM} pretraining}.
\newblock \emph{Preprint}, arXiv:2605.23857.

\bibitem[{Mei et~al.(2025)Mei, Fu, Li, Wang, Zhang, and Wu}]{mei2025real}
Zhiyu Mei, Wei Fu, Kaiwei Li, Guangju Wang, Huanchen Zhang, and Yi~Wu. 2025.
\newblock \href
  {https://proceedings.mlsys.org/paper_files/paper/2025/file/3b3889d313ba9476c12c2d77ea66b24f-Paper-Conference.pdf}
  {{ReaL}: Efficient {RLHF} training of large language models with parameter
  reallocation}.
\newblock In \emph{Proceedings of Machine Learning and Systems}, volume~7.
  MLSys.

\bibitem[{Milakov and Gimelshein(2018)}]{milakov2018online}
Maxim Milakov and Natalia Gimelshein. 2018.
\newblock \href {https://arxiv.org/abs/1805.02867} {Online normalizer
  calculation for softmax}.
\newblock \emph{Preprint}, arXiv:1805.02867.

\bibitem[{{Mistral Team}(2026)}]{liu2026ministral3}
{Mistral Team}. 2026.
\newblock \href {https://arxiv.org/abs/2601.08584} {{Ministral} 3}.
\newblock \emph{Preprint}, arXiv:2601.08584.

\bibitem[{Muralidharan et~al.(2024)Muralidharan, Turuvekere~Sreenivas, Joshi,
  Chochowski, Patwary, Shoeybi, Catanzaro, Kautz, and
  Molchanov}]{muralidharan2024compact}
Saurav Muralidharan, Sharath Turuvekere~Sreenivas, Raviraj Joshi, Marcin
  Chochowski, Mostofa Patwary, Mohammad Shoeybi, Bryan Catanzaro, Jan Kautz,
  and Pavlo Molchanov. 2024.
\newblock \href {https://doi.org/10.52202/079017-1299} {Compact language models
  via pruning and knowledge distillation}.
\newblock In \emph{Advances in Neural Information Processing Systems},
  volume~37, pages 41076--41102. Curran Associates, Inc.

\bibitem[{{NeMo RL}(2025)}]{nemorl}
{NeMo RL}. 2025.
\newblock {NeMo RL}: A scalable and efficient post-training library.
\newblock \url{https://github.com/NVIDIA-NeMo/RL}.
\newblock GitHub repository.

\bibitem[{{NVIDIA}(2024--2026)}]{modelopt}
{NVIDIA}. 2024--2026.
\newblock {NVIDIA Model Optimizer}.
\newblock \url{https://github.com/NVIDIA/Model-Optimizer}.
\newblock GitHub repository.

\bibitem[{{NVIDIA}(2025{\natexlab{a}})}]{megatronbridge}
{NVIDIA}. 2025{\natexlab{a}}.
\newblock {NeMo Megatron Bridge}.
\newblock \url{https://github.com/NVIDIA-NeMo/Megatron-Bridge}.
\newblock GitHub repository.

\bibitem[{{NVIDIA}(2025{\natexlab{b}})}]{nvidia2025nemotron3}
{NVIDIA}. 2025{\natexlab{b}}.
\newblock \href {https://arxiv.org/abs/2512.20856} {{NVIDIA} {Nemotron} 3:
  Efficient and open intelligence}.
\newblock \emph{Preprint}, arXiv:2512.20856.

\bibitem[{{NVIDIA}(2025{\natexlab{c}})}]{nvidia2025nemotronnano2}
{NVIDIA}. 2025{\natexlab{c}}.
\newblock \href {https://arxiv.org/abs/2508.14444} {{NVIDIA} {Nemotron} {Nano}
  2: An accurate and efficient hybrid {Mamba-Transformer} reasoning model}.
\newblock \emph{Preprint}, arXiv:2508.14444.

\bibitem[{{NVIDIA}(2026)}]{nvidia2026nemotron3ultra}
{NVIDIA}. 2026.
\newblock \href {https://arxiv.org/abs/2606.15007} {{Nemotron} 3 {Ultra}: Open,
  efficient {Mixture-of-Experts} hybrid {Mamba-Transformer} model for agentic
  reasoning}.
\newblock \emph{Preprint}, arXiv:2606.15007.

\bibitem[{Patel et~al.(2024)Patel, Choukse, Zhang, Shah, Goiri, Maleki, and
  Bianchini}]{patel2024splitwise}
Pratyush Patel, Esha Choukse, Chaojie Zhang, Aashaka Shah, Íñigo Goiri, Saeed
  Maleki, and Ricardo Bianchini. 2024.
\newblock \href {https://doi.org/10.1109/ISCA59077.2024.00019} {{Splitwise}:
  Efficient generative {LLM} inference using phase splitting}.
\newblock In \emph{2024 ACM/IEEE 51st Annual International Symposium on
  Computer Architecture (ISCA)}, pages 118--132.

\bibitem[{Penedo et~al.(2024)Penedo, Kydl\'{\i}\v{c}ek, allal, Lozhkov,
  Mitchell, Raffel, Von~Werra, and Wolf}]{penedo2024fineweb}
Guilherme Penedo, Hynek Kydl\'{\i}\v{c}ek, Loubna~Ben allal, Anton Lozhkov,
  Margaret Mitchell, Colin Raffel, Leandro Von~Werra, and Thomas Wolf. 2024.
\newblock \href {https://doi.org/10.52202/079017-0970} {The {FineWeb} datasets:
  Decanting the web for the finest text data at scale}.
\newblock In \emph{Advances in Neural Information Processing Systems},
  volume~37, pages 30811--30849. Curran Associates, Inc.

\bibitem[{Peng et~al.(2025)Peng, Lv, Bai, Yao, Zhang, Hou, and
  Li}]{peng2025pretraining}
Hao Peng, Xin Lv, Yushi Bai, Zijun Yao, Jiajie Zhang, Lei Hou, and Juanzi Li.
  2025.
\newblock \href {https://doi.org/10.18653/v1/2025.acl-long.181} {Pre-training
  distillation for large language models: A design space exploration}.
\newblock In \emph{Proceedings of the 63rd Annual Meeting of the Association
  for Computational Linguistics (Volume 1: Long Papers)}, pages 3603--3618,
  Vienna, Austria. Association for Computational Linguistics.

\bibitem[{Qin et~al.(2025)Qin, Li, He, Cui, Ren, Zhang, Wu, Zheng, and
  Xu}]{qin2025mooncake}
Ruoyu Qin, Zheming Li, Weiran He, Jialei Cui, Feng Ren, Mingxing Zhang, Yongwei
  Wu, Weimin Zheng, and Xinran Xu. 2025.
\newblock \href {https://www.usenix.org/conference/fast25/presentation/qin}
  {Mooncake: Trading more storage for less computation {\textemdash} a
  {KVCache-centric} architecture for serving {LLM} chatbot}.
\newblock In \emph{23rd USENIX Conference on File and Storage Technologies
  (FAST 25)}, pages 155--170, Santa Clara, CA. USENIX Association.

\bibitem[{{Qwen Team}(2025)}]{yang2025qwen3}
{Qwen Team}. 2025.
\newblock \href {https://arxiv.org/abs/2505.09388} {{Qwen3} technical report}.
\newblock \emph{Preprint}, arXiv:2505.09388.

\bibitem[{{Qwen Team}(2026)}]{qwen2026qwen35}
{Qwen Team}. 2026.
\newblock \href {https://qwen.ai/blog?id=qwen3.5} {{Qwen3.5}: Towards native
  multimodal agents}.

\bibitem[{Ryskulov et~al.(2026)Ryskulov, Garc{\'i}a-Ferrero, Montero, Jansen,
  Hashemi, Garcia, Tiene, and Or{\'u}s}]{ryskulov2026topk}
Bakbergen Ryskulov, Iker Garc{\'i}a-Ferrero, David Montero, David Jansen, Ali
  Hashemi, Jezabel~R. Garcia, Antonio Tiene, and Rom{\'a}n Or{\'u}s. 2026.
\newblock \href {https://arxiv.org/abs/2608.03796} {Efficient knowledge
  distillation for {LLMs}: Offline {Top-K} logits and a fused chunked {KL}
  loss}.
\newblock \emph{Preprint}, arXiv:2608.03796.

\bibitem[{{SB Intuitions}(2026)}]{flexeval}
{SB Intuitions}. 2026.
\newblock {FlexEval}: Flexible evaluation tool for language models.
\newblock \url{https://github.com/sbintuitions/flexeval}.
\newblock Version 0.17.3, commit 8a07c3c, 23 July 2026.

\bibitem[{Sheng et~al.(2025)Sheng, Zhang, Ye, Wu, Zhang, Zhang, Peng, Lin, and
  Wu}]{sheng2025hybridflow}
Guangming Sheng, Chi Zhang, Zilingfeng Ye, Xibin Wu, Wang Zhang, Ru~Zhang,
  Yanghua Peng, Haibin Lin, and Chuan Wu. 2025.
\newblock \href {https://doi.org/10.1145/3689031.3696075} {{HybridFlow}: A
  flexible and efficient {RLHF} framework}.
\newblock In \emph{Proceedings of the Twentieth European Conference on Computer
  Systems}, EuroSys '25, page 1279–1297, New York, NY, USA. Association for
  Computing Machinery.

\bibitem[{Shoeybi et~al.(2020)Shoeybi, Patwary, Puri, LeGresley, Casper, and
  Catanzaro}]{shoeybi2019megatron}
Mohammad Shoeybi, Mostofa Patwary, Raul Puri, Patrick LeGresley, Jared Casper,
  and Bryan Catanzaro. 2020.
\newblock \href {https://arxiv.org/abs/1909.08053} {{Megatron-LM}: Training
  multi-billion parameter language models using model parallelism}.
\newblock \emph{Preprint}, arXiv:1909.08053.

\bibitem[{Tang et~al.(2026)Tang, Wang, Zheng, Wang, Men, Zhang, Yuan, Qiu,
  Shen, and Liu}]{tang2026slimqwen}
Shengkun Tang, Zekun Wang, Bo~Zheng, Liangyu Wang, Rui Men, Siqi Zhang, Xiulong
  Yuan, Zihan Qiu, Zhiqiang Shen, and Dayiheng Liu. 2026.
\newblock \href {https://openreview.net/forum?id=Rfm4WL8nTs} {{SlimQwen}:
  Exploring the pruning and distillation in large {MoE} model pre-training}.
\newblock In \emph{Third Conference on Language Modeling}.

\bibitem[{Wijmans et~al.(2025)Wijmans, Huval, Hertzberg, Koltun, and
  Kr\"{a}henb\"{u}hl}]{wijmans2025cutlosses}
Erik Wijmans, Brody Huval, Alexander Hertzberg, Vladlen Koltun, and Philipp
  Kr\"{a}henb\"{u}hl. 2025.
\newblock \href
  {https://proceedings.iclr.cc/paper_files/paper/2025/file/aa963ac256590bb7ad5fc26c68229a3a-Paper-Conference.pdf}
  {Cut your losses in large-vocabulary language models}.
\newblock In \emph{International Conference on Learning Representations},
  volume 2025, pages 68174--68193.

\bibitem[{Yuan et~al.(2026)Yuan, Chen, Peng, Zhou, Ruan, Wang, Zheng, Men,
  Wang, Zhang, Chen, Yuan, Gao, Qian, Lin, Li, Lin, Wang, and
  Zhou}]{yuan2026maestro}
Xiulong Yuan, Hongqing Chen, Jiaxuan Peng, Fan Zhou, Zhixiang Ruan, Zekun Wang,
  Bo~Zheng, Rui Men, Haiquan Wang, Zhipeng Zhang, Langshi Chen, Man Yuan, Jiaqi
  Gao, Zhengping Qian, Junyang Lin, Yong Li, Wei Lin, Junhua Wang, and Jingren
  Zhou. 2026.
\newblock \href {https://arxiv.org/abs/2605.10501} {Accelerating compound {LLM}
  training workloads with {Maestro}}.
\newblock \emph{Preprint}, arXiv:2605.10501.

\bibitem[{Zhang et~al.(2026)Zhang, Zhang, Zhang, Hu, Chen, and
  Xu}]{zhang2026kdflow}
Songming Zhang, Xue Zhang, Tong Zhang, Bojie Hu, Yufeng Chen, and Jinan Xu.
  2026.
\newblock \href {https://arxiv.org/abs/2603.01875} {{KDFlow}: A user-friendly
  and efficient knowledge distillation framework for large language models}.
\newblock \emph{Preprint}, arXiv:2603.01875.

\bibitem[{Zhong et~al.(2024)Zhong, Liu, Chen, Hu, Zhu, Liu, Jin, and
  Zhang}]{zhong2024distserve}
Yinmin Zhong, Shengyu Liu, Junda Chen, Jianbo Hu, Yibo Zhu, Xuanzhe Liu, Xin
  Jin, and Hao Zhang. 2024.
\newblock \href
  {https://www.usenix.org/conference/osdi24/presentation/zhong-yinmin}
  {{DistServe}: Disaggregating prefill and decoding for goodput-optimized large
  language model serving}.
\newblock In \emph{18th USENIX Symposium on Operating Systems Design and
  Implementation (OSDI 24)}, pages 193--210, Santa Clara, CA. USENIX
  Association.

\end{thebibliography}

\captionsetup{font=10pt}
\appendix
\raggedbottom

\section{Estimation Models in Detail}

\subsection{Memory Model}
\label{app:memory}
\paragraph{Per-module accounting.}
Table~\ref{tab:modules} lists the terms of \eqref{eq:statics}, \eqref{eq:colocate}, \eqref{eq:trainer}, and \eqref{eq:producer} per module. The red constants were fitted to measured OOM events and peak readings (Tables~\ref{tab:brackets}, \ref{tab:trainerruns}, and~\ref{tab:producerruns}), the green ones follow from the code, and the black ones are allowances.

\begin{table*}[t]
\centering
\footnotesize
\setlength{\tabcolsep}{4pt}
\renewcommand{\arraystretch}{1.1}
\begin{tabular}{L{4.6cm}L{10.4cm}}
\toprule
Module & Term \\
\midrule
Student weights, gradients, optimizer & $(\fromcode{2}{+}\fromcode{4})\,P_{\mathrm{loc}}\bytes+\fromcode{12}(D_s/(\mathrm{tp}\,\mathrm{pp}\,\mathrm{dp}\,\mathrm{cp})+E_s/(\mathrm{pp}\,\mathrm{ep}\,\mathrm{edp}))\bytes$, experts on $\mathrm{pp}\cdot\mathrm{ep}$, their optimizer state on $\mathrm{edp}$ ranks \\
Student activations, colocated & $\fromcode{1}\tfrac{\ell_s}{\mathrm{pp}}\tfrac{L}{\mathrm{cp}\,\mathrm{tp}}\ds\, b\bytes$ stored layer inputs of the 1 microbatch in flight at the loss stage (full recomputation) \\
Student attention scores & $(L/\mathrm{cp})^2(h/\mathrm{tp})\,b\bytes$ when attention is not fused (zero for the students of this paper) \\
Student MoE dispatch & $\fitted{2}\tfrac{L}{\mathrm{cp}}k_{\mathrm{e}}\ds\, b/\mathrm{tp}\bytes$ \\
Colocated teacher shard & $\fromcode{2}(D_t/(\mathrm{pp}\,\mathrm{tp})+E_t/(\mathrm{pp}\,\mathrm{ep}))\bytes$ (bf16) \\
Colocated logits buffers & $(\fitted{7}{+}1{+}\fitted{1.5})\,(L/\mathrm{cp})(V/\mathrm{tp})\,\fromcode{4}\bytes$ (7 buffers, 1 for the loss-averaging all-reduce, and 1.5 for the backward pass; 1 more with the LM loss) \\
Teacher forward-pass memory, colocated (the peak is the maximum with this term) & MoE expert-MLP tile (\fromcode{8} fp32 copies of the dispatched token pairs, divided by $\mathrm{tp}\cdot\mathrm{cp}$; \fitted{0.705} of it for a distributed-checkpoint teacher), Mamba activations (\fitted{4.99} bf16 copies of each Mamba layer's inner states, divided by $\mathrm{tp}\cdot\mathrm{cp}$), and \fitted{10.9}\gib{} of staging when the teacher is loaded from a Hugging Face checkpoint \\
Non-PyTorch memory, colocated $\sigma_{\mathrm{col}}$ & \fitted{12.27}\gib{}, the mean of the 2 cp4 runs; \fitted{3.1}\gib{} in the first iteration, the mean of 4 OOM events (3 Qwen3.5, 1 Nemotron) \\
Trainer stack, disaggregated (loss stage) & $S{+}\kappa_{\mathrm{str}}{+}A_{\mathrm{str}}{+}\sigma_{\mathrm{str}}$, with $A_{\mathrm{str}}=A_{\mathrm{gc}}{+}A_{\mathrm{bwd}}{+}A_{\mathrm{moe}}$; $\kappa_{\mathrm{str}}$ is \eqref{eq:kstr} (1.1\gib{} for the Qwen3.5 (2+1) trainer) \\
Trainer reference-cycle garbage $A_{\mathrm{gc}}$ & $\fitted{3}\,n_{\mathrm{mb}}\tfrac{r}{\mathrm{tp}}\,\ds\, b\bytes$ (12.0\gib{} at 64 microbatches) \\
Trainer loss-backward and recompute temporary memory $A_{\mathrm{bwd}}$ & $\bigl(\fitted{2.5}\cdot4\ds\Vloc+\fitted{6}\cdot4\,C\Vloc+34\,\ell_{\mathrm{rc}}\tfrac{r}{\mathrm{tp}}\ds\bigr)\bytes$ (8.6\gib{}); $\ell_{\mathrm{rc}}$: layers recomputed at once \\
Trainer expert-layer block $A_{\mathrm{moe}}$ (MoE students) & $\fitted{2.5}\,\ell^{\mathrm{moe}}_{\mathrm{last}}\,\tfrac{r}{\mathrm{tp}}\,\ds\, b\bytes$, $\ell^{\mathrm{moe}}_{\mathrm{last}}$ the expert layers of the loss stage (3.75\gib{} for Qwen3-30B-A3B) \\
Trainer remainder $\sigma_{\mathrm{str}}$ & \fitted{15.32}\gib{}, the mean remainder of 8 trainers \\
Producer shard & $(\fromcode{2}D_t/\mathrm{tp}_p+\beta_eE_t/(\mathrm{tp}_p\mathrm{ep}_p))\bytes$ of the whole checkpoint, $\beta_e{=}\fromcode{1}$ (fp8 experts) or \fromcode{2} (bf16) \\
Producer remainder $\varphi_p$ (forward-pass workspace and other memory) & \fitted{22.70}\gib{} (MoE: the mean residual of 6 producers), \fitted{4.4}\gib{} plus \fitted{12.8} bf16 buffers of $L\times d_t$ (dense: through the 2 dense producers) \\
Producer NCCL buffers $\omega_p$ & 1\gib{}, a flat allowance \\
\bottomrule
\end{tabular}
\caption{Per-module terms of the memory model ($k_{\mathrm{e}}$: routed experts per token, $h$: attention heads). \fromcode{Green} coefficients follow from the code or the framework; \fitted{red} coefficients are fitted to measurements; black numbers are allowances. Parameter terms are in exact GiB (Appendix~\ref{app:memory}).}
\label{tab:modules}
\end{table*}

\begin{table}[t]
\centering
\footnotesize
\setlength{\tabcolsep}{3pt}
\begin{tabular}{L{1.9cm}L{2.3cm}L{2.5cm}}
\toprule
Model & In the fit & Outside the fit \\
\midrule
Trainer device peak & 8 runs: RMS 1.2, largest 2.5 & 13 runs: RMS 1.7, largest 4.0 \\
Producer device peak & 8 runs: RMS 5.0, largest 8.4 & leave one MoE family out (6 runs): RMS 8.6, largest 12.3 \\
Colocated verdict & 6 configurations: 5 correct & 18 configurations: 17 correct \\
Colocated peak & 4 peaks, 4 OOM points: RMS 0.5, largest 1.2 & 9 peaks: RMS 1.5, largest 3.0; 2 OOM points, largest $-$15.8 \\
Inter-node traffic (Qwen3.5) & $-$ & within 2\% (3 and 4 nodes) \\
NVLink traffic (Qwen3.5) & $-$ & 4.4\% (colocated), 15\% (trainer), 3\% (producer) \\
\bottomrule
\end{tabular}
\caption{Accuracy of the estimation models: root-mean-square (RMS) and largest error (predicted minus measured, GiB) of the trainer and producer device peaks (Tables~\ref{tab:trainerruns} and~\ref{tab:producerruns}), where the producer's constants are fitted on all 8 runs and the second column predicts each MoE family from the others (the 2 dense runs fix their own constants); the colocated rows are verdicts (fits or OOM) and peaks or OOM points (Table~\ref{tab:brackets}), where the fit rows are those that set a constant; traffic is the deviation of the counters from the calculation (Tables~\ref{tab:nodes} and~\ref{tab:estimates}). The trainer's 13 runs outside the fit include the 2 runs with all 3 parts of the cache release (CR all; Appendix~\ref{app:cache}), of which the (2+1) reading, 79.6\gib{}, is at the device limit.}
\label{tab:accuracy}
\end{table}

\paragraph{The colocated loss buffers.}
ModelOpt's logits-KL loss materializes, at the last pipeline stage, 7 fp32-equivalent logits buffers, each $\in\mathbb{R}^{L/\mathrm{cp}\times V/\mathrm{tp}}$: the bf16 student and teacher logits (together 1 fp32-equivalent buffer), their 2 fp32 upcasts, the exponentials of the student logits that autograd retains for the backward pass, the student and teacher log-probabilities, and the elementwise KL tensor that is summed over the vocabulary. We fitted the count from the 2 colocated OOM events of the Qwen3.5 pair on 4 nodes (Table~\ref{tab:brackets}), whose OOM events fail at the 4th and the 7th buffer. The cp1 event failed on the fourth buffer, as did the 3-node event outside the fit, and the cp2 event on the seventh with 5 to 6 live, which gives the 7 live buffers under the allocation order listed. An eighth buffer stages the loss-averaging all-reduce, set from the Qwen3-235B 8-node cp2 OOM at its first call (Table~\ref{tab:brackets}). A run that reaches the backward pass also holds the gradient of the fp32 student logits and its bf16 copy, 1.5 buffers: fitting that count to the peak of each cp4 run gives 1.46 for both runs, and the cp8 run, outside the fit, is predicted within 0.7\gib{} (Table~\ref{tab:brackets}). Together they are the 9.5 buffers of \eqref{eq:colbuf}. The LM loss adds 1 more buffer, since the torch-allocated peak of the \ballk{} run is 3.8\gib{} above the KD-only run's and 1 buffer is 3.8\gib{}. The fused loss peaks 19.2\gib{} below the \sqlk{} run at cp4, 5.1 buffers, and is outside the count (Table~\ref{tab:main}). The streamed loss retains $\kappa_{\mathrm{str}}$ of \eqref{eq:kstr} instead, which per microbatch holds $r(\dt{+}\ds)$ hidden elements where the colocated loss holds $2r\,\Vloc$ logit elements (teacher and student), a reduction of $2\Vloc/(\dt{+}\ds)$: 40$\times$ for the Qwen3.5 pair and 1.3$\times$ for the dense Llama pair at $\mathrm{tp}{=}8$.

\paragraph{Units and parameter counts.}
Every memory term is in GiB ($2^{30}$ bytes): a parameter term is the count in billions times the bytes per parameter times $10^9/2^{30}$. A measured peak is quoted as the quantity that was logged (Megatron's decimal-GB readings converted to GiB): torch-allocated (the maximum allocated by PyTorch on the loss-stage ranks), device (the peak device memory in use, which adds the non-PyTorch memory), or, at an OOM error, the bytes allocated plus the bytes requested. The trainer's loss-stage parameter count sums the blocks of the stage, with the layers split evenly over the pipeline stages, the output layer ($V\ds/\mathrm{tp}$) on the last stage (a copy of the embedding when the embeddings are tied and $\mathrm{pp}>1$), and the Gated-DeltaNet, gated-attention, router, and Mamba parameters of each layer type; it reproduces the logged per-rank count of each of 7 trainers within 0.1\%. Totals $D_s$ and $E_s$ taken from the configuration alone undercount the Qwen3.5 student by 9\% (8.2 against 9.0 billion parameters, because a Gated-DeltaNet layer and the gated attention output count as plain attention) and the Qwen3-4B student by 1.1\gib{}, because its tied embedding is held on both end stages, so the counts come from the checkpoint tensor index and the block sums. The optimizer term shards over the combined DP$\times$CP group, as Megatron-Core's distributed optimizer does, so at $(\mathrm{dp},\mathrm{cp})=(1,4)$ the optimizer state per rank is 4 times smaller than a DP-only form gives. The Adam moments are allocated at the first optimizer step, so a run that stops in its first iteration holds $S$ without them, and its non-PyTorch memory is 3.1\gib{} (the mean of 4 OOM events, 2.5, 2.7, 3.2, and 4.0, at the failing rank with the most memory in use; the 4.0 is the Nemotron 4-node run with tp8/pp1/cp4/ep32), against 12.27\gib{} at steady state (the device peak minus the torch-allocated peak of the 2 cp4 runs, 12.65 and 11.88\gib{}). The student's routed experts are sharded over $\mathrm{pp}\cdot\mathrm{ep}$ only because the bridge pins their expert TP to 1, and the Qwen3.5 teacher's 2 key-value heads limit TP to 2 without head replication.

\paragraph{Producer-side terms.}
The producer holds the teacher's whole checkpoint, the vision tower and the multi-token-prediction block where the teacher has them. The shard term of \eqref{eq:producer} counts dense tensors at 2 bytes and routed experts at 1 byte when fp8 and 2 bytes when bf16, and reproduces the per-rank parameter count in the producer log to 0.15\% for the Qwen3.5, Qwen3-235B, Llama, and Qwen3-32B teachers and to 4.2\% for the Nemotron teacher, whose Mamba mixers are not fully tensor-parallel. The device peak lies 17.7 to 32.1\gib{} above the shard for the MoE teachers and 5.9 to 18.2\gib{} for the dense ones (Table~\ref{tab:producerruns}). The forward-pass term $\varphi_p$ is the mean of these remainders after the 1\gib{} of NCCL buffers $\omega_p$: 22.70\gib{} for the 6 MoE producers, and for the 2 dense producers the line through them, 4.4\gib{} plus 12.8 bf16 buffers of $L\times d_t$ ($L{=}32768$ for Llama, 4096 for Qwen3-32B). The line predicts the Qwen3-32B producers at $L{=}8192$ and 16384 (peaks 37.56 and 39.84\gib{}) to $-0.6$ and $-1.9$\gib{}; a single buffer count fitted on Llama underestimates all 3 Qwen3-32B runs by 4.3 to 5.6\gib{}. The 2 runs of each MoE family differ by at most 1.1\gib{}, and the remainder differs between families (21.4 for Qwen3.5, 18.1 for Qwen3-235B, and 31.6 for Nemotron), which a constant does not follow, so the estimate is too high for the Qwen3-235B producers and too low for the Nemotron producers.

\begin{table*}[t]
\centering
\footnotesize
\setlength{\tabcolsep}{4pt}
\begin{tabular}{L{3.2cm}L{2.4cm}rrrrrr}
\toprule
Run (T+P nodes) & Replica & Shard & Pred. & Meas. & Err. & Resid. & Err., family left out \\
\midrule
Qwen3.5 (2+1) & TP2, fp8 & 18.1 & 41.8 & 39.4 & $+2.4$ & 21.3 & $+3.5$ \\
Qwen3.5 (3+1) & TP2$\times$EP2, fp8 & 10.4 & 34.1 & 32.0 & $+2.1$ & 21.6 & $+3.2$ \\
Qwen3-235B (2+1) & TP8, bf16 & 54.7 & 78.4 & 72.4 & $+6.1$ & 17.7 & $+8.9$ \\
Qwen3-235B (4+4) & TP8, bf16 & 54.7 & 78.4 & 73.3 & $+5.2$ & 18.6 & $+8.0$ \\
Nemotron (2+2) & TP8$\times$EP2, fp8 & 38.6 & 62.3 & 69.6 & $-7.3$ & 31.0 & $-11.2$ \\
Nemotron (4+4) & TP8$\times$EP2, fp8 & 38.6 & 62.3 & 70.7 & $-8.4$ & 32.1 & $-12.3$ \\
Llama (8+8) & TP16, bf16 & 47.3 & 65.4 & 65.4 & $0.0$ & 18.2 & $-$ \\
Qwen3-32B (1+1) & TP2, bf16 & 30.5 & 36.4 & 36.4 & $0.0$ & 5.9 & $-$ \\
\bottomrule
\end{tabular}
\caption{Producer memory per GPU (GiB) of 8 of the 10 runs of Table~\ref{tab:runs}. Shard: the whole checkpoint's per-rank weights, dense tensors in bf16 and routed experts in fp8 or bf16 as listed (Appendix~\ref{app:memory}); Pred.: shard plus the forward-pass term (22.70 for MoE teachers; 4.4 plus 12.8 bf16 buffers of $L\times d_t$ for dense ones, at the sequence length of the run: 32768, and 4096 for Qwen3-32B) plus the 1 for NCCL buffers; Meas.: device peak; Err.: predicted minus measured; Resid.: measured minus shard; Err., family left out: the error when the forward-pass term is fitted on the other families (not defined for the 2 dense teachers, whose 2 parameters the 2 dense runs fix). The Qwen3-235B (2+1) row uses the 72.4 of an earlier run of that configuration; the run of Table~\ref{tab:runs} reads 72.9.}
\label{tab:producerruns}
\end{table*}

\paragraph{Fitted and held-out configurations.}
The colocated constants are set on 4 peak readings and 4 OOM points (Table~\ref{tab:brackets}). The readings are the device and torch-allocated peaks of the Qwen3.5 cp4 runs with the KD-only and \ballk{} losses; they set the non-PyTorch memory $\sigma_{\mathrm{col}}$. The OOM points are the 2 Qwen3.5 4-node events, which fix the 7-buffer count, and the 2 Nemotron 8-node events, which set the teacher's forward-pass terms and the 10.9\gib{} staging of the Hugging Face load. The trainer remainder $\sigma_{\mathrm{str}}$ is the mean remainder of the device peaks (the maximum over a run's trainer GPUs) of the 8 trainers with CR $-$ after the other terms of \eqref{eq:trainer} are subtracted, and the producer's $\varphi_p$ is the mean remainder of the 6 MoE producer peaks, with the 2 dense producers setting the intercept and the $L\times d_t$ buffer count (Table~\ref{tab:producerruns}). Table~\ref{tab:validation} compares prediction and measurement on the Qwen3.5 pair under both layouts, with the use of each row (fit or held out), and Table~\ref{tab:trainerruns} lists the terms, prediction, and measurement of each trainer. Outside the trainer fit are the 2 Qwen3-235B trainers with CR all, whose device readings are allocator ceilings, 7 Qwen3.5 trainers on other layouts (Appendix~\ref{app:primary}), the Qwen3-4B student at 3 further sequence lengths, and the Qwen3-30B student with the garbage collected after every microbatch, whose prediction was recorded before its run (Table~\ref{tab:trainerruns}).

\begin{table*}[t]
\centering
\footnotesize
\setlength{\tabcolsep}{4pt}
\begin{tabular}{L{4.4cm}lrrrrrrrrrl}
\toprule
Trainer & Grid & $n_{\mathrm{mb}}$ & $S$ & $\kappa_{\mathrm{str}}$ & $A_{\mathrm{gc}}$ & $A_{\mathrm{bwd}}$ & $A_{\mathrm{moe}}$ & Pred. & Meas. & Err. & Use \\
\midrule
Qwen3.5 (2+1) & 2/2/2/2 & 64 & 18.8 & 1.1 & 12.0 & 8.6 & 0.0 & 55.8 & 57.1 & $-$1.3 & fit \\
Qwen3.5 (3+1) & 2/2/2/3 & 42 & 16.7 & 1.1 & 7.9 & 8.6 & 0.0 & 49.6 & 50.3 & $-$0.7 & fit \\
Qwen3.5 (2+1), tp4 & 4/2/2/1 & 128 & 12.5 & 0.7 & 12.0 & 4.3 & 0.0 & 44.9 & 45.1 & $-$0.2 & fit \\
Qwen3-32B (1+1), seq 4096 & 2/2/1/2 & 64 & 12.3 & 1.2 & 1.9 & 10.1 & 0.0 & 40.8 & 40.3 & +0.5 & fit \\
Llama (8+8) & 8/2/2/2 & 64 & 37.0 & 2.0 & 6.0 & 3.2 & 0.0 & 63.5 & 61.0 & +2.5 & fit \\
Nemotron (2+2) & 2/2/1/4/4 & 32 & 47.4 & 2.8 & 7.9 & 4.5 & 2.5 & 80.4 & 79.6 & +0.8 & fit \\
Nemotron (4+4) & 2/2/1/8/4 & 16 & 35.9 & 2.8 & 3.9 & 4.5 & 2.5 & 65.0 & 65.4 & $-$0.4 & fit \\
Qwen3-235B (4+4), ep8 & 2/2/1/8/8 & 16 & 23.0 & 1.6 & 3.0 & 9.5 & 3.8 & 56.1 & 57.4 & $-$1.3 & fit \\
Qwen3-235B (2+1) & 2/2/1/4/4 & 32 & 43.8 & 1.3 & 6.0 & 9.5 & 3.8 & 79.6 & 79.6 & 0.0 & ceiling \\
Qwen3-235B (4+4) & 2/2/1/8/4 & 16 & 33.1 & 1.3 & 3.0 & 9.5 & 3.8 & 66.0 & 67.8 & $-$1.8 & ceiling \\
Qwen3.5 (4+2) & 2/2/2/4 & 32 & 15.6 & 1.1 & 6.0 & 8.6 & 0.0 & 46.7 & 47.5 & $-$0.8 & held out \\
Qwen3.5 (4+2), pp4 & 2/4/2/2 & 64 & 11.5 & 1.1 & 12.0 & 8.6 & 0.0 & 48.6 & 49.5 & $-$0.9 & held out \\
Qwen3.5 (4+2), cp4 & 2/2/4/2 & 64 & 15.6 & 0.9 & 6.0 & 8.1 & 0.0 & 46.0 & 48.3 & $-$2.4 & held out \\
Qwen3.5 (3+1), seq 8192 & 2/2/2/3 & 168 & 16.7 & 0.8 & 7.9 & 7.8 & 0.0 & 48.5 & 48.3 & +0.2 & held out \\
Qwen3.5 (4+2), pp4/cp4, seq 8192 & 2/4/4/1 & 512 & 11.5 & 0.8 & 12.0 & 7.7 & 0.0 & 47.3 & 49.1 & $-$1.8 & held out \\
Qwen3.5 (2+1), collected & 2/2/2/2 & 64 & 18.8 & 1.1 & 0.2 & 8.6 & 0.0 & 44.0 & 44.8 & $-$0.8 & held out \\
Qwen3.5 (4+1), pp4/cp4 & 2/4/4/1 & 128 & 11.5 & 0.9 & 12.0 & 8.1 & 0.0 & 47.8 & 51.8 & $-$4.0 & held out \\
Qwen3-32B (1+1), seq 8192 & 2/2/1/2 & 32 & 12.3 & 1.3 & 1.9 & 11.4 & 0.0 & 42.3 & 42.6 & $-$0.4 & held out \\
Qwen3-32B (1+1), seq 16384 & 2/2/1/2 & 16 & 12.3 & 1.6 & 1.9 & 14.1 & 0.0 & 45.2 & 45.2 & 0.0 & held out \\
Qwen3-32B (1+1), seq 32768 & 2/2/1/2 & 8 & 12.3 & 2.0 & 1.9 & 19.4 & 0.0 & 50.9 & 49.0 & +2.0 & held out \\
Qwen3-235B (2+1), collected & 2/2/1/4/4 & 32 & 43.8 & 1.6 & 0.2 & 9.5 & 3.8 & 74.1 & 75.3 & $-$1.2 & held out \\
\bottomrule
\end{tabular}
\caption{Terms of the trainer memory model \eqref{eq:trainer} and the measured device peak of the trainer runs the model uses (GiB). Grid: tp/pp/cp/dp, with ep appended for MoE students; sequence length 32768 unless stated; $n_{\mathrm{mb}}$: microbatches per DP replica. Pred.: the 5 terms shown plus $\sigma_{\mathrm{str}}{=}15.32$; Meas.: device peak, the maximum over the trainer GPUs; Err.: predicted minus measured. Use: fit (the 8 runs whose mean remainder is $\sigma_{\mathrm{str}}$), held out (outside that mean; the term forms were chosen with their residuals visible), or ceiling (CR all, Table~\ref{tab:cachesplit}: the device reading is an allocator ceiling). Collected: Python's garbage collected after every microbatch. The Qwen3-235B pair rows are the 30B student; the Qwen3-235B (2+1), collected row was predicted at 73.8 before its run, 0.3 below this value, which adds the fp8 head copy of its configuration. Teacher precision: the Qwen3-235B (4+4), ep8 and (2+1), collected rows are fp8-teacher runs outside Table~\ref{tab:runs}; the 2 ceiling rows are bf16.}
\label{tab:trainerruns}
\end{table*}

\begin{table}[t]
\centering
\footnotesize
\setlength{\tabcolsep}{3pt}
\begin{tabular}{L{2.75cm}rl@{\hspace{4pt}}l}
\toprule
Configuration & Pred. & Measured & Use \\
\midrule
Coloc., 3 nodes, cp1 & 86.8 & OOM, 72.9+15.2 & held out \\
Coloc., 4 nodes, cp1 & 86.4 & OOM, 72.4+15.2 & fit \\
Coloc., 4 nodes, cp2 & 78.3 & OOM, 71.4+7.6 & fit \\
Coloc., 4 nodes, cp4 & 75.8 & 76.1 & fit \\
Coloc., 4 nodes, cp4, \ballk{} & 79.6 & 79.1 & fit \\
Coloc., 4 nodes, cp8 & 77.3 & 77.0 & held out \\
Coloc., 4 nodes, cp4, \sqlk{} & $-$ & 76.8 & outside \\
Coloc., 4 nodes, cp4, \sqlk{} fused & $-$ & 58.7 & outside \\
Coloc., 3 nodes, cp1, \sqlk{} fused & $-$ & OOM, 57.7+15.2 & outside \\
Coloc., 3 nodes, cp1, \sqlk{} fused-W & $-$ & 78.9 & outside \\
Coloc., 2 nodes, cp2, \sqlk{} fused-W & $-$ & 65.8 & outside \\
Coloc., 2 nodes, cp1, \sqlk{} fused-W & $-$ & OOM, 64.9+0.5 & outside \\
Disagg. (2+1), trainer & 55.8 & 57.1 & fit \\
Disagg. (3+1), trainer & 49.6 & 50.3 & fit \\
Disagg. (tp4), trainer & 44.9 & 45.1 & fit \\
Disagg. (4+1, pp4, cp4), trainer & 47.8 & 51.8 & held out \\
Disagg. (2+1, collected), trainer & 44.0 & 44.8 & held out \\
Disagg. Qwen3-30B (2+1, collected), trainer & 74.1 & 75.3 & held out \\
Disagg. (2+1), producer & 41.8 & 39.4 & fit \\
Disagg. (3+1), producer & 34.1 & 32.0 & fit \\
\bottomrule
\end{tabular}
\caption{Predicted (Pred.) memory demand and measured peak per GPU (GiB) on the Qwen3.5 pair (and 1 Qwen3-30B trainer); Use: whether the row set a constant of the model (fit), not (held out), or is outside the model because its loss is not ModelOpt's (outside). OOM rows compare the predicted demand at the failing allocation with the measured bytes allocated plus requested; all other rows are device peaks (torch-allocated trainer peak 39.2). The colocated cp4 rows fix the non-PyTorch memory and the logits-buffer count, the trainer rows the trainer remainder $\sigma_{\mathrm{str}}$, and the producer rows the forward-pass term $\varphi_p$ (Appendix~\ref{app:memory}); ``collected'' is the (2+1) trainer with Python's garbage collected after every microbatch. Tables~\ref{tab:brackets}, \ref{tab:trainerruns}, and~\ref{tab:producerruns} list all calibration points.}
\label{tab:validation}
\end{table}

\paragraph{Colocated verdicts and peaks.}
The model classifies 17 of the 18 colocated configurations outside its fit correctly and 5 of the 6 that set constants (OOM or fits; Table~\ref{tab:brackets}). The 2 exceptions are the Qwen3.5 \ballk{} run, a fit row that peaks at 79.1\gib{} and is predicted 79.6, and the Nemotron run on 4 nodes at tp8/pp1/cp4, which ran OOM in the teacher's first forward pass although the model predicts a device peak of 63.0\gib{}. The measured demand at the failing allocation is 75.7\gib{} against 59.9 predicted, and the model has no second tp 8 teacher run to calibrate the teacher's forward-pass terms. Outside the fit, the model predicts the peaks of 9 runs to within 3.0\gib{} (Table~\ref{tab:brackets}); the largest errors, $-3.0$ and $-2.8$\gib{}, are the Qwen3-235B runs on 2 and 4 nodes, which ended with a NaN gradient in their first iteration, before the optimizer step: the model evaluates them without Adam moments and with the 3.1\gib{} non-PyTorch memory of that iteration.

\paragraph{The Llama colocated bound.}
For the 405B-to-70B pair the model evaluates every colocated grid from 8 to 32 nodes with TP at most 8, the student's number of key-value (KV) heads, and PP dividing 2 (the layer counts are 80 and 126). The lowest-peak grid at 8, 16, 24, 28, and 32 nodes is predicted at 89 to 102\gib{}, over the budget, and the other node counts have higher minima or no legal grid, so none was run (Table~\ref{tab:brackets}). Megatron-Core also allows TP above the KV head count by replicating the heads, as the producer does at 16. With $\mathrm{tp}{=}32$ the colocated layout fits on 16 nodes and on 8, and the model predicts both within 1.2\gib{} of the torch-allocated peak without any Llama measurement in the colocated fit (the Llama trainer and producer are in their fits); node counts below 8 were not run.

\paragraph{Cache-release ceilings.}
The Qwen3-235B (2+1) trainer with CR all reserves 68.1\gib{} and holds 11.5\gib{} outside the allocator, with a device peak of 79.6 (Table~\ref{tab:trainerledger}), so its device reading is a ceiling; the (4+4) reading with CR all, 67.8, is a feasibility check for the same reason.

\paragraph{Trainer terms.}
3 terms account for most of the spread of the trainer's remainder, the device peak after $S$ and $\kappa_{\mathrm{str}}$ (Table~\ref{tab:trainerterms}); Table~\ref{tab:trainerledger} gives the measured memory of each trainer and Table~\ref{tab:trainerprobes} the probes of the terms. (1) Reference-cycle garbage. The trainer collects Python garbage at the iteration boundary only, so the allocation of the loss stage grows with every microbatch (Ramp in Table~\ref{tab:trainerledger}). The growth per microbatch is 0.70 to 0.97 of 3 sequence-parallel shards of the bf16 hidden states, $3\,(r/\mathrm{tp})\,\ds\, b$ bytes, on the 6 trainers where it was logged. On the Qwen3.5 (2+1) trainer, collecting after every microbatch lowers the end-of-iteration allocation by 11.8\gib{}, as the model gives, and each iteration takes 57\% longer (Table~\ref{tab:trainerprobes}). (2) Temporary memory of the loss backward and of the recomputed layers. The fp32 head gradient has $4\ds\Vloc$ bytes, and the fused backward holds it in 2 to 3 copies (the accumulator, the matmul temporary, and the bf16 cast) next to 5 to 8 fp32 window tensors $\in\mathbb{R}^{C\times\Vloc}$ (the forward scan holds 1 per distribution), depending on the order of operations; we use 2.5 and 6, and over that range the standard deviation (SD) of the remainder is 1.1 to 1.6\gib{}. The temporary memory measured on the Qwen3-4B student at TP 8 agrees with the terms (Table~\ref{tab:trainerprobes}). Recomputing $\ell_{\mathrm{rc}}$ layers at once holds $34\,(r/\mathrm{tp})\ds$ bytes per layer \citep{korthikanti2023reducing}, the count of a dense transformer, used for every student ($\ell_{\mathrm{rc}}{=}1$ except in the 4 Qwen3-32B-pair trainer runs, where $\ell_{\mathrm{rc}}{=}8$). (3) The expert-layer block of MoE students. Across the optimizer step their first pipeline stage holds more than its statics (Stage 1 in Table~\ref{tab:trainerledger}). An allocation dump of the Qwen3-30B trainer (1 run) attributes this to tensors that stay live across the boundary: the retained autograd graph, the bias-dropout activations, and the dispatch row maps and workspaces (Table~\ref{tab:trainerprobes}), about 2.5 bf16 copies of $[r/\mathrm{tp},\ds]$ per expert layer. The model applies 2.5 such copies to the expert layers of the loss stage ($A_{\mathrm{moe}}$); the 2.5 is calibrated on 1 Qwen3-30B dump, at tp 2 and cp 1. The remaining $\sigma_{\mathrm{str}}$ is memory outside PyTorch and unallocated memory that the caching allocator holds reserved (Table~\ref{tab:trainerledger}), and it contains the layer inputs saved for the recomputation, which no other term counts (1.0\gib{} on the loss stage at (2+1)). On the Qwen3.5 tp-4 trainer, turning NVLink SHARP (NVLS) multicast off changes the device peak by $-$0.3\gib{} and loading kernels eagerly by $+$1.7 (Table~\ref{tab:trainerprobes}); neither probe was run on another family. Memory outside PyTorch differs between families by more than the residuals do; the 8 trainers come from 4 student families, which do not separate family from model dimensions. At cp 4 the held-out Qwen3.5 trainers hold more memory outside PyTorch than at cp 2 (Table~\ref{tab:trainerledger}) and are the 3 most under-predicted of the 7 (Table~\ref{tab:trainerruns}).

\begin{table}[t]
\centering
\footnotesize
\setlength{\tabcolsep}{4pt}
\begin{tabular}{L{3.6cm}rrrr}
\toprule
Remainder after subtracting & Mean & SD & LOO & LOFO \\
\midrule
$S$, $\kappa_{\mathrm{str}}$ & 29.91 & 4.7 & 5.1 & 5.7 \\
\quad and $A_{\mathrm{gc}}$ & 23.09 & 4.1 & 4.4 & 4.5 \\
\quad and $A_{\mathrm{bwd}}$ & 16.41 & 2.3 & 2.4 & 2.4 \\
\quad and $A_{\mathrm{moe}}$ & 15.32 & 1.3 & 1.4 & 1.4 \\
\bottomrule
\end{tabular}
\caption{Remainder of the device peaks (GiB) of the 8 fitted trainers after the terms listed are subtracted cumulatively: its mean, SD, and the RMS error of predicting each trainer by the mean of the others (LOO) or of the other model families (LOFO). The last row is $\sigma_{\mathrm{str}}$ and its spread.}
\label{tab:trainerterms}
\end{table}

\begin{table*}[t]
\centering
\footnotesize
\setlength{\tabcolsep}{4pt}
\begin{tabular}{L{3.4cm}rrrrrrrrr}
\toprule
Trainer & Device & Torch & Reserved & Non-PyTorch & Unallocated & End of iter. & Stage 1 & Ramp & $A_{\mathrm{gc}}$ \\
\midrule
Qwen3.5 (2+1) & 57.1 & 39.2 & 43.1 & 13.9 & 3.8 & 31.8 & +0.3 & 11.6 & 12.0 \\
Qwen3.5 (3+1) & 50.3 & 33.0 & 36.6 & 13.4 & 3.6 & 25.6 & $-$ & $-$ & 7.9 \\
Qwen3.5 (2+1), tp4 & 45.1 & 29.0 & 31.1 & 14.0 & 2.1 & 25.1 & $-$ & 11.7 & 12.0 \\
Qwen3-32B (1+1) & 40.3 & 25.1 & 28.8 & 11.5 & 3.7 & 15.6 & +0.2 & 1.7 & 1.9 \\
Llama (8+8) & 61.0 & 50.3 & 52.7 & 7.8 & 2.4 & 43.8 & +0.2 & $-$ & 6.0 \\
Nemotron (2+2) & 79.6 & 67.5 & 69.8 & 9.7 & 3.6 & 59.3 & +2.2 & $-$ & 7.9 \\
Nemotron (4+4) & 65.4 & 52.5 & 54.2 & 9.5 & 4.2 & 43.8 & +2.2 & $-$ & 3.9 \\
Qwen3-235B (4+4), ep8 & 57.4 & 40.7 & 46.0 & 11.3 & 5.3 & 30.5 & $-$ & 2.6 & 3.0 \\
Qwen3-235B (2+1), CR all & 79.6 & 64.3 & 68.1 & 11.5 & 3.8 & 54.1 & +3.7 & 5.1 & 6.0 \\
Qwen3-235B (4+4), CR all & 67.8 & 50.7 & 56.1 & 11.2 & 5.5 & 40.4 & +3.7 & 2.1 & 3.0 \\
\midrule
Qwen3.5 (4+2) & 47.5 & 30.1 & 33.6 & 13.9 & 3.5 & 22.7 & $-$ & $-$ & 6.0 \\
Qwen3.5 (4+2), pp4 & 49.5 & 31.5 & 35.2 & 14.3 & 3.7 & 24.6 & $-$ & $-$ & 12.0 \\
Qwen3.5 (4+2), cp4 & 48.3 & 29.5 & 32.4 & 15.9 & 2.9 & 22.7 & $-$ & $-$ & 6.0 \\
Qwen3.5 (3+1), seq 8192 & 48.3 & 32.1 & 34.9 & 13.4 & 2.8 & 25.6 & $-$ & $-$ & 7.9 \\
Qwen3.5 (4+2), pp4/cp4, seq 8192 & 49.1 & 30.8 & 33.2 & 15.8 & 2.4 & 24.6 & $-$ & $-$ & 12.0 \\
Qwen3.5 (2+1), collected & 44.8 & 27.7 & 30.8 & 13.9 & 3.1 & 20.0 & $-$ & $-$ & 0.2 \\
Qwen3.5 (4+1), pp4/cp4 & 51.8 & 31.1 & 36.7 & 15.2 & 5.6 & 24.6 & $-$ & $-$ & 12.0 \\
\bottomrule
\end{tabular}
\caption{Measured memory of the 8 fitted trainers, the 2 runs with CR all, and the 7 Qwen3.5 trainers outside the $\sigma_{\mathrm{str}}$ fit (below the rule; GiB), on the last pipeline stage. Device: peak device memory, the maximum over the trainer GPUs; Torch: peak allocated (the maximum over the trainer GPUs for the 2 Nemotron rows); Reserved: peak reserved; Non-PyTorch: device memory minus reserved on the last-stage rank; Unallocated: reserved minus peak allocated (Device is a maximum over the trainer GPUs and the other columns are logged on 1 last-stage rank, so a row need not add up); End of iter.: allocation after the iteration; Stage 1: steady allocation of the first pipeline stage minus its statics (not logged for 3 runs and the held-out trainers); Ramp: growth of the allocation over the microbatches of an iteration, from the per-microbatch memory log (not logged for 4 runs and the held-out trainers); $A_{\mathrm{gc}}$: the model's garbage term.}
\label{tab:trainerledger}
\end{table*}

\begin{table*}[t]
\centering
\footnotesize
\setlength{\tabcolsep}{4pt}
\begin{tabular}{L{3.5cm}L{2.4cm}rrrrrL{2.6cm}}
\toprule
Probe & Trainer grid & s/it & Torch & Device & NCCL & Comms & Measured; model \\
\midrule
Qwen3.5 (2+1), garbage kept & tp2/pp2/cp2/dp2 & 74.3 & 39.2 & 57.1 & $-$ & $-$ & end of iter. 31.8 \\
Qwen3.5 (2+1), collected & tp2/pp2/cp2/dp2 & 116.9 & 27.7 & 44.8 & $-$ & $-$ & end of iter. 20.0; $-$11.8 \\
Qwen3.5 (2+1), tp4, reference & tp4/pp2/cp2/dp1 & 97.0 & 29.0 & 45.1 & 59.6 (3.73) & 17 & $-$ \\
\quad NVLS multicast off & tp4/pp2/cp2/dp1 & 98.2 & 29.0 & 44.8 & 59.6 (3.72) & 17 & $-$ \\
\quad eager kernel loading & tp4/pp2/cp2/dp1 & 100.2 & 29.0 & 46.8 & 59.6 (3.72) & 17 & $-$ \\
Qwen3-32B (1+1), tp2 & tp2/pp2/cp1/dp2 & $-$ & 25.1 & $-$ & 38.2 (4.78) & 19 & $-$ \\
Qwen3-32B (1+1), tp8 & tp8/pp1/cp1/dp1 & 34.7 & 11.8 & 23.2 & 23.2 (2.90) & 19 & temporary memory 2.56; 2.52 \\
Qwen3-235B (2+1), collected & tp2/pp2/cp1/dp4, ep4 & 402--405 & 59.6 & 75.3 & $-$ & $-$ & retained set 3.57; 3.75; statics 43.75; 43.76 \\
\bottomrule
\end{tabular}
\caption{Probes of the trainer memory terms (GiB; $-$: not measured), all with CR $-$ (Table~\ref{tab:cachesplit}). Torch: peak allocated; Device: peak device memory; NCCL: all communicator buffers over the trainer ranks, per rank in parentheses; Comms: communicators per rank. Collected: Python's garbage collected after every microbatch, which lowers the end-of-iteration allocation by 11.8 GiB against the garbage-kept run, equal to $A_{\mathrm{gc}}$ of the same grid (calibrated on its ramp). Temporary memory: peak minus end-of-iteration allocation against the term $A_{\mathrm{bwd}}$ (Qwen3-4B student); retained set: tensors alive across the optimizer step (autograd graph 1.50, bias-dropout activations 1.44, dispatch maps and workspaces 0.63) against $A_{\mathrm{moe}}$ (the family that calibrates its 2.5 copies); statics: the allocation dump against $S$. Qwen3-235B (2+1), collected is the 30B student with the 235B teacher.}
\label{tab:trainerprobes}
\end{table*}

\subsection{Traffic Model}
\label{app:traffic}
\paragraph{Data-movement derivations.}
Table~\ref{tab:traffic} lists the volume of each parallelism dimension with its Qwen3.5 values ($\dt{=}2048$, $\ds{=}4096$, $V{=}248320$, $\ell_t{=}40$, $\ell_s{=}32$, $\tau{=}2^{20}$), and Table~\ref{tab:estimates} sums them with the transfer. The tables hide 3 points: (1) in the Qwen3.5 architecture every 4th layer is full attention: CP exchanges the keys and values of those layers around a ring \citep{liu2024ringattention}, and the remaining linear-attention layers redistribute their input projection from sequence to head sharding, and their output back, by all-to-all, as Megatron-Bridge implements it; (2) gradients are reduced over the DP$\times$CP group, so the student's gradient traffic is nonzero at $\mathrm{dp}{=}1$ when $\mathrm{cp}>1$; (3) a ring over a group that spans several nodes sends $(n{-}1)/n$ of its bytes over InfiniBand at each node boundary of the ring, and the transfer multiplies the hidden-state volume by $\mathrm{tp}$ because the sender transmits the rows of each cp rank to its $\mathrm{tp}$ last-stage ranks.

\begin{table*}[t]
\centering
\footnotesize
\setlength{\tabcolsep}{2.5pt}
\renewcommand{\arraystretch}{1.2}
\begin{tabular}{L{1.7cm}L{4.4cm}L{4.9cm}rr}
\toprule
 & & & \multicolumn{2}{c}{Qwen3.5 (GB; student; teacher)} \\
Dimension & Teacher (forward) & Student (training) & colocated & disaggregated \\
\midrule
Tensor & $2\tfrac{\ell_t}{\mathrm{pp}}\tfrac{2(\mathrm{tp}-1)}{\mathrm{tp}}\,\tau\,\dt b\bytes$ & $6\tfrac{\ell_s}{\mathrm{pp}}\tfrac{2(\mathrm{tp}-1)}{\mathrm{tp}}\,\tau\,\ds b\bytes$ & 412.3; 85.9 & 824.6; 343.6 \\
Pipeline, per boundary and direction & $n_{\mathrm{mb}}\tfrac{L}{\mathrm{cp}}\tfrac{\dt b}{\mathrm{tp}}\bytes$, both directions colocated & $n_{\mathrm{mb}}\tfrac{L}{\mathrm{cp}}\tfrac{\ds b}{\mathrm{tp}}\bytes$, both directions & 4.29; 2.15 & 4.29; 0 \\
Context, attention & $\tfrac{\ell^{\mathrm{att}}_t}{\mathrm{pp}}(\mathrm{cp}{-}1)\,\tau\,2\tfrac{h_{\mathrm{kv},t}}{\mathrm{tp}}d_{\mathrm{head}} b\bytes$ & $4\tfrac{\ell^{\mathrm{att}}_s}{\mathrm{pp}}(\mathrm{cp}{-}1)\,\tau\,2\tfrac{h_{\mathrm{kv},s}}{\mathrm{tp}}d_{\mathrm{head}} b\bytes$ & 51.5; 8.1 & 34.4; 0 \\
Context, linear attention & $\tfrac{\ell^{\mathrm{lin}}_t}{\mathrm{pp}}\tfrac{\mathrm{cp}-1}{\mathrm{cp}}\,\tau\,\tfrac{w}{\mathrm{tp}}b\bytes$ & $3\tfrac{\ell^{\mathrm{lin}}_s}{\mathrm{pp}}\tfrac{\mathrm{cp}-1}{\mathrm{cp}}\,\tau\,\tfrac{w}{\mathrm{tp}}b\bytes$ & 232.8; 97.0 & 310.5; 0 \\
Data and context & none & $\tfrac{\mathrm{dp}\,\mathrm{cp}-1}{\mathrm{dp}\,\mathrm{cp}}P_{\mathrm{loc}}\,(4{+}2)\bytes$ & 5.0; 0 & 10.1; 0 \\
Expert & $\tfrac{\ell^{\mathrm{moe}}_t}{\mathrm{pp}}\,\tau\, k_{\mathrm{e}}\dt\bigl[\tfrac{2(\mathrm{ep}-1)}{\mathrm{ep}}b+\tfrac{\mathrm{etp}-1}{\mathrm{etp}}(b{+}4)\bigr]\bytes$ & $3\tfrac{\ell^{\mathrm{moe}}_s}{\mathrm{pp}}\tfrac{2(\mathrm{ep}-1)}{\mathrm{ep}}\,\tau\, k_{\mathrm{e}}\,\ds b\bytes$ & 0; 0 & 0; 2061.6 \\
\bottomrule
\end{tabular}
\caption{Bytes sent per GPU and iteration by parallelism dimension, with the values on the Qwen3.5 pair (4 nodes colocated, 3 disaggregated; the disaggregated teacher is the producer). A dimension of degree $n$ and rank stride $\Delta$ spans $n\Delta$ ranks and uses NVLink if $n\Delta\le8$ (the GPUs of a node) and InfiniBand at its node boundaries otherwise; the strides are 1 (tensor), tp (context, and data at ep${=}1$), $\mathrm{tp}\,\mathrm{cp}$ (expert), and $\mathrm{tp}\,\mathrm{cp}\,\mathrm{ep}\,\mathrm{dp}$ (pipeline). The factors 2 and 6 are 2 all-reduces per layer times the teacher's 1 forward pass or the student's 3 passes (forward, recompute, backward); 3 and 4 count passes (the ring attention \citep{liu2024ringattention} backward exchanges keys and values and their gradients, hence 4). $\ell^{\mathrm{att}}$, $\ell^{\mathrm{lin}}$, and $\ell^{\mathrm{moe}}$ count full-attention, linear-attention, and MoE layers, $h_{\mathrm{kv}}$ the KV heads, $d_{\mathrm{head}}$ the head width, and $w{=}16{,}448$ the per-token width of the linear-attention exchange of the Qwen3.5 layers (12,352 for the input projection of $q,k,v$, gate, and 2 scalars per head, 4,096 for the output). The expert term counts every routed copy of a token: the all-to-all over the EP group and, at $\mathrm{etp}>1$, the dispatch all-gather in bf16 and the combine reduce-scatter in fp32 over the expert tensor-parallel group; it is 0 at $\mathrm{ep}{=}1$ and $\mathrm{etp}{=}1$. Gradients are reduced in fp32 (4 bytes) and parameters gathered in bf16 (2 bytes).}
\label{tab:traffic}
\end{table*}

\paragraph{Traffic validation.}
Table~\ref{tab:nodes} compares the calculated inter-node bytes of every node with its measured InfiniBand traffic (for the disaggregated runs the node's total rate over its adapters, averaged over the 5\,s samples from iteration 3, checkpoint iterations excluded, times the mean iteration time; for the colocated run the byte counts of the 8 adapters that carry the traffic, read at the ends of iterations 2 and 12; Table~\ref{tab:estimates} divides the bytes by 90.5\,s for colocated KD-only, the median iteration time of its 200-iteration run, and by 76.0\,s, iterations 3 to 400, for the 400-iteration run of the (2+1) split that carried the counters (not a row of Table~\ref{tab:runs}), whose measured rate is the mean received rate from iteration 3), placing ranks in Megatron's order (tp, cp, dp, pp) on 8-GPU nodes. On the colocated run each node is 1 pipeline stage, only the pipeline crosses InfiniBand, and every counter is within 1.4\% of the calculation; ModelOpt concatenates the teacher's hidden width to the student's in the pipeline tensor, so the teacher's activation-sized bytes cross each boundary in both directions. At 3 nodes the tensor, context, and data groups stay inside the nodes, so only the pipeline activation transfers and the transfer cross InfiniBand, and every counter is within 2\% of the calculation. At 4 nodes the ring of each DP$\times$CP group also spans nodes, and its gradient and parameter traffic is 1.0 to 2.0 times the pipeline's; the calculation stays within 2\% and below every counter. NVLink counters (bytes sent per GPU over iterations 3 to 12, and 3 to 10 for the disaggregated trainer and producer) lie above the calculation on the colocated run (4.4\%, with the 4 pipeline stages within 3\% of each other; the counter snapshots lie within 1 iteration of the stated boundaries), the disaggregated trainer, and the producer (Table~\ref{tab:estimates}; Table~\ref{tab:accuracy} gives the deviations). The producer's tensor term, 344\,GB, is within 1.4\% of the volume of its small all-gather and reduce-scatter pairs (339\,GB); the other 86\% of the calculated bytes are the expert term, the dispatch all-gather and the fp32 combine reduce-scatter over the expert tensor-parallel group, which move the top-8-replicated tokens of every MoE layer, and the combine reduce-scatter alone carries 57\% of the calculated collective bytes. The NVLink counter reproduced a 2-GPU all-reduce, a 4-GPU all-to-all, and a send/receive of known size within 0.4\% on 1 node; the InfiniBand counters were not calibrated this way.

\begin{table}[t]
\centering
\footnotesize
\setlength{\tabcolsep}{3pt}
\begin{tabular}{L{2.0cm}rrrrr}
\toprule
 & \multicolumn{2}{c}{Sent} & \multicolumn{2}{c}{Received} & Sent: \\
Node & calc. & meas. & calc. & meas. & pipe.+ring \\
\midrule
\multicolumn{6}{l}{3 nodes (dp2)} \\
Stage 0 & 34.4 & 34.8 & 34.4 & 34.8 & 34.4+0 \\
Last stage & 34.4 & 34.9 & 68.7 & 69.6 & 34.4+0 \\
Producer & 34.4 & 34.8 & 0 & 0.1 & \\
\multicolumn{6}{l}{4 nodes (dp3)} \\
Node 0 & 44.9 & 45.5 & 44.9 & 45.4 & 22.5+22.4 \\
Node 1 & 67.3 & 68.0 & 78.6 & 79.4 & 22.5+44.8 \\
Node 2 & 44.9 & 45.1 & 67.5 & 67.9 & 22.5+22.4 \\
Producer & 33.8 & 34.2 & 0 & 0.1 & \\
\multicolumn{6}{l}{4 nodes colocated (pp4, dp1)} \\
Stage 0 & 51.5 & 52.1 & 51.5 & 52.2 & 51.5+0 \\
Stage 1 & 103.1 & 104.3 & 103.1 & 104.2 & 103.1+0 \\
Stage 2 & 103.1 & 104.2 & 103.1 & 104.2 & 103.1+0 \\
Last stage & 51.5 & 52.1 & 51.5 & 52.1 & 51.5+0 \\
\bottomrule
\end{tabular}
\caption{Inter-node traffic per node and iteration (GB): calculated and measured, and the calculated split of the sent bytes into pipeline activation transfers and DP$\times$CP rings. Received bytes at the trainer's last-stage nodes include the transfer. Colocated: 1 pipeline stage per node, iterations 3 to 12.}
\label{tab:nodes}
\end{table}

\section{Method Details}

\subsection{Streaming Pipeline}
\label{app:pipeline}
\paragraph{Receiver groups and replication.}
A receiver group is the $\mathrm{tp}\times\mathrm{cp}$ last-stage ranks of one DP replica, so a trainer has $\mathrm{dp}$ receiver groups. Rows are divided over the cp ranks in Megatron's load-balanced layout for causal attention \citep{grattafiori2024llama3}: cp rank $q$ owns pieces $q$ and $2\,\mathrm{cp}-1-q$ of $2\,\mathrm{cp}$ equal pieces of the sequence (the layout that ring-attention implementations call zigzag), and slots are sent in bf16. The producer serves receiver groups in 2 layouts. In the \emph{dedicated} layout (replicas $\ge$ groups, a multiple of dp) each replica serves exactly 1 group and each group is served by $R/\mathrm{dp}$ replicas; in the \emph{shared} layout (replicas $<$ groups) every replica joins every group and sends 1 message per slot per destination rank. Slot $\eta$ goes to group $\eta\bmod n_{\mathrm{groups}}$ from replica $\eta\bmod R$ at microbatch size 1. The Qwen3.5 3-node run is dedicated and its 4-node run shared.

\paragraph{Communication protocol.}
Each receiver group has its own NCCL communicator containing the group's last-stage ranks and the producer ranks mapped to it, so each communicator contains only the ranks of 1 receiver group and its producer ranks. A slot is an int64 header carrying the slot id followed by the bf16 payload $\in\mathbb{R}^{r\times d_t}$, already sliced to the cp rank's share of the rows; per-pair FIFO order and the header check make ordering structural, and the trainer raises on any slot-id mismatch. The producer runs at most $\delta$ iterations ahead of the trainer's consumption counter, receive buffers are posted at request creation, and the prefetch ring holds up to 4 slots per last-stage rank by default. Only a 128-byte communicator id per receiver group crosses the rendezvous store, and every hidden-state byte flows over InfiniBand P2P. The lm\_head weights are written once to a file in the stream directory, and $V$ must be divisible by $\mathrm{tp}$. The trainer's step for 1 microbatch is Algorithm~\ref{alg:step}, and the producer loop is Algorithm~\ref{alg:producer}.

\paragraph{Resume.}
The producer takes a start slot so that it neither recomputes nor re-sends the prefix the trainer has consumed, and logs the start slot; unit tests check the slot arithmetic of a resume. The window width $C$ does not change the loss in exact arithmetic; it sets the size of the per-window buffers, which scale with $C\,\Vloc$.

\paragraph{Pace-limiting side.}
\label{app:pacing}
Each producer replica logs its cumulative token count; the mean rate between the 5th and the last report is the \emph{consumed} rate. Because a windowed pipeline equalizes the mean rates of producer and trainer, the consumed rate agrees with $N_{\mathrm{tok}}/(R\cdot\text{iteration time})$ within 4\% in all runs and does not show which side limits. We therefore also measure the forward-only rate $\rho_{\mathrm{rep}}$ of one replica and compare the producer-only time $N_{\mathrm{tok}}/(R\rho_{\mathrm{rep}})$ with the measured iteration time (Table~\ref{tab:rate}). For the Qwen3.5 runs we measured $\rho_{\mathrm{rep}}$ with the producer writing to a sink on idle nodes with a warm page cache (the steady rate of slots after the first); for the Qwen3-235B runs we took the steady forward time per slot that each replica logs, and for the other pairs the mean forward time of the first slots of the run. The duty, the ratio of the 2 times, is 1 when the producers limit and below 1 when the trainer does, because the transfer of the Qwen3.5 pair takes 0.09\,s per iteration at the link limits (Table~\ref{tab:estimates}). The producers limit the Qwen3-235B and Nemotron runs and the Llama (8+8) training run; for Nemotron and Llama $\rho_{\mathrm{rep}}$ is the first-slot forward rate, so their duty of 1 is an upper bound, and the Llama training run's $\rho_{\mathrm{rep}}$ was read on a shared network (the pace check's 6.1 gives it a duty of 0.92). At constant $\rho_{\mathrm{rep}}$ the iteration time falls with the replica count $R$. The trainer limits every run with a duty below 1, among them the Qwen3.5 (2+1) run, the first (3+1) run (duty 0.90), the Qwen3-32B run, the Llama (4+2) run, and the Llama (8+8) pace check, whose replicas consumed less than their forward-only rate; the (3+1) rerun has duty 0.97 and the seed-43 run with CR $-$ duty 1, so the (3+1) split is balanced within 10\%. The forward-only rate of the TP2$\times$EP2 fp8 replica of the Qwen3.5 producers is the same on all 8 nodes measured (32.6 to 32.9k tokens/s). The seed-43 run with CR all is limited by the trainer, and with CR $-$ by the producers (Table~\ref{tab:rate}).

\begin{table*}[t]
\centering
\footnotesize
\setlength{\tabcolsep}{4pt}
\begin{tabular}{L{4.0cm}rrrrrrlc}
\toprule
Pair (T+P nodes) & $R$ & Consumed & $\rho_{\mathrm{rep}}$ & $N_{\mathrm{tok}}/(R\rho_{\mathrm{rep}})$ & Meas. & Duty & Limit & CR \\
\midrule
Qwen3.5 (1+1), 15 it. & 4 & 3.7 & 28.1 & 37 & 283.7 & 0.13 & trainer & $-$ \\
Qwen3.5 (2+1) & 4 & 13.8 & 28.1 & 37 & 74.3 & 0.50 & trainer & $-$ \\
Qwen3.5 (3+1) & 2 & 29.5 & 32.8 & 63 & 69.9 & 0.90 & trainer & $-$ \\
Qwen3.5 (3+1), rerun & 2 & 31.4 & 32.8 & 63 & 65.0 & 0.97 & trainer & $-$ \\
Qwen3.5 (3+1), seed 43 & 2 & 20.1 & 32.8 & 63 & 99.4 & 0.63 & trainer & all \\
Qwen3.5 (3+1), seed 43, 30 it. & 2 & 33.2 & 32.8 & 63 & 61.9 & 1.0 & producer & $-$ \\
Qwen3.5 (3+1), seed 43, 30 it. & 2 & $-$ & 32.8 & 63 & 94.4 & 0.67 & trainer & all \\
Qwen3.5 (2+1), bf16 teacher & 2 & 13.6 & 34.7 & 60 & 153.0 & 0.40 & trainer & all \\
Qwen3-235B (2+1) & 1 & 10.6 & 10.7 & 393 & 392.9 & 1.0 & producer & all \\
Qwen3-235B (4+4) & 4 & 9.3 & 9.6 & 109 & 108.5 & 1.0 & producer & all \\
Nemotron (2+2) & 1 & 6.7 & 6.6 & 631 & 630.0 & 1.0 & producer & $-$ \\
Nemotron (4+4) & 2 & 6.6 & 6.6 & 316 & 315.9 & 1.0 & producer & $-$ \\
Llama (8+8), training run & 4 & 5.6 & 5.6 & 188 & 186.1 & 1.0 & producer & $-$ \\
Llama (8+8), pace check & 4 & 4.5 & 6.1 & 173 & 234.8 & 0.74 & trainer & all \\
Llama (4+2) & 1 & 5.1 & 6.7 & 627 & 818 & 0.77 & trainer & all \\
Qwen3-32B (1+1) & 4 & 9.6 & 10.4 & 13 & 13.5 & 0.93 & trainer & $-$ \\
\bottomrule
\end{tabular}
\caption{Which side limits the iteration time (T+P: trainer and producer nodes; $R$: producer replicas). Consumed: mean per-replica rate between the 5th and the last report of the producer logs; $\rho_{\mathrm{rep}}$: forward-only rate of one replica, both in $10^3$ tokens/s; $N_{\mathrm{tok}}/(R\rho_{\mathrm{rep}})$: the iteration time if the trainer were instantaneous; Meas.: measured median iteration time, both in seconds; Duty: $N_{\mathrm{tok}}/(R\rho_{\mathrm{rep}})$ over Meas., capped at 1; CR: the parts of the cache release that were on (Table~\ref{tab:cachesplit}). The producers limit a run when the duty is 1 (the trainer is then at least as fast) and the trainer limits it when the duty is below 1. Rerun: the 3+1 run of Table~\ref{tab:main} repeated for 12 iterations (65.0\,s against 69.9\,s, 7\%). Each row is 1 run.}
\label{tab:rate}
\end{table*}

\begin{algorithm}[t]
\footnotesize
\begin{tabular}{@{}r@{\hspace{5pt}}p{6.7cm}@{}}
\multicolumn{2}{@{}p{7.2cm}@{}}{\textbf{Input:} teacher replica, corpus order, look-ahead $\delta$ iterations, start slot $\eta_0$} \\[2pt]
1 & load the teacher shard (routed experts in fp8 for the Qwen3.5 and Nemotron teachers, bf16 otherwise); write the lm\_head shards once \\
2 & \textbf{for} $\eta=\eta_0,\eta_0{+}1,\dots$ \textbf{do} \\
3 & \ind{\textbf{if} $\eta\bmod R\ne$ my replica \textbf{then continue}} \\
4 & \ind{$\bm{H}_t\leftarrow$ forward of microbatch $\eta$ (final hidden states)} \\
5 & \ind{wait until $\eta<(\text{iteration}+\delta)\cdot\text{slots per iteration}$} \\
6 & \ind{slice rows per last-stage rank; send header and payload to the receiver group} \\
7 & \textbf{end for} \\
\end{tabular}
\caption{Producer loop of one replica.}
\label{alg:producer}
\end{algorithm}

\begin{algorithm}[t]
\footnotesize
\begin{tabular}{@{}r@{\hspace{5pt}}p{6.7cm}@{}}
\multicolumn{2}{@{}p{7.2cm}@{}}{\textbf{Input:} student hidden states $\bm{H}_s^{(q)}$, prefetch ring, shards $\bm{W}_t^{(j)},\bm{W}_s^{(j)}$, labels $y$, window size $C$} \\[2pt]
1 & $\bm{H}_t^{(q)}\leftarrow$ pop the next slot from the ring; post receives for up to 4 following slots of the same iteration \\
2 & \textbf{for} $k=0,\dots,K{-}1$ \textbf{do} \\
3 & \ind{$\bm{Z}_{x,k}^{(q,j)}\leftarrow\bm{H}_{x,k}^{(q)}\bm{W}_x^{(j)}$ for $x\in\{t,s\}$} \\
4 & \ind{kernel: for each row, one pass over the entries $u=0,\dots,\Vloc{-}1$ in tiles gives $m_x^{(j)},a_x^{(j)}$ ($x\in\{t,s\}$) and $c^{(j)}$} \\
5 & \ind{MAX all-reduce of $m_x^{(j)}$; rescale; SUM all-reduce of $a_x^{(j)},c^{(j)}$ over the tp ranks; the label logit of $\mathrm{CE}_i$ comes from a separate vocabulary-parallel pass} \\
6 & \ind{$\mathcal{L}\leftarrow\mathcal{L}+\sum_i[(1{-}\lambda)\mathrm{CE}_i+\lambda\,\mathrm{KD}_i]$ over the rows of window $k$; free $\bm{Z}_{x,k}^{(q,j)}$} \\
7 & \textbf{end for} \\
8 & backward: recompute each window; accumulate $\partial\mathcal{L}/\partial\bm{H}_s^{(q)}$ and $\partial\mathcal{L}/\partial\bm{W}_s^{(j)}$ \\
\end{tabular}
\caption{Streamed KD step for one microbatch at cell $(q,j)$ (\S\ref{sec:loss}).}
\label{alg:step}
\end{algorithm}

\subsection{Teacher Precision}
\label{app:fp8}
\paragraph{In the producer.}
The producer runs the teacher forward only, without gradients, optimizer state, or backward activations, so its weights and a fitted remainder $\varphi_p$ (forward-pass workspace, CUDA context, allocator) fill its memory. It holds the routed experts of the Qwen3.5 and Nemotron teachers in fp8 and every other module (attention, mamba mixers, latent projections, norms, embeddings) in bf16; the shared experts are bf16 on GPT-path teachers and fp8 on the hybrid Nemotron teacher. The routed experts are most of an MoE teacher's parameters (546.7B of 560.4B for the Nemotron teacher), so the shard is 51\% of its bf16 size. For the 560B teacher the model predicts that bf16 everywhere does not fit a 16-GPU replica: about 65\gib{} per GPU for weights alone, plus the 31 to 32\gib{} remainder measured with fp8 (no bf16 run of this teacher exists). The experts are stored as fp8 primary weights (e4m3 with tensor-wise scales), their grouped GEMMs run under Transformer Engine's current-scaling fp8 autocast, and activations, including the hidden states sent to the trainer, stay bf16. The Qwen3-235B producers and those of the dense Llama and Qwen3-32B teachers are bf16 throughout, which is why the shard term of \eqref{eq:producer} counts 1 byte per routed-expert parameter of an fp8 teacher and 2 otherwise. On the trainer, the teacher-logit matrix product can also run in fp8 (e4m3 inputs with per-tensor scales), which leaves the hidden states sent to the trainer in bf16; it was on in the Qwen3.5 runs with an fp8 teacher, the Nemotron and Qwen3-32B runs, and the 2 Qwen3-235B trainers of Table~\ref{tab:trainerruns} with $\kappa_{\mathrm{str}}$ of 1.6, and off in the Qwen3.5 run with a bf16 teacher, the other Qwen3-235B runs, and the Llama runs.

The fp8 build changes the stock provider in 4 ways: (1) the quantization recipe is set on the teacher's model provider and matched to the expert modules by name; (2) the mixed-precision wrapper is skipped, because it casts the fp8 parameters back to bf16; (3) the expert modules of GPT-path models are given module names, because the recipe matches modules by name and an unnamed build constructs every expert in bf16; (4) the checkpoint is cast into the fp8 parameters on load, with the scale tensors of a native-fp8 release applied first, and the CPU copies of the initial values that the fp8 parameters keep are released after the load. No run isolates a single change.

\paragraph{In the colocated stack.}
The colocated runs use none of this engineering. Megatron-Bridge's distillation provider builds their teacher in bf16, and loading a native-fp8 release through it drops the scale tensors in the families we ran. Giving the colocated stack an fp8 teacher means doing the same on the teacher's provider: setting the recipe and skipping the wrapper, naming the experts, and casting the pre-converted bf16 checkpoint into fp8 parameters on load. We implemented it for the routed experts: the build engages and the peak allocated memory falls by 5.8\gib{}, but the colocated run's KD loss leaves the bf16 run's at iteration 3 (Table~\ref{tab:fp8colo}). The recipe covers routed experts only, so the dense Llama and Qwen3-32B teachers stay in bf16.

\begin{table}[t]
\centering
\footnotesize
\setlength{\tabcolsep}{3pt}
\begin{tabular}{L{1.0cm}rrL{2.1cm}L{1.9cm}}
\toprule
Experts & s/it & Torch & KD loss, it.\ 1, 2, 3 & Grad norm, it.\ 1, 2, 3 \\
\midrule
bf16 & 100.3 & 61.7 & 0.215, 0.207, 0.207 & 0.70, 0.70, 1.04 \\
fp8 & 119.1 & 55.9 & 0.215, 0.207, 0.795 & 0.70, 0.70, 14.5 \\
\bottomrule
\end{tabular}
\caption{Colocated KD-only runs of the Qwen3.5 pair on 4 nodes (tp2/pp4/cp4/dp1, 32768 tokens, 12 iterations) with the teacher's routed experts in bf16 or fp8 (Appendix~\ref{app:fp8}). s/it: median over iterations 3 to 12; Torch: peak allocated in GiB after iteration 1. From iteration 4 to 12 the KD loss is 0.505 to 0.665 with fp8 against 0.249 to 0.290 with bf16, so the 2 s/it values do not compare precisions.}
\label{tab:fp8colo}
\end{table}

\subsection{Proof of Exactness}
\label{app:loss-equiv}
We prove Theorem~\ref{thm:exact} for $C$ dividing $r$, $2\,\mathrm{cp}$ dividing $L$, and $\mathrm{tp}$ dividing $V$. A row has logits $\bm{z}_t,\bm{z}_s\in\mathbb{R}^V$ (\S\ref{sec:loss}). For a set $\mathcal{S}$ of its columns define the statistics
\begin{equation}
\label{eq:stat}
\begin{gathered}
m_{x,\mathcal{S}}=\max_{v\in\mathcal{S}}z_{x,v},\quad a_{x,\mathcal{S}}=\sum_{v\in\mathcal{S}}e^{z_{x,v}-m_{x,\mathcal{S}}},\\
c_{\mathcal{S}}=\sum_{v\in\mathcal{S}}e^{z_{t,v}-m_{t,\mathcal{S}}}\,z_{s,v},
\end{gathered}
\end{equation}
with $(-\infty,0,0)$ for the empty set. Then \eqref{eq:stats} is the case of all $V$ columns, $\log\sum_{v\in\mathcal{S}}e^{z_{x,v}}=m_{x,\mathcal{S}}+\log a_{x,\mathcal{S}}$, and $1\le a_{x,\mathcal{S}}\le|\mathcal{S}|$ for nonempty $\mathcal{S}$.

\begin{lemma}[Merge]
\label{lem:merge}
Let $\mathcal{P}_0,\dots,\mathcal{P}_{n-1}$ be $n$ disjoint sets of columns with union $\mathcal{S}$, and let $l$ index the parts. Then
\begin{equation}
\label{eq:merge}
\begin{aligned}
m_{x,\mathcal{S}}&=\max_l m_{x,\mathcal{P}_l},\\
a_{x,\mathcal{S}}&=\sum_{l}a_{x,\mathcal{P}_l}\,e^{m_{x,\mathcal{P}_l}-m_{x,\mathcal{S}}},\\
c_{\mathcal{S}}&=\sum_{l}c_{\mathcal{P}_l}\,e^{m_{t,\mathcal{P}_l}-m_{t,\mathcal{S}}}.
\end{aligned}
\end{equation}
\end{lemma}
\begin{proof}
The maximum of a union is the largest of the maxima. Splitting the sum by part and inserting $e^{\pm m_{x,\mathcal{P}_l}}$ gives
\begin{align*}
a_{x,\mathcal{S}}&=\sum_l\sum_{v\in\mathcal{P}_l}e^{z_{x,v}-m_{x,\mathcal{S}}}\\
&=\sum_l e^{m_{x,\mathcal{P}_l}-m_{x,\mathcal{S}}}\sum_{v\in\mathcal{P}_l}e^{z_{x,v}-m_{x,\mathcal{P}_l}},
\end{align*}
and the inner sum is $a_{x,\mathcal{P}_l}$. The same computation with the weight $z_{s,v}$ and the teacher's maximum gives $c_{\mathcal{S}}$.
\end{proof}

\begin{lemma}[Scan]
\label{lem:scan}
Let the columns of $\mathcal{S}$ be read in consecutive tiles $\mathcal{U}_0,\dots,\mathcal{U}_{U-1}$, and let the state after tile $u$ be the merge \eqref{eq:merge} of the state after tile $u{-}1$ with the statistics \eqref{eq:stat} of $\mathcal{U}_u$, starting from $(-\infty,0,0)$. Then the state after tile $u$ equals the statistics of $\mathcal{U}_0\cup\dots\cup\mathcal{U}_u$. For tiles of 1 entry, write $z_{x,u}$ for $z_{x,u}^{(j)}$, $u=0,\dots,\Vloc{-}1$, and let $m_{x,u}$, $a_{x,u}$ and $c_u$ be the state after entries $0,\dots,u$, starting from $m_{x,-1}=-\infty$, $a_{x,-1}=0$ and $c_{-1}=0$; this update is the recurrence
\begin{equation}
\label{eq:running}
\begin{aligned}
m_{x,u}&=\max(m_{x,u-1},z_{x,u}),\\
a_{x,u}&=a_{x,u-1}e^{m_{x,u-1}-m_{x,u}}+e^{z_{x,u}-m_{x,u}},\\
c_u&=c_{u-1}e^{m_{t,u-1}-m_{t,u}}\\
&\quad+e^{z_{t,u}-m_{t,u}}\,z_{s,u}.
\end{aligned}
\end{equation}
Whenever the maximum rises, the exponential factor on the previous sum rescales it to the new maximum, and this factor is at most 1.
\end{lemma}
\begin{proof}
We use induction on $u$. The empty prefix has statistics $(-\infty,0,0)$. If the state after tile $u{-}1$ is the statistics of $\mathcal{U}_0\cup\dots\cup\mathcal{U}_{u-1}$, Lemma~\ref{lem:merge} with $n=2$ shows that merging it with $\mathcal{U}_u$ gives the statistics of $\mathcal{U}_0\cup\dots\cup\mathcal{U}_u$. A column $v$ has statistics $(z_{x,v},1,z_{s,v})$, and inserting them into \eqref{eq:merge} with $n=2$ gives \eqref{eq:running}.
\end{proof}

\begin{proof}[Proof of Theorem~\ref{thm:exact}]
\emph{Statistics.} Tp rank $j$ scans the columns of its shard in any tiling, so by Lemma~\ref{lem:scan} its final state $(m_x^{(j)},a_x^{(j)},c^{(j)})$ is the statistics of the shard. Lemma~\ref{lem:merge} applied to the $\mathrm{tp}$ shards yields the combination
\begin{equation}
\label{eq:combine}
\begin{aligned}
m_x&=\max_j m_x^{(j)},\\
a_x&=\sum_j a_x^{(j)}e^{m_x^{(j)}-m_x},\quad c=\sum_j c^{(j)}e^{m_t^{(j)}-m_t},
\end{aligned}
\end{equation}
so the all-reduced $(m_x,a_x,c)$ are the statistics of all $V$ columns, that is, \eqref{eq:stats}.

\emph{Loss values.} With $p_{t,v}=e^{z_{t,v}-m_t}/a_t$ we have $\sum_v p_{t,v}z_{s,v}=c/a_t$, and $\mathrm{lse}_s=m_s+\log a_s$, so
\begin{align*}
-\sum_v p_{t,v}\log p_{s,v}&=\mathrm{lse}_s-\sum_v p_{t,v}z_{s,v}\\
&=\mathrm{lse}_s-\frac{c}{a_t},\\
-\log p_{s,y_i}&=\mathrm{lse}_s-z_{s,y_i}.
\end{align*}
These are \eqref{eq:rowloss} and the definitions in \eqref{eq:objective}. Exactly 1 tp rank holds $z_{s,y_i}$ and the others contribute 0. (The implementation evaluates the LM term in a separate vocabulary-parallel log-softmax pass over the same window, which yields the same normalizer.)

\emph{Gradients.} For a row, $\partial\mathrm{KD}_i/\partial z_{s,v}=p_{s,v}-p_{t,v}$ and $\partial\mathrm{CE}_i/\partial z_{s,v}=p_{s,v}-[v{=}y_i]$, hence
\[
\frac{\partial\mathcal{L}}{\partial\bm{z}_s}=(1-\lambda)\,(p_s-\bm{e}_{y_i})+\lambda\,(p_s-p_t).
\]
This depends on the row only through the statistics $(m_s,a_s)$ and $(m_t,a_t)$, which the backward pass regenerates (a MAX and a SUM all-reduce per distribution and window), so it obtains $p_s$ and $p_t$ from recomputed logits and the same values as a full-row evaluation. For window $k$, tp rank $j$ holds the block $\bm{G}^{(j)}\in\mathbb{R}^{C\times\Vloc}$ of this gradient, and the chain rule through \eqref{eq:window} gives
\begin{align*}
\frac{\partial\mathcal{L}}{\partial\bm{W}_s^{(j)}}&=\sum_q\sum_k\bm{H}_{s,k}^{(q)\top}\bm{G}_k^{(q,j)},\\
\frac{\partial\mathcal{L}}{\partial\bm{H}_{s,k}^{(q)}}&=\sum_j\bm{G}^{(j)}\bm{W}_s^{(j)\top},
\end{align*}
the latter being an all-reduce over the tp ranks.

\emph{Windows and cp ranks.} $\mathcal{L}$ is a sum of per-row terms \eqref{eq:objective}. Partitioning the rows first over the cp ranks and then into windows only regroups a finite sum, so the loss as a function of the hidden states and head weights, and its gradients, depend on neither $C$ nor the sharding; the student's forward pass, the loss mask, and the normalization lie outside the identity.
\end{proof}

\begin{remark}
Theorem~\ref{thm:exact} is an identity in exact arithmetic. It does not cover the floating-point summation order, the trainer's fp8 teacher-logit product, the producer's fp8 expert weights, or the bf16 hidden-state transfer. ModelOpt and KDFlow minimize the full-vocabulary forward KL, which equals $\mathrm{KD}_i$ up to the teacher's entropy, under different loss weights, and NeMo-RL takes it over 8192 columns (\S\ref{sec:setup}), so the identity equates per-row loss functions and not recipes.
\end{remark}

\section{Experimental Setup}
\label{app:setup}

\paragraph{Search space.}
Both layouts are searched as the optimal setting of \S\ref{sec:estimation} defines, over the same axes: tp over the sizes that the heads, key-value heads, hidden size, and Mamba groups admit and that we ran, pp over the divisors of the layer count, cp over 1 to 16 in powers of 2, ep over the divisors of the expert count, and dp the remaining factor, which divides the global batch (128, or 126 where dp is 3 or 7: the 3-, 6-, 24-, and 28-node colocated cells and the 4-node disaggregated run). A colocated cell (a pair on a number of nodes) uses 1 grid for both models, because ModelOpt requires the teacher at the student's tp, pp, and cp and has no interleaved pipeline; the search also gives both models one ep, the largest common divisor of the expert counts, with ep$\cdot$pp dividing the GPUs. tp and pp are common to the 2 models (Qwen3.5 tp 1 to 2, Qwen3-235B 1 to 4, Nemotron 1 to 8, Llama 1 to 32). Across the 12 searched colocated cells of Table~\ref{tab:brackets}, a cell has 12 (Qwen3.5 on 3 nodes) to 305 (Nemotron on 8 nodes) grids. A disaggregated cell has a trainer grid over the student alone, replicas of the teacher spanning $\mathrm{tp}_p\times\mathrm{ep}_p$ GPUs (tp$_p$ over the divisors of its heads and Mamba groups, up to 16, and 8 for Nemotron), and a split of the nodes between the 2 roles. The wire requires the number of replicas $R$ to be a multiple of dp, or below dp at microbatch 1; the trainer and producer inequalities of \S\ref{sec:memory} bound each role, and $R\,\rho_{\mathrm{rep}}\ge N_{\mathrm{tok}}/T_{\mathrm{iter}}$ bounds $R$ below. The Megatron-Bridge runs use microbatch 1, full activation recomputation, and 32768-token sequences unless the text or a table states otherwise.

\paragraph{Sharded data parallelism.}
The search of both layouts excludes fully sharded data parallelism (FSDP). In the Megatron-Bridge release we use, the Megatron-FSDP device mesh reshapes the world as (dp$\cdot$cp, ep, tp) with no pipeline axis, and setup raised an error at pp 2 and 4 on MoE students. Every trainer grid of Table~\ref{tab:runs} has pp 2, and the only colocated grids of the primary pair on 4 nodes that the model predicts under the budget (2 of 31) have pp 4 and pp 2 (Table~\ref{tab:brackets}); these counts come from the memory model without a sharded-state term, and FSDP at pp 1 was not run. The producer holds weights alone, split over $\mathrm{tp}_p\times\mathrm{ep}_p$ ranks. KDFlow trains the student with FSDP2 and NeMo-RL with its DTensor policy, which wraps it with FSDP2, and on the primary pair KDFlow's device reading is at the limit and NeMo-RL's 2.1\gib{} below it (79.1 and 77.0\gib{}, sampled device memory; Table~\ref{tab:stacks}): KDFlow with the idle teacher offloaded and the loss chunked, NeMo-RL at a quarter of the batch, because its full-vocabulary logits buffer, 121\gib{} at the full batch, does not fit. At cp 1 the logits buffers are 144 of the 186\gib{} of the colocated peak on 3 nodes (Figure~\ref{fig:overview}), memory that sharding the states does not reduce (the buffers shrink with cp).

\subsection{Disaggregated Runs}
\label{app:runs}
\paragraph{Training recipe.}
Model sizes in Table~\ref{tab:pairs} are config-derived (the Nemotron release name reads 550B). Each GPU reports 79.65\gib{} to the driver, of which 79.11\gib{} is allocatable by CUDA (the capacity PyTorch reports, the budget $M$); in-run device readings include the remainder, so a full card reads 79.6\gib{} and the readings of 79.6 are at the limit. All runs train on FineWeb-Edu at sequence length 32{,}768 (4096 for the Qwen3-32B pair unless a table states otherwise) with full activation recomputation, disaggregated teachers with fp8 routed-expert weights and bf16 elsewhere (the Qwen3-235B, Llama, and Qwen3-32B teachers are bf16 throughout), and bf16 students with fp32 master weights, at global batch 128 (4.19M tokens per iteration, 0.52M for the Qwen3-32B pair; 126 for the 4-node Qwen3.5 run so that $\mathrm{dp}{=}3$ divides it). The recipe is AdamW at learning rate $4\times10^{-4}$ ($5\times10^{-5}$ for the Llama pair) with cosine decay, 32 warmup iterations, and, in the \sqlk{} runs, KD weight $\lambda$ decaying from 1.0 to 0.75 ($\lambda{=}1$ in the KD-only runs). On the Qwen3.5, Qwen3-235B, Nemotron, and Qwen3-32B pairs the colocated runs decay the learning rate to $10^{-5}$ where the disaggregated runs decay to $3\times10^{-5}$ (both layouts decay to $3.75\times10^{-6}$ on Llama), and in the \ballk{} runs ModelOpt's rescaling makes its total loss 2$\times$ its LM loss at every iteration. The colocated runs hold the teacher in bf16; one disaggregated run matches the learning-rate floor and the teacher precision and multiplies its total loss by 2 in the same way (below). MMLU is scored 4-shot by generative exact match with FlexEval over 14{,}042 items. The window width is $C{=}1024$ to $4096$, and \system{} is built on Megatron-Bridge.

\paragraph{Runs.}
Table~\ref{tab:runs} lists the disaggregated runs of the main configurations with their node allocation and grid, and Figure~\ref{fig:curves} plots their KD loss; Appendices~\ref{app:primary} and~\ref{app:otherpairs} give the primary-pair and other-pair measurements. For each run of Table~\ref{tab:runs} the trainer's token rate is within 4\% of the producers' mean rate in the steady state (Table~\ref{tab:rate}). Median iteration times are over iterations 3 to the last, except 8 to the last for the Qwen3.5 (2+1) run and 6 to the last for the seed-43 colocated runs; the captions of Tables~\ref{tab:cachesplit}, \ref{tab:speedup}, \ref{tab:otherpairs}, and~\ref{tab:fp8colo} state the windows of their runs.

\begin{table*}[t]
\centering
\footnotesize
\setlength{\tabcolsep}{4pt}
\begin{tabular}{L{2.2cm}L{2.6cm}L{2.3cm}rrrccc}
\toprule
Run (T+P nodes) & Trainer grid & Producer replicas & Iter. & s/it & $10^3$ tok/s & KD loss & Peak & CR \\
\midrule
Qwen3.5 (1+1) & tp2/pp2/cp2/dp1 & 4$\times$TP2 & 15 & 283.7 & 14.8 & 2.228 $\to$ 2.488 & 57.5; 39.4 & $-$ \\
Qwen3.5 (2+1) & tp2/pp2/cp2/dp2 & 4$\times$TP2 & 200 & 74.3 & 56.4 & 2.228 $\to$ 2.227 & 39.2; 39.4 & $-$ \\
Qwen3.5 (3+1) & tp2/pp2/cp2/dp3 & 2$\times$TP2$\times$EP2 & 200 & 69.9 & 59.1 & 2.226 $\to$ 2.225 & 33.0; 32.0 & $-$ \\
Qwen3-235B (2+1) & tp2/pp2/ep4 & 1$\times$TP8 (bf16) & 200 & 392.9 & 10.7 & 2.020 $\to$ 1.962 & 64.3; 72.9 & all \\
Qwen3-235B (4+4) & tp2/pp2/ep4, dp8 & 4$\times$TP8 (bf16) & 200 & 108.5 & 38.7 & 2.020 $\to$ 1.962 & 50.7; 73.3 & all \\
Qwen3-235B (4+4), KD-only & tp2/pp2/ep4, dp8 & 4$\times$TP8 & 60 & 105.0 & 40.0 & 2.066 $\to$ 2.029 & 51.0; 72.9 & all \\
Nemotron (2+2) & tp2/pp2/ep4 & 1$\times$TP8$\times$EP2 & 200 & 630.0 & 6.7 & 2.19 $\to$ 1.87 & 67.7; 69.6 & $-$ \\
Nemotron (4+4) & tp2/pp2/ep4, dp8 & 2$\times$TP8$\times$EP2 & 200 & 315.9 & 13.3 & 2.19 $\to$ 1.87 & 52.5; 70.7 & $-$ \\
Llama (8+8) & tp8/pp2/cp2/dp2 & 4$\times$TP16 (bf16) & 400 & 186.1 & 22.5 & 1.91 $\to$ 1.81 & 50.3; 65.4 & $-$ \\
Qwen3-32B (1+1) & tp2/pp2/dp2 & 4$\times$TP2 & 400 & 13.5 & 38.8 & 2.139 $\to$ 2.099 & 25.1; 36.4 & $-$ \\
\bottomrule
\end{tabular}
\caption{Disaggregated runs (T+P: trainer and producer nodes): trainer grid, producer replicas (fp8 experts; the Qwen3-235B, Llama, and Qwen3-32B teachers are bf16), iterations completed (Iter.), median iteration time (s/it), throughput ($10^3$ tokens/s), KD loss at the first and last iteration, and peak memory in GiB (trainer torch-allocated on the last pipeline stage; producers' device). CR: the parts of the cache release that were on (Table~\ref{tab:cachesplit}).}
\label{tab:runs}
\end{table*}

\begin{figure*}[t]
\centering
\begin{tikzpicture}
\pgfplotsset{
  compat=1.17,
  cpanel/.style={
    width=5.6cm, height=3.3cm, scale only axis,
    title style={yshift=-2pt},
    grid=major, grid style={gray!25, line width=0.3pt},
    axis line style={line width=0.4pt}, tick style={line width=0.4pt},
    unbounded coords=jump, xmin=0, xmax=400, xtick={0,100,200,300,400},
    xlabel={iteration}, ylabel={KD loss},
  },
  rawl/.style={line width=0.35pt, opacity=0.28, forget plot},
  main/.style={blue!65!black, line width=1.1pt},
  alt/.style={orange!85!red, line width=1.1pt, dashed},
}
\tikzset{lab/.style={font=\footnotesize, inner sep=1pt}}
\begin{groupplot}[group style={group size=2 by 2, horizontal sep=1.6cm, vertical sep=2.1cm}]
\nextgroupplot[cpanel, title={Qwen3.5-35B $\rightarrow$ 9B}, ymin=2.1, ymax=2.5, xmax=200, xtick={0,50,100,150,200},
  legend to name=legendcurves, legend columns=2,
  legend style={draw=none, /tikz/every even column/.append style={column sep=8pt}}]
\addplot[rawl, main] table[x=it, y=kd, col sep=comma] {figs/p1_ours3n.csv};
\addplot[main] table[x=it, y=kd_s, col sep=comma] {figs/p1_ours3n.csv};
\addlegendentry{smaller node set}
\addplot[rawl, alt] table[x=it, y=kd, col sep=comma] {figs/p1_ours4n.csv};
\addplot[alt] table[x=it, y=kd_s, col sep=comma] {figs/p1_ours4n.csv};
\addlegendentry{equal node count}
\node[lab, blue!65!black] at (axis cs:150,2.43) {200 iter};
\node[lab, orange!85!red] at (axis cs:150,2.155) {200 iter};
\nextgroupplot[cpanel, title={Qwen3-235B $\rightarrow$ 30B}, ymin=1.8, ymax=2.3, xmax=200, xtick={0,50,100,150,200}]
\node[lab, blue!65!black] at (axis cs:105,2.14) {200 iter};
\node[lab, orange!85!red] at (axis cs:150,1.91) {200 iter};
\addplot[rawl, main] table[x=it, y=kd, col sep=comma] {figs/p2_ours3n.csv};
\addplot[main] table[x=it, y=kd_s, col sep=comma] {figs/p2_ours3n.csv};
\addplot[rawl, alt] table[x=it, y=kd, col sep=comma] {figs/p2_ours8n.csv};
\addplot[alt] table[x=it, y=kd_s, col sep=comma] {figs/p2_ours8n.csv};
\nextgroupplot[cpanel, title={Nemotron-Ultra $\rightarrow$ Nano}, ymin=1.7, ymax=2.3, xmax=200, xtick={0,50,100,150,200}]
\node[lab, blue!65!black] at (axis cs:120,2.06) {200 iter};
\node[lab, orange!85!red] at (axis cs:150,1.77) {200 iter};
\addplot[rawl, main] table[x=it, y=kd, col sep=comma] {figs/p4_ours4n.csv};
\addplot[main] table[x=it, y=kd_s, col sep=comma] {figs/p4_ours4n.csv};
\addplot[rawl, alt] table[x=it, y=kd, col sep=comma] {figs/p4_ours8n.csv};
\addplot[alt] table[x=it, y=kd_s, col sep=comma] {figs/p4_ours8n.csv};
\nextgroupplot[cpanel, title={Llama-3.1-405B $\rightarrow$ 70B}, ymin=1.7, ymax=2.0]
\node[lab, blue!65!black] at (axis cs:300,1.955) {400 iter};
\addplot[rawl, main] table[x=it, y=kd, col sep=comma] {figs/p5_ours.csv};
\addplot[main] table[x=it, y=kd_s, col sep=comma] {figs/p5_ours.csv};
\end{groupplot}
\node[anchor=south] at ($(group c1r1.north west)!0.5!(group c2r1.north east)+(0,0.72cm)$) {\pgfplotslegendfromname{legendcurves}};
\end{tikzpicture}
\caption{KD-loss trajectories of the disaggregated runs on the 4 pairs of Table~\ref{tab:pairs} (per-iteration values, thin; 17-iteration moving averages, thick). Students start from released checkpoints, so levels differ across pairs. Blue: the smallest disaggregated run of the pair; orange: the run on the node count of the colocated cell (4 nodes for Qwen3.5, 8 for Qwen3-235B and Nemotron); the Llama panel shows its 16-node run (blue). The labels give the number of iterations of each run, in the color of its curve.}
\label{fig:curves}
\end{figure*}
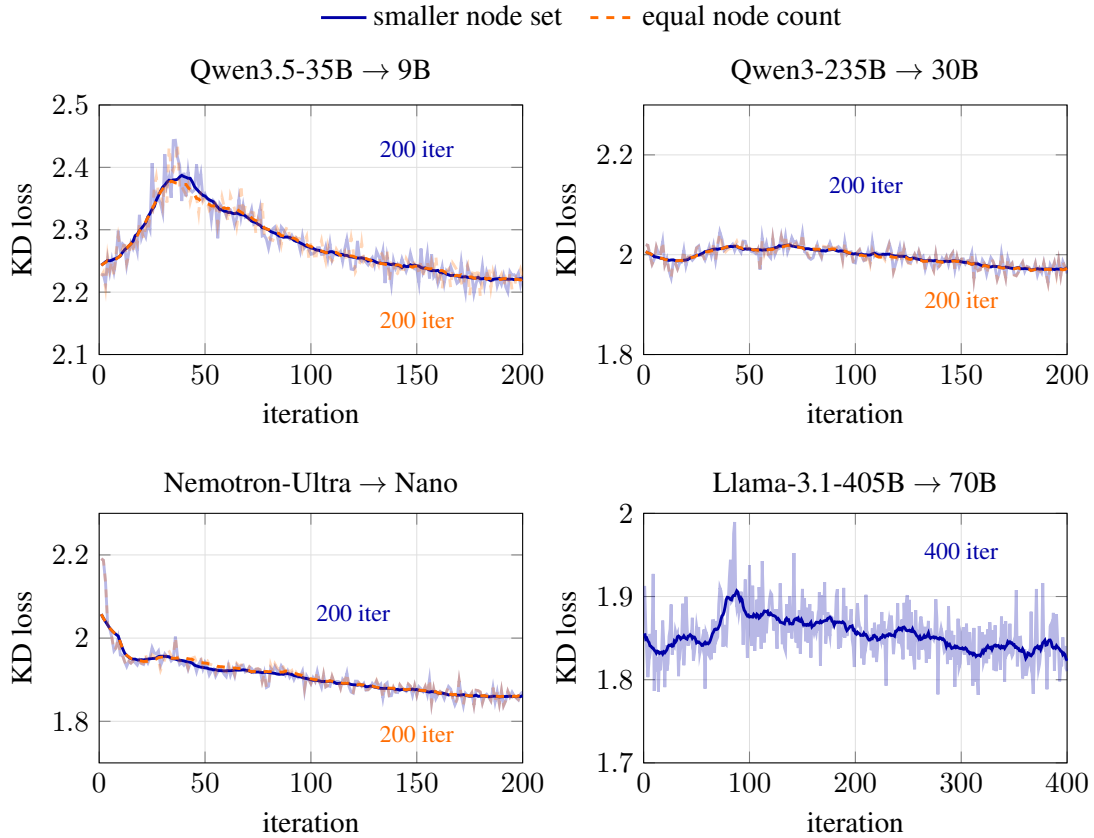

\paragraph{MMLU evaluation.}
MMLU \citep{hendrycks2021mmlu} is scored with the public FlexEval release \citep{flexeval} on the public test set: a fixed 4-shot prompt for each of the 57 subjects, greedy decoding of at most 256 new tokens that stops at a newline, and exact match after NFKC normalization, lowercasing, and stripping.

\subsection{Compared Stacks}
\label{app:stacks}
\paragraph{Colocated configurations.}
Table~\ref{tab:brackets} lists the colocated configurations behind Table~\ref{tab:pairs}. The rows that ended with NaN in their first iteration (Qwen3-235B on 2 and 4 nodes, Nemotron with cp8 and with tp4) stopped before the first optimizer step and are evaluated without Adam moments and with the 3.1\gib{} non-PyTorch memory of that iteration. The rows that set constants of the model are the 2 Qwen3.5 cp4 rows, the OOM points of the 2 Qwen3.5 4-node rows, and those of the 2 Nemotron 8-node rows with cp4 (Appendix~\ref{app:memory}); the other rows are outside the fit, except that the 3.1\gib{} first-iteration constant is the mean of the 3 Qwen3.5 OOM points (including the 3-node one) and the Nemotron 4-node OOM point.

\begin{table*}[tp]
\centering
\footnotesize
\setlength{\tabcolsep}{2pt}
\renewcommand{\arraystretch}{0.9}
\begin{tabular}{L{1.85cm}L{4.0cm}rrrrrrL{4.4cm}}
\toprule
 & & \multicolumn{2}{c}{Device} & \multicolumn{2}{c}{Torch} & \multicolumn{2}{c}{OOM point} & \\
\cmidrule(lr){3-4}\cmidrule(lr){5-6}\cmidrule(lr){7-8}
Pair & Configuration & Pred. & Meas. & Pred. & Meas. & Pred. & Meas. & Outcome \\
\midrule
Qwen3.5 & 3 nodes, tp2/dp3, best grid tp2/pp2/cp2, global batch 126 & 133.1 & $-$ & 120.8 & $-$ & $-$ & $-$ & not run \\
 & 3 nodes, tp2/pp4/cp1/dp3 & 185.9 & $-$ & 173.6 & $-$ & 86.8 & 88.0 & OOM at the 4th buffer (student divide) \\
 & 4 nodes, cp1 & 184.6 & $-$ & 172.3 & $-$ & 86.4 & 87.6 & OOM at the 4th buffer \\
 & 4 nodes, cp2 & 112.1 & $-$ & 99.8 & $-$ & 78.3 & 79.0 & OOM at the 7th buffer \\
 & 4 nodes, cp4/dp1 & 75.8 & 76.1 & 63.6 & 63.4 & $-$ & $-$ & trains 200 iterations \\
 & 4 nodes, cp4/dp1, \ballk{} & 79.6 & 79.1 & 67.4 & 67.2 & $-$ & $-$ & trains 200 iterations \\
 & 4 nodes, cp4/dp1, \sqlk{} & $-$ & 76.8 & $-$ & 59.9 & $-$ & $-$ & trains 200 iterations \\
 & 4 nodes, cp4/dp1, \sqlk{} fused & $-$ & 58.7 & $-$ & 40.7 & $-$ & $-$ & trains 15 iterations \\
 & 3 nodes, tp2/pp4/cp1/dp3, \sqlk{} fused & $-$ & $-$ & $-$ & $-$ & $-$ & 72.9 & OOM in the first forward \\
 & 3 nodes, tp2/pp4/cp1/dp3, \sqlk{} fused-W & $-$ & 78.9 & $-$ & 64.1 & $-$ & $-$ & trains 15 iterations, batch 126 \\
 & 2 nodes, tp2/pp4/cp2/dp1, \sqlk{} fused-W & $-$ & 65.8 & $-$ & 50.3 & $-$ & $-$ & trains 15 iterations \\
 & 2 nodes, tp2/pp4/cp1/dp2, \sqlk{} fused-W & $-$ & $-$ & $-$ & $-$ & $-$ & 65.4 & OOM after iteration 1 (windowed pick) \\
 & 4 nodes, tp2/pp2/cp8/dp1 & 77.3 & 77.0 & 65.0 & 65.7 & $-$ & $-$ & fits, trains 3 iterations \\
 & 4 nodes, cp4/dp1, 8192 tokens & 48.7 & $-$ & 36.4 & 37.4 & $-$ & $-$ & fits, 12 iterations \\
 & 4 nodes, tp2/pp2/cp2/dp4, 8192 tokens & 77.3 & $-$ & 65.0 & 65.7 & $-$ & $-$ & fits, 12 iterations \\
Qwen3-235B & 3 nodes & 210.4 & $-$ & 198.1 & $-$ & $-$ & $-$ & OOM in model construction \\
 & 3 nodes, ep4, \sqlk{} fused-W & $-$ & $-$ & $-$ & $-$ & $-$ & 85.8 & OOM in model construction \\
 & 6 nodes, cp2 & 107.2 & $-$ & 94.9 & $-$ & $-$ & $-$ & OOM in the loss computation \\
 & 6 nodes, tp4/pp1/cp4/ep16/dp3 & 73.5 & $-$ & 61.2 & $-$ & $-$ & $-$ & not run \\
 & 2 nodes, tp4/pp1/cp4/ep16 & 64.3 & 67.3 & 61.2 & $-$ & $-$ & $-$ & NaN in the gradient norm of the first backward pass \\
 & 4 nodes, tp4/pp1/cp8/ep32 & 36.7 & 39.5 & 33.6 & $-$ & $-$ & $-$ & NaN in the gradient norm of the first backward pass \\
 & 8 nodes, cp2 & 87.1 & $-$ & 74.8 & $-$ & $-$ & $-$ & OOM in the loss computation (74.4 in use, 4.6 requested) \\
 & 8 nodes, cp4 & 64.5 & $-$ & 52.2 & $-$ & $-$ & $-$ & fits; diverges with the all-to-all dispatcher \\
 & 8 nodes, cp4, all-gather dispatcher, \ballk{} & 64.5 & 66.3 & 52.2 & 54.4 & $-$ & $-$ & trains 60 iterations, 75 s/it \\
Nemotron & 4 nodes, tp2/pp4/cp1 & 334.8 & $-$ & 322.5 & $-$ & $-$ & $-$ & OOM during model construction (78.7 in use, 7.1 requested) \\
 & 4 nodes, ep4, \sqlk{} fused-W & $-$ & $-$ & $-$ & $-$ & $-$ & 77.0 & OOM in model construction \\
 & 4 nodes, tp8/pp1/cp4/ep32 & 63.0 & $-$ & 59.9 & $-$ & 59.9 & 75.7 & OOM in the teacher forward \\
 & 8 nodes, cp4, Hugging Face teacher & $-$ & $-$ & $-$ & $-$ & 96.7 & 96.6 & OOM in the teacher forward \\
 & 8 nodes, cp4, distributed-checkpoint teacher & 93.8 & $-$ & 81.5 & $-$ & 77.6 & 77.7 & OOM in the teacher forward \\
 & 8 nodes, tp2/pp4/cp8 & 71.0 & $-$ & 67.9 & $-$ & $-$ & $-$ & memory fits; NaN in the gradient norm of the first backward pass \\
 & 8 nodes, tp4/pp4/cp4 & 52.0 & $-$ & 48.9 & $-$ & $-$ & $-$ & memory fits; NaN in the forward loss of the first iteration \\
 & 8 nodes, tp4/pp4/cp4, all-gather dispatcher & 52.0 & 72.4 & 48.9 & $-$ & $-$ & $-$ & NaN in the gradient norm of the first backward pass \\
Qwen3-32B & 1 node, tp2/pp2/dp2, 4096 tokens & 51.0 & $-$ & 38.8 & 38.9 & $-$ & $-$ & trains 12 iterations \\
 & 1 node, same grid, \ballk{} & 52.2 & $-$ & 39.9 & 40.0 & $-$ & $-$ & trains 12 iterations \\
Llama & 6 nodes, tp16/pp1/cp1/dp3, global batch 126 & 112.4 & $-$ & 100.1 & $-$ & $-$ & $-$ & not run with \ballk{} \\
 & 6 nodes, tp16/dp3, \sqlk{} fused-W & $-$ & $-$ & $-$ & $-$ & $-$ & 80.1 & OOM in model construction \\
 & 8, 16, 24, 28, 32 nodes, tp$\le$8 (lowest-peak grid) & \multicolumn{2}{c}{89 to 102} & \multicolumn{2}{c}{76 to 90} & $-$ & $-$ & not run \\
 & 8 nodes, tp32/pp2/cp1/dp1 & 47.8 & $-$ & 35.6 & 36.7 & $-$ & $-$ & trains 5 iterations \\
 & 16 nodes, tp32/pp2/cp2/dp1 & 39.0 & $-$ & 26.8 & 27.5 & $-$ & $-$ & trains 200 iterations \\
\bottomrule
\end{tabular}
\caption{Colocated configurations for each pair (Megatron-Bridge ModelOpt; KD-only unless a row says \ballk{}, \sqlk{}, \sqlk{} fused, or all-gather; W: the model's cross-entropy fusion off), 32768-token sequences unless stated. Pred. and Meas. compare the model with the measurement of the same quantity in GiB: the device peak, the peak allocated by PyTorch (torch), or, at an OOM error, the bytes allocated plus the bytes requested (OOM point); $-$: not predicted or not measured. The model's MoE dispatch term is that of the all-to-all dispatcher, and the all-gather rows and the \sqlk{} rows, whose loss is not the modeled one, are not counted in Table~\ref{tab:accuracy}.}
\label{tab:brackets}
\end{table*}

\paragraph{Fused and fused-W losses.}
All colocated \sqlk{} runs of the primary pair evaluate $\mathcal{L}_{\mathrm{SQ}}$ of \eqref{eq:3losses} on the logits that ModelOpt's forward pass materializes; they differ in how $\mathrm{KD}_i$ and $\mathrm{CE}_i$ are computed. The fused and fused-W runs install the loss module of \system{} in ModelOpt, the kernel of Section~\ref{sec:loss} and the label pick of the trainer, so their results are a contribution of this work and not a competing baseline: the module removes the loss-side memory limit of the colocated stack, and the layout comparison is between 2 stacks that share it. (1)~\sqlk{}: the criterion forms the fp32 log-probabilities of the whole vocabulary and accumulates $\mathrm{KD}_i$ in tiles of 16384 columns, and the model forward computes $\mathrm{CE}_i$ with Megatron's fused vocabulary-parallel cross-entropy. (2)~\sqlk{} fused: one criterion computes both terms per token and returns $(1-\lambda_\iota)\,\mathrm{CE}_i+\lambda_\iota\,\mathrm{KD}_i$. It evaluates $\mathrm{KD}_i=\mathrm{lse}_s-\sum_v p_{t,v}z_{s,v}$ with the one-pass kernel of \system{} (an online running maximum and sum of exponentials over tiles, and the gradient $p_s-p_t$ of \eqref{eq:rowgrad} re-derived from the saved per-token statistics, so no fp32 tensor of size tokens${\times}V/\mathrm{tp}$ is saved for the backward pass; the kernel saves the logits that ModelOpt already holds and per-token statistics), and $\mathrm{CE}_i=\mathrm{lse}_s-z_{s,y_i}$ by the vocabulary-parallel label pick, in windows of 1024 rows under activation checkpointing. The cross-entropy of the model forward is still computed and logged, with zero gradient weight. Unlike \system{}, which forms the logits window by window from hidden states, this criterion reads the full materialized logits. (3)~\sqlk{} fused-W (W: the model's cross-entropy fusion off): (2) with that fusion turned off, so that the logged-only cross-entropy takes the unfused vocabulary-parallel path, which normalizes in place on its single fp32 copy of the logits. The 3-node fused run ran OOM allocating an fp32 tile of tokens${\times}V/\mathrm{tp}$ values (15.16\gib{} at cp1 with 32768 rows per rank) inside the fused cross-entropy (Appendix~\ref{app:primary}). \system{} evaluates no model cross-entropy: its forward step returns the pre-head hidden states, and the loss module computes the LM and KD terms on windows of 1024 to 4096 rows (Appendix~\ref{app:runs}), so the fusion option does not apply to it.

\paragraph{Colocated-side enablers.}
Keeping the colocated runs on Megatron-Bridge used 3 enablers: a text-extracted copy of the Qwen3.5 teacher (the release's vision wrapper cannot ride PP), a pre-converted sharded checkpoint for the Qwen3-235B teacher, and communicator pre-warming at training start; the converted and extracted teachers were not compared with the originals.

\begin{table*}[t]
\centering
\footnotesize
\setlength{\tabcolsep}{4pt}
\begin{tabular}{L{2.2cm}L{1.6cm}rrrrrrL{3.6cm}}
\toprule
Stack & Pair & Nodes & Steps & Tokens ($10^3$) & s/step & Tok/s/GPU & Peak & Setting \\
\midrule
ModelOpt, KD-only & Qwen3.5 & 4 & 200 & 4194 & 90.5 & 1450 & 63.4 & tp2/pp4/cp4/dp1 \\
ModelOpt, \ballk{} & Qwen3.5 & 4 & 200 & 4194 & 176.8 & 741 & 67.2 & tp2/pp4/cp4/dp1 \\
ModelOpt, \sqlk{} & Qwen3.5 & 4 & 200 & 4194 & 96.3 & 1361 & 59.9 & tp2/pp4/cp4/dp1 \\
ModelOpt, \sqlk{} fused & Qwen3.5 & 4 & 15 & 4194 & 100.6 & 1303 & 40.7 & tp2/pp4/cp4/dp1 \\
ModelOpt, \sqlk{} fused-W & Qwen3.5 & 3 & 15 & 4129 & 72.0 & \textbf{2389} & 64.1 & tp2/pp4/cp1/dp3 \\
ModelOpt, \sqlk{} fused-W & Qwen3.5 & 2 & 15 & 4194 & 146.8 & \textbf{1786} & 50.3 & tp2/pp4/cp2/dp1 \\
\system{} (ours) & Qwen3.5 & 2 & 15 & 4194 & 283.7 & 924 & 57.5 & 1+1 nodes, tp2/pp2/cp2/dp1 \\
KDFlow & Qwen3.5 & 4 & 12 & 4183 & 335 & 390 & 79.1 & sleep mode, chunked loss \\
NeMo-RL & Qwen3.5 & 4 & 12 & 1047 & 50.7 & 645 & 77.0 & batch 256 $\times$ 4096 \\
\system{} (ours) & Qwen3.5 & 3 & 200 & 4194 & 74.3 & 2352 & 39.2 & 2+1 nodes, tp2/pp2/cp2/dp2, CR $-$ \\
\system{} (ours) & Qwen3.5 & 4 & 200 & 4129 & 69.9 & \textbf{1846} & 33.0 & 3+1 nodes, tp2/pp2/cp2/dp3, CR $-$ \\
\midrule
ModelOpt, KD-only & Qwen3-32B & 1 & 12 & 524 & 18.4 & 3570 & 38.9 & tp2/pp2/dp2 \\
ModelOpt, \ballk{} & Qwen3-32B & 1 & 12 & 524 & 18.7 & 3500 & 40.0 & tp2/pp2/dp2 \\
KDFlow & Qwen3-32B & 1 & 12 & 515 & 9.8 & \textbf{6570} & 78.9 & sleep mode off \\
NeMo-RL & Qwen3-32B & 1 & 11 & 524 & 36.8 & 1780 & 77.5 & batch 128 $\times$ 4096 \\
\system{} (ours) & Qwen3-32B & 2 & 400 & 524 & 13.5 & 2426 & 25.1 & 1+1 nodes, tp2/pp2/dp2, CR $-$ \\
\bottomrule
\end{tabular}
\caption{Measured runs of the compared stacks: nodes of 8 GPUs, steps run, tokens per step as counted, median step time (a 9-step mean for KDFlow on Qwen3-32B; steps 3 to 15 for the 15-step run), tokens/s per GPU, and peak memory per GPU in GiB (torch-allocated for ModelOpt and \system{}, device memory for KDFlow and NeMo-RL, whose torch peaks were not logged; the producers' device peaks of \system{} are in Table~\ref{tab:runs}). KDFlow on Qwen3.5 used sleep mode, and its step includes the teacher offload cycle; NeMo-RL on Qwen3.5 ran at a quarter of the batch without the grouped-GEMM kernel. CR (in the Setting column): the parts of the cache release that were on (Table~\ref{tab:cachesplit}). Bold: highest tokens/s per GPU in each pair and node count; the runs differ in loss, sequence length, steps, tokens per step, and teacher precision.}
\label{tab:stacks}
\end{table*}

\paragraph{Token accounting.}
Our first KDFlow and NeMo-RL inputs held 1 document per sample, and FineWeb-Edu documents average 840 to 900 tokens, so a global batch of 128 would hold 0.11M tokens against 4.19M (Qwen3.5) or 0.52M (Qwen3-32B) for ModelOpt and \system{}, which concatenate documents into full-length sequences. We therefore ran KDFlow and NeMo-RL on the corpus packed into samples of the same length (documents concatenated, the remainder dropped; 32{,}768 tokens for Qwen3.5 and 4096 for Qwen3-32B, except NeMo-RL on Qwen3.5, which ran 4096-token samples), and counted tokens by re-tokenizing the packed samples. KDFlow receives the packed text as a response to a fixed prompt without a chat template, because its 2 chat-template paths disagreed by 2 tokens per sample; its samples render to 32{,}682 tokens on average on Qwen3.5 (at most 32{,}723, under its limit of 32{,}768) and 4026 on Qwen3-32B. NeMo-RL's samples average 4091 tokens. The tokens per step are in Table~\ref{tab:stacks}.

NeMo-RL's x-token recipe (its off-policy distillation) exports the teacher's full-vocabulary logits into a persistent per-rank buffer whose size is proportional to the tokens per step on the Qwen3.5 pair: 30\,GiB at 1.05M tokens and 121\,GiB at 4.19M, which does not fit an 80\,GiB GPU beside the models, as 2 failed 4.19M-token runs showed. The step of 1.05M tokens in Table~\ref{tab:stacks} was the largest that ran on Qwen3.5; 4.19M tokens ran OOM and 2.1M was not attempted. KDFlow at 4.18M tokens ran OOM without sleep mode (the student's forward, with the teacher's weights resident) and without the chunked loss (a teacher-head chunk wanted 15\,GiB with 8\,GiB free); with sleep mode and a chunk of 4096 tokens its device reading is at the limit (Table~\ref{tab:stacks}).

\paragraph{Systems comparison.}
Table~\ref{tab:stacks} lists the measured runs behind Table~\ref{tab:systems}. The global batch is 128 for every stack, except the 4-node \system{} run and the 3-node fused-W run on Qwen3.5, whose global batch is 126 so that $\mathrm{dp}{=}3$ divides it, and NeMo-RL on Qwen3.5, whose batch is 256 sequences of 4096 tokens. \system{} trains on Qwen3-32B with \sqlk{} on 1 trainer node and 1 producer node of 4 TP2 replicas, so 2 nodes against 1. KDFlow and NeMo-RL are summarized by the median step (NeMo-RL on Qwen3-32B in 2 runs with medians 36.9 and 36.7\,s), except KDFlow on Qwen3-32B, whose 1\,s step log gives the mean of its 9 steady steps (1 step of 36\,s excluded). KDFlow's optional sleep mode, which offloads the idle model, was off on Qwen3-32B, as in its example script, so its teacher's weights stay in GPU memory beside the student's, whereas NeMo-RL offloads the teacher after its forward pass. On Qwen3.5, NeMo-RL's expert layers ran without the grouped-GEMM kernel. In the rows on the teacher's output and the loss in Table~\ref{tab:systems}, columns are vocabulary entries held per TP rank, tokens are the microbatch's tokens on a GPU, and the KL is the forward KL between teacher and student next-token distributions; NeMo-RL adds a CE term (scale 0.1) and KDFlow weights KL and CE 0.5 each at learning rate $2\times10^{-5}$. NeMo-RL's same-tokenizer path takes it over the 8192 teacher columns with the largest logits in the microbatch (its default, of $V{\approx}152{,}000$ for Qwen3-32B and $248{,}000$ for Qwen3.5), renormalized within that subset, whereas the other 3 use all $V$ columns.

\section{Additional Results}

\subsection{Primary Pair}
\label{app:primary}
\paragraph{The \sqlk{} and \ballk{} losses.}
For row $i$ with label $y_i$ and one-hot vector $\bm{e}_{y_i}$, the label loss, the soft cross-entropy, and the KL divergence satisfy
\begin{equation}
\label{eq:klkd}
\begin{aligned}
\mathrm{CE}_i&=-\log p_{s,y_i},\\
\mathrm{KD}_i&=-\sum_v p_{t,v}\log p_{s,v},\\
H_i&=-\sum_v p_{t,v}\log p_{t,v},\\
\mathrm{KL}_i&=\sum_v p_{t,v}\log\tfrac{p_{t,v}}{p_{s,v}}=\mathrm{KD}_i-H_i,
\end{aligned}
\end{equation}
with $\mathcal{L}_{\mathrm{KL}}=\sum_i\mathrm{KL}_i$. The teacher entropy $H_i$ has no student parameters, so $\nabla\mathcal{L}_{\mathrm{KL}}=\nabla\mathcal{L}_{\mathrm{KD}}$, and the gradients with respect to the student logits $\bm{z}_s$ of row $i$ are
\begin{equation}
\label{eq:rowgrad}
\tfrac{\partial\mathrm{CE}_i}{\partial\bm{z}_s}=p_s-\bm{e}_{y_i},\qquad
\tfrac{\partial\mathrm{KD}_i}{\partial\bm{z}_s}=\tfrac{\partial\mathrm{KL}_i}{\partial\bm{z}_s}=p_s-p_t.
\end{equation}
With $\iota\in[0,1]$ the completed fraction of the iterations, the losses of the primary-pair runs are
\begin{equation}
\label{eq:3losses}
\begin{aligned}
\mathcal{L}_{\mathrm{SQ}}&=(1-\lambda_\iota)\,\mathcal{L}_{\mathrm{LM}}+\lambda_\iota\,\mathcal{L}_{\mathrm{KD}},\\
\lambda_\iota&=1-0.25\,\iota,\\
\mathcal{L}_{\mathrm{Bal}}&=\mathcal{L}_{\mathrm{LM}}+\gamma\,\mathcal{L}_{\mathrm{KL}},\quad \gamma=\tfrac{\mathcal{L}_{\mathrm{LM}}}{\mathcal{L}_{\mathrm{KL}}},\\
\mathcal{L}_{\mathrm{res}}&=\mathcal{L}_{\mathrm{LM}}+\gamma_{\mathrm{res}}\,\mathcal{L}_{\mathrm{KD}},\quad \gamma_{\mathrm{res}}=\tfrac{\mathcal{L}_{\mathrm{LM}}}{\mathcal{L}_{\mathrm{KD}}}.
\end{aligned}
\end{equation}
$\mathcal{L}_{\mathrm{SQ}}$ is our objective \eqref{eq:objective} (\sqlk{}): the disaggregated runs, except the rescaled run, and the colocated run of Table~\ref{tab:main} that installs it in ModelOpt with the cross-entropy form of the KD term. $\mathcal{L}_{\mathrm{Bal}}$ is ModelOpt's loss balancer (\ballk{}) of the colocated runs, and $\mathcal{L}_{\mathrm{res}}$ is the rescaled disaggregated run of Table~\ref{tab:parity}, which applies the rescaling to our cross-entropy value; $\gamma$ and $\gamma_{\mathrm{res}}$ are held constant in the backward pass. The rescaled losses are constant in value, $\mathcal{L}_{\mathrm{Bal}}=\mathcal{L}_{\mathrm{res}}=2\,\mathcal{L}_{\mathrm{LM}}$, so their totals carry no KD value, and the KD term acts through the gradient. By \eqref{eq:rowgrad}, the gradients with respect to $\bm{z}_s$ of row $i$ are
\begin{equation}
\label{eq:balgrad}
\begin{aligned}
\tfrac{\partial\mathcal{L}_{\mathrm{SQ}}}{\partial\bm{z}_s}&=p_s-(1-\lambda_\iota)\,\bm{e}_{y_i}-\lambda_\iota\,p_t,\\
\tfrac{\partial\mathcal{L}_{\mathrm{Bal}}}{\partial\bm{z}_s}&=(1+\gamma)\,p_s-\bm{e}_{y_i}-\gamma\,p_t,\\
\tfrac{\partial\mathcal{L}_{\mathrm{res}}}{\partial\bm{z}_s}&=(1+\gamma_{\mathrm{res}})\,p_s-\bm{e}_{y_i}-\gamma_{\mathrm{res}}\,p_t.
\end{aligned}
\end{equation}
Each gradient moves $p_s$ toward the target $(1-\omega)\,\bm{e}_{y_i}+\omega\,p_t$ with KD weight $\omega=\lambda_\iota$, $\gamma/(1{+}\gamma)$, and $\gamma_{\mathrm{res}}/(1{+}\gamma_{\mathrm{res}})$, at scale 1, $1{+}\gamma$, and $1{+}\gamma_{\mathrm{res}}$. Per token, $\mathcal{L}_{\mathrm{KL}}$ is 0.2 to 0.35 against $\mathcal{L}_{\mathrm{LM}}$ of 2.1 to 2.2, so $\gamma$ is 6.5 to 9.6 at the iterations of Table~\ref{tab:kdshare}, the weight of \ballk{} averages 0.887 over the run against 0.875 for $\lambda_\iota$, and the logged gradient norm at iteration 100 is 6.2 for \ballk{} against 0.29 for the colocated \sqlk{} run and 0.63 for the disaggregated run. The logged KD values differ by $H_i$, about 2 nats. For the rescaled run $\gamma_{\mathrm{res}}=0.96$ (2.2 to 2.3 against 2.2), a constant weight of 0.49 (Table~\ref{tab:kdshare}); its mixture is not ModelOpt's. At $\lambda_\iota=1$ the KD-only losses $\mathcal{L}_{\mathrm{KD}}$ and $\mathcal{L}_{\mathrm{KL}}$ have the same gradient.

\paragraph{Loss parity.}
Table~\ref{tab:parity}a gives the gap between the LM loss of a disaggregated run and that of the colocated \ballk{} run, over moving averages of $w$ iterations. We smooth over 17 iterations ($\pm 8$) because the loss changes by 0.014 to 0.018 nats between consecutive iterations, the size of the gap at $w{=}1$; the mean absolute gap shrinks with the window, and the signed mean stays at $+$0.001 for the 4-node run and rises from $+$0.005 to $+$0.008 for the rescaled run. Against the colocated \sqlk{} run the 4-node gap at $w{=}17$ is 0.008 (signed $+$0.005, largest 0.028). The rescaled run applies the colocated teacher precision (bf16), learning-rate floor ($10^{-5}$), and global batch (128), and rescales the KD term to the LM loss so that its total loss is 2$\times$ the LM loss, as in ModelOpt. It runs on 3 nodes (2+1) against 4 for the colocated run, at 153.0\,s per iteration against 176.8 (Tables~\ref{tab:rate} and~\ref{tab:main}), and differs from the colocated run in the grid, the kernels, the KD share (Table~\ref{tab:kdshare}), the document order, and CR all. Table~\ref{tab:parity}b gives the KD losses: the logged values differ by 2.00 nats on average (SD 0.02) from the ModelOpt runs and by 0.02 from the \sqlk{} run, and after subtracting each run's first 3 iterations they rise by 0.14 to 0.16 and end within 0.01 of their start.

\begin{table}[t]
\centering
\footnotesize
\setlength{\tabcolsep}{1.5pt}
\begin{tabular}{L{0.75cm}rr@{\hspace{5pt}}rr@{\hspace{5pt}}rr}
\toprule
\multicolumn{7}{l}{\emph{(a) LM loss gap to the colocated run}} \\
 & \multicolumn{4}{c}{\ballk{}} & \multicolumn{2}{c}{\sqlk{}} \\
\cmidrule(lr){2-5}\cmidrule(lr){6-7}
 & \multicolumn{2}{c}{Disagg., 4 nodes} & \multicolumn{2}{c}{Disagg., rescaled} & \multicolumn{2}{c}{Disagg., 4 nodes} \\
\cmidrule(lr){2-3}\cmidrule(lr){4-5}\cmidrule(lr){6-7}
$w$ & $|$gap$|$ & signed & $|$gap$|$ & signed & $|$gap$|$ & signed \\
\midrule
1 & 0.020 & $+$0.001 & 0.021 & $+$0.005 & 0.017 & $+$0.005 \\
5 & 0.014 & $+$0.001 & 0.014 & $+$0.005 & 0.010 & $+$0.005 \\
9 & 0.012 & $+$0.001 & 0.012 & $+$0.006 & 0.009 & $+$0.005 \\
17 & 0.012 & $+$0.001 & 0.010 & $+$0.007 & 0.008 & $+$0.005 \\
33 & 0.011 & $+$0.001 & 0.009 & $+$0.008 & 0.006 & $+$0.006 \\
Max & 0.056 & & 0.043 & & 0.028 & \\
\bottomrule
\end{tabular}

\vspace{3pt}
\begin{tabular}{L{3.0cm}rrrr}
\toprule
\multicolumn{5}{l}{\emph{(b) KD loss, first 3 iterations subtracted}} \\
Run & Rise & End & $|$gap$|$ & signed \\
\midrule
Disagg., 4 nodes & 0.16 & 0.00 & -- & -- \\
Coloc., KD-only & 0.15 & $+$0.01 & 0.011 & $-$0.010 \\
Coloc., \ballk{} & 0.14 & $+$0.01 & 0.019 & $-$0.015 \\
Coloc., \sqlk{} & 0.16 & $+$0.01 & 0.011 & $-$0.010 \\
\bottomrule
\end{tabular}
\caption{Loss parity on the primary pair. (a)~Mean absolute and mean signed gap between the LM loss of a disaggregated run and a colocated run (\ballk{} or \sqlk{}), over moving averages of $w$ iterations, at the iterations at which both have a full window (183 at $w{=}17$ for the 4-node run, 184 for the rescaled 2+1 run against the \ballk{} run). (b)~Largest rise of each run's KD loss, its value at the last full window (iteration 192), and the gap to the disaggregated curve (disaggregated minus colocated; mean absolute and mean signed over iterations 9 to 192). Max: the maximum absolute gap at $w{=}17$.}
\label{tab:parity}
\end{table}

\begin{table}[t]
\centering
\footnotesize
\setlength{\tabcolsep}{3pt}
\begin{tabular}{L{3.4cm}rrrr}
\toprule
Iteration & 20 & 50 & 100 & 200 \\
\midrule
Colocated \ballk{} & 0.88 & 0.87 & 0.89 & 0.91 \\
\sqlk{} (disagg. and coloc.) & 0.98 & 0.94 & 0.88 & 0.75 \\
Disaggregated, rescaled & 0.49 & 0.49 & 0.49 & 0.49 \\
Colocated KD-only & 1 & 1 & 1 & 1 \\
\bottomrule
\end{tabular}
\caption{Weight of the KD term in the loss of the primary-pair runs: $\lambda$ for the mixture, $\gamma/(1{+}\gamma)$ with $\gamma=\mathrm{LM}/\mathrm{KL}$ for \ballk{} and $\gamma_{\mathrm{res}}=\mathrm{LM}/\mathrm{KD}$ for the rescaled run.}
\label{tab:kdshare}
\end{table}

\paragraph{Fused loss in the colocated stack.}
The colocated runs \sqlk{} fused and fused-W (Appendix~\ref{app:stacks}; 15 iterations each, otherwise the \sqlk{} recipe) differ from \system{} in reading the materialized logits. Its torch peak is 40.7\gib{} and its sampled device peak 58.7\gib{}, against 59.9 and 76.8 for \sqlk{}, and its median iteration takes 100.6\,s over iterations 3 to 15, against 110.3\,s for \sqlk{} over the same iterations (96.3\,s over iterations 3 to 200). At iteration 1, with the same weights and data, the total loss of a rank is 18207.279 against 18207.285 (sums over the rank's tokens), and the model-forward LM sum agrees with the fused label pick to $2\times10^{-7}$ (the printed precision). The 15-iteration runs differ from the 200-iteration ones in $\lambda_\iota$ and in the learning rate, whose warmup covers iterations 1 to 14 of the 15, so only iteration 1 is compared, and the 15-iteration KD losses of different layouts are not comparable. At 3 nodes (tp2/pp4/cp1/dp3, global batch 126) the fused run runs OOM in the first forward, with 66.23\gib{} in use and a request of 15.16\gib{}: the fp32 tile that the model's own fused cross-entropy forms at 32768 rows per rank, to which the fused criterion gives no gradient weight. With that fusion off (fused-W) the 3-node run trains 15 iterations with the iteration-1 values of the 4-node runs (LM 2.0945, KD 2.1907), at 72.0\,s per iteration (4.13M tokens, 2389 tokens/s per GPU, against 74.3\,s and 2352 for \system{} on 3 nodes), with a torch peak of 64.1\gib{} and a sampled device peak of 78.9\gib{} against the 79.11\gib{} limit. At 2 nodes the fused-W run (tp2/pp4/cp2/dp1, global batch 128) trains 15 iterations at 146.8\,s per iteration, with a torch peak of 50.3\gib{} and a sampled device peak of 65.8\gib{}; with cp1 and dp2 it ran OOM after the first iteration, with 78.9\gib{} in use and a request of 0.5\gib{} at the windowed label pick. \system{} on 2 nodes (1 trainer node with tp2/pp2/cp2/dp1, 1 producer node with 4 TP2 replicas) trains 15 iterations at 283.7\,s per iteration with a trainer torch peak of 57.5\gib{}; the student is the bottleneck at dp1 and the producers are idle most of the time (duty 0.13, Table~\ref{tab:rate}). At 2 nodes the colocated fused-W run is therefore faster and holds less memory per GPU than \system{}, at 3 nodes the 2 runs have a similar speed and \system{} holds less memory per GPU, and \system{} is faster than the colocated runs with ModelOpt's loss. On the 3 larger pairs the fused-W cells at the node counts and student grids of \system{} run OOM during model construction, before the first forward pass: Qwen3-235B on 3 nodes (tp2/pp2/cp1/ep4, global batch 126; 74.9\gib{} in use and a request of 13.5\gib{}), Nemotron on 4 nodes (tp2/pp2/cp1/ep4; 79.09\gib{} in use while the teacher's expert layers are built), and Llama on 6 nodes (tp16/pp1/cp1/dp3, global batch 126; 66.0\gib{} in use and a request of 16.4\gib{} for the gradient buffer of the student). The colocated memory model of Appendix~\ref{app:memory} describes ModelOpt's loss and does not predict these runs.

\paragraph{Fused-W on 2 nodes.}
2 fp32 buffers with tokens${\times}V/\mathrm{tp}$ values each decide the peak of a colocated run of the primary pair, 15.16\gib{} each at 32768 rows per rank ($V/\mathrm{tp}{=}124160$) and 7.58\gib{} at cp2. (1)~The intermediates of ModelOpt's logits-KL loss, on the KD side: the stock criterion computes fp32 log-probabilities of this size and saves the student-side tensors for the backward pass. The fused criterion removes them: the kernel of \system{} saves no fp32 tensor of this size (only the logits that ModelOpt already holds and per-token scalars), and the label pick materializes 1 window of 1024 rows at a time. (2)~The model's own cross-entropy, on the student side: Megatron's fused vocabulary-parallel cross-entropy in the model forward forms 1 such tensor, although the fused criterion gives that loss zero gradient weight. Turning the fusion off (W) routes it through the unfused path, which normalizes in place on its single fp32 copy of the logits. Table~\ref{tab:fusedw} lists the 3-node cells, which are 3 different runs (the KD-only run has no LM term, hence no buffer (2)): the failing allocation moves from the loss to the model forward, and with the fusion off the fused run fits. At 2 nodes the same student grid with cp1 and dp2 holds 32768 rows per rank and ran OOM after the first iteration at a request of 0.5\gib{} (64.9\gib{} allocated, 78.9\gib{} in use), whereas cp2 and dp1 halves every buffer of this size, fits at 50.3\gib{} allocated (65.8\gib{} sampled device peak), and trains at 146.8\,s per iteration. \system{} on the same 2 nodes trains the student on 8 GPUs while the producer node waits for the trainer, and takes 283.7\,s (Table~\ref{tab:rate}).

\begin{table}[t]
\centering
\footnotesize
\setlength{\tabcolsep}{3pt}
\begin{tabular}{L{2.3cm}rrrL{1.7cm}}
\toprule
Loss & Alloc. & Req. & In use & Failed in \\
\midrule
KD-only, ModelOpt & 72.9 & 15.2 & 75.5 & the loss \\
\sqlk{} fused & 57.7 & 15.2 & 66.2 & model forward \\
\sqlk{} fused-W & 64.1 & $-$ & 78.9 & trains \\
\bottomrule
\end{tabular}
\caption{Colocated 3-node cell of the primary pair (tp2/pp4/cp1/dp3) at its first failing allocation or, for the run that trains, at its peak (GiB): PyTorch-allocated bytes (Alloc.), bytes requested by the failing allocation (Req.), and device memory in use (device peak for fused-W).}
\label{tab:fusedw}
\end{table}

\begin{table}[t]
\centering
\footnotesize
\setlength{\tabcolsep}{3pt}
\begin{tabular}{L{3.3cm}rrrr}
\toprule
Run & Nodes & s/it & GPU-s & Peak \\
\midrule
Colocated, cp4/dp1 & 4 & 96.3 & 3082 & 59.9 \\
Colocated fused, cp4/dp1 & 4 & 100.6$^\dagger$ & 3219 & 40.7 \\
Colocated fused-W, cp1/dp3 & 3 & 72.0$^\dagger$ & 1728 & 64.1 \\
Colocated fused-W, cp2/dp1 & 2 & 146.8$^\dagger$ & 2349 & 50.3 \\
Disaggregated, cp2/dp1 & 1+1 & 283.7$^\dagger$ & 4539 & 57.5 \\
Disaggregated, cp2/dp3 & 3+1 & 69.9 & 2237 & 33.0 \\
Disaggregated, cp2/dp2 & 2+1 & 74.3 & 1783 & 39.2 \\
\bottomrule
\end{tabular}
\caption{\sqlk{} runs of the primary pair at 32768 tokens (tp2; pp4 colocated, pp2 disaggregated): median s/it, GPU-seconds per iteration over all nodes (producers included), and torch-allocated peak in GiB (the trainer's for the disaggregated runs). $^\dagger$: 15 iterations; fused-W: the model's cross-entropy fusion off, global batch 126 at 3 nodes and 128 at 2 nodes.}
\label{tab:gpusec}
\end{table}

\paragraph{Loss parity under the offset.}
Figure~\ref{fig:parity}a plots the LM loss, the student's next-token cross-entropy, which every run defines in the same way whatever its training weights, so the curves measure one quantity. The weights change the trajectory, not the definition, so we measured how much: the 4-node disaggregated run (mean KD weight 0.875) and the rescaled 2+1 run (0.49, which differs from the 4-node disaggregated run in node split, global batch, learning-rate floor, teacher precision, and CR) lie 0.012 and 0.010 nats from the colocated run. Figure~\ref{fig:parity}b plots KD losses that are defined differently, cross-entropy against KL; their offset is the teacher's entropy (the logged difference is 2.00 nats; no run logged the entropy), which we remove by subtracting each run's first-3-iteration mean. The plotted disaggregated run is the mixture, whose weight is the closer to ModelOpt's. A comparison at equal share would apply the balancer to the KL value, which needs the teacher's entropy per token and is not part of these runs. The runs also differ in their data order: the colocated runs hold out 0.01\% of the documents for validation (a 9999:1:0 split) before shuffling the document order, and the disaggregated runs read the whole corpus with the same shuffle seed, so a reconstruction of both data pipelines from the corpus index gives different document sequences from the first iteration (Table~\ref{tab:dataorder}, seed-42 runs). The parity statistic therefore compares runs on the same corpus, not on the same samples.

\begin{table*}[t]
\centering
\footnotesize
\setlength{\tabcolsep}{4pt}
\begin{tabular}{L{4.0cm}rrcL{2.6cm}rrrrc}
\toprule
\multicolumn{3}{l}{\emph{(a) Profile of the colocated runs}} & & \multicolumn{6}{l}{\emph{(b) Streamed trainer on other grids}} \\
\cmidrule(r){1-3}\cmidrule(l){5-10}
 & KD-only & \ballk{} & & Trainer grid & Nodes & s/it & GPU-s & Relative & CR \\
\midrule
Loss stage, GPU-kernel share (\%) & 85.9 & 77.6 & & tp2/pp2/cp2/dp4 & 4+2 & 39.3 & 1886 & 1.00 & $-$ \\
Loss stage, NCCL collectives (s) & 2.3 & 18.3 & & tp2/pp4/cp2/dp2 & 4+2 & 50.8 & 2438 & 1.29 & $-$ \\
Loss stage, copy operators (s) & 7.1 & 23.4 & & tp2/pp2/cp4/dp2 & 4+2 & 50.0 & 2400 & 1.27 & $-$ \\
Loss stage, stream syncs (s) & 13.2 & 30.4 & & tp2/pp4/cp4/dp1 & 4+2 & 70.9 & 3403 & 1.80 & $-$ \\
Stage 1, GPU-kernel share (\%) & 98.0 & 98.5 & & tp2/pp2/cp2/dp3 & 3+1 & 65.0 & 2080 & $-$ & $-$ \\
\bottomrule
\end{tabular}

\vspace{4pt}
\begin{tabular}{L{2.2cm}L{2.7cm}rrrr}
\multicolumn{6}{l}{\emph{(c) KD-only at 8192-token sequences}} \\
\toprule
Stack & Grid & Nodes & s/it & GPU-s & Peak \\
\midrule
Colocated & tp2/pp4/cp4/dp1 & 4 & 137.2 & 4389 & 37.4 \\
Colocated & tp2/pp2/cp2/dp4 & 4 & 67.5 & 2161 & 65.7 \\
Streamed & tp2/pp2/cp2/dp3 & 3+1 & 114.7 & 3672 & 32.1 \\
Streamed & tp2/pp4/cp4/dp1 & 4+2 & 77.5 & 3720 & 30.8 \\
\bottomrule
\end{tabular}

\vspace{4pt}
\begin{tabular}{L{4.4cm}L{5.0cm}rr}
\multicolumn{4}{l}{\emph{(d) Colocated KD-only iteration at 32768 tokens (tp2/pp4/cp4/dp1)}} \\
\toprule
Run & Part of the iteration & s/it & GPU-s \\
\midrule
Colocated, seed 42 & all & 90.5 & 2894 \\
Colocated, seed 43 & all & 97.7 & 3126 \\
\quad forward-only & both forward passes and the loss & 77.9 & $-$ \\
\quad difference & training step minus evaluation step & 19.8 & $-$ \\
Streamed trainer, 4+2 nodes & all & 70.9 & 3403 \\
Producers, estimated & teacher forward & 16 & $-$ \\
\bottomrule
\end{tabular}

\caption{(a)~Profiles of iteration 3 of the colocated KD-only and \ballk{} runs (tp2/pp4/cp4/dp1), on 1 rank per pipeline stage. Times are seconds of traced time in the iteration, whose wall time on the loss stage is 111.3\,s (KD-only) and 120.0\,s (\ballk{}); iteration 3 is not steady (the median iterations take 90.5\,s and 176.8\,s, Table~\ref{tab:main}); the rows are overlapping event times, so differences and shares, not absolute times, are the result. (b)~Streamed trainer (\sqlk{}) on the Qwen3.5 pair at 4.19M tokens per iteration (4.13M in the 3+1 row), on 4 trainer and 2 producer nodes (Nodes), 12 iterations each; GPU-s: GPU-seconds per iteration over all nodes, producers included; Relative: s/it over the first row's; no row of the 4+2 nodes is limited by its producers, and the last row is a rerun of the 3+1 run of Table~\ref{tab:main} (12 iterations). (c)~KD-only runs at 8192-token sequences and 4.19M tokens per iteration (4.13M for the streamed 3+1 run), 12 iterations each, of the colocated stack (ModelOpt) and the streamed trainer with CR $-$; GPU-s: GPU-seconds per iteration over all nodes, producers included; Peak: torch-allocated GiB. The trainers of the streamed runs wait at most 0.004\,s per iteration for hidden states. (d)~Forward-only: an evaluation-mode run of the seed-42 configuration (median of steps 3 to 14), on different nodes, subtracted from the seed-43 training run; the teacher forward is estimated from the forward-only rate of one producer replica of the 3+1 run (32.8k tokens/s, Table~\ref{tab:rate}) over 8 replicas, the 4 nodes. CR: the parts of the cache release that were on (Table~\ref{tab:cachesplit}).}
\label{tab:speedup}
\end{table*}

\begin{table}[t]
\centering
\footnotesize
\setlength{\tabcolsep}{4pt}
\begin{tabular}{L{2.0cm}L{3.3cm}r}
\toprule
Runs & Documents read & Shared \\
\midrule
Colocated & 99.99\% of the corpus (9999:1:0), shuffled & 0 of 128 \\
Disaggregated & the whole corpus, same shuffle seed & $-$ \\
\bottomrule
\end{tabular}
\caption{Document order of the primary-pair runs. Shared: sequences of iteration 1 whose documents are the same in the colocated and the disaggregated run, reconstructed from the corpus index (seed-42 runs).}
\label{tab:dataorder}
\end{table}

\paragraph{Source of the speedup.}
On the 4-node Qwen3.5 pair the colocated \ballk{}, colocated \sqlk{}, colocated KD-only, and disaggregated runs cost 5656, 3082, 2894, and 2237 GPU-seconds per iteration (median iteration time times the 8 GPUs of each of 4 nodes, producer node included; the colocated \sqlk{} fused run costs 3219 and the 3-node fused-W run 1728, Table~\ref{tab:gpusec}). At seed 42 the 3419 GPU-second gap to the colocated \ballk{} run is the colocated LM term, 2762 GPU-seconds (81\%; 68\% with the seed-43 colocated runs and the 61.9\,s pace), plus the 657 (19\%) by which the colocated KD-only run still exceeds the disaggregated run, which also computes the LM term. With \sqlk{} the colocated LM term costs 188 GPU-seconds, and the gap to the disaggregated run is 845, of which 657 is the KD-only excess. Table~\ref{tab:speedup}a locates the added events: on the last pipeline stage, which evaluates the loss, the LM term adds NCCL collective events and a chain of small copy and conversion operators with their stream synchronizations, and the GPU-kernel share of traced time falls (85.9\% to 77.6\%), while stage 1 is unchanged (98.0\% and 98.5\%). The KD-only gap has a different source. The colocated stack executes the teacher forward pass and the student training step sequentially on the same 4 nodes, so the iteration time is their sum. The streamed \sqlk{} trainer, which executes only the student step on the same grid, takes 70.9\,s against the colocated 90.5\,s (seed 42) and 97.7\,s (seed 43, different nodes), 19.6 and 26.8\,s less, against a teacher-forward estimate of 16\,s from the producer rate (Table~\ref{tab:speedup}d); at 8192 tokens on the same grid the gap is 59.7\,s (Table~\ref{tab:speedup}c). In the 3+1 layout the producer node executes the teacher forward pass while the trainer nodes execute the student step. The producers' forward time is within 10\% of the iteration time (Table~\ref{tab:rate}), so the iteration time is the longer of the producers' forward time and the trainers' step time, not their sum.

Grid. With ModelOpt's loss, 3 trainer nodes cannot host the colocated grid, which needs 4 nodes per replica; with the fused-W loss the colocated stack runs the narrow grid tp2/pp4/cp1/dp3 on 3 nodes (Table~\ref{tab:gpusec}). The grid with the smaller cp is faster for the streamed trainer at fixed student GPUs (Table~\ref{tab:speedup}b) and keeps the trainers' step time within 10\% of the producers' forward time in the 3+1 runs with CR $-$ (Table~\ref{tab:rate}). At 8192 tokens the colocated stack fits the 4-node grids of Table~\ref{tab:speedup}c and needs fewer GPU-seconds than the split. This does not extend to 32768 tokens, where the model predicts ModelOpt's loss over the budget on the narrow grid (the measured cp2 OOM of Table~\ref{tab:brackets} is on pp 4), and the same 3+1 layout takes 65.0 to 69.9\,s for 4.13M tokens at 32768 tokens per sequence against 114.7\,s at 8192 (Tables~\ref{tab:speedup}c and~\ref{tab:rate}). The colocated runs hold the teacher in bf16 (Appendix~\ref{app:fp8}). Holding the routed experts in fp8 lowers the colocated stack's torch peak at its cp4 grid, but the KD loss and gradient norm of that run leave those of the bf16 run at iteration 3 (Table~\ref{tab:fp8colo}), and the model predicts the cp2 grid over the budget with fp8 experts as well (no fp8 run at cp2).

\begin{table}[t]
\centering
\footnotesize
\setlength{\tabcolsep}{4pt}
\begin{tabular}{L{2.0cm}rrrr}
\toprule
 & Seed 42 & Seed 43 & Change & $p$ \\
\midrule
Disagg. (ours), 4 nodes & 68.52 & \textbf{70.09} & $+1.57$ & $6{\times}10^{-6}$ \\
Disagg. (ours), 3 nodes & \textbf{70.14} & -- & -- & -- \\
Coloc., \ballk{} & 67.51 & 69.78 & $+2.26$ & $9{\times}10^{-13}$ \\
Coloc., KD-only & 67.97 & 69.22 & $+1.25$ & $10^{-4}$ \\
\midrule
\multicolumn{5}{l}{\emph{Between layouts, seed 42}} \\
4 nodes $-$ coloc. \ballk{} & -- & -- & $+1.01$ & 0.001 \\
4 nodes $-$ coloc. KD-only & -- & -- & $+0.56$ & 0.09 \\
3 nodes $-$ 4 nodes & -- & -- & $+1.62$ & $10^{-8}$ \\
\bottomrule
\end{tabular}
\caption{MMLU (\%) of the primary-pair students at iteration 200 for 2 seeds (the 3-node run only for seed 42), with the change over all 14{,}042 questions and its exact 2-sided McNemar $p$; the last block gives the differences between layouts at seed 42.}
\label{tab:seeds}
\end{table}

\paragraph{Seeds.}
Table~\ref{tab:seeds} gives the primary-pair runs with seed 43 for the data order and random state. The seed-43 runs score 1.25 to 2.26 points above the seed-42 runs, more than the seed-42 lead of the 4-node disaggregated student (+1.01, +0.56); the 3- and 4-node students, which differ in node split, global batch, and trainer grid, differ by 1.62 points, and 2 seeds do not establish a ranking of the layouts. The seed-43 colocated runs take 97.7\,s (KD-only) and 172.1\,s (\ballk{}) per iteration, on different nodes.

\begin{table}[t]
\centering
\footnotesize
\setlength{\tabcolsep}{4pt}
\begin{tabular}{L{3.0cm}rrr}
\toprule
CR & s/it & Effect & KD loss \\
\midrule
$-$ & 13.5 & -- & 2.42 \\
mb & 68.1 & $+54.6$ & 2.42 \\
iter & 14.9 & $+1.4$ & 2.41 \\
mb+iter & 69.3 & $+55.8$ & 2.42 \\
producer & 13.5 & $+0.0$ & 2.40 \\
all & 68.8 & $+55.3$ & 2.41 \\
\bottomrule
\end{tabular}
\caption{Parts of the cache release on Qwen3-32B to 4B (1+1 nodes, 4096-token sequences, 64 microbatches per DP rank), 30 iterations per row; s/it: median of iterations 8 to 30; Effect: s/it minus that of the row with no part on; KD loss: mean over iterations 8 to 30 (2.139 at iteration 1 in every row). CR: the parts of the cache release that were on: mb (the last pipeline stage releases after every microbatch), iter (once at each iteration boundary), producer (each producer replica after each of its first 400 slots), all for the 3, and $-$ for none; the release changes the pace, and the KD loss agrees within 0.02 nats (Appendix~\ref{app:cache}, Cache release).}
\label{tab:cachesplit}
\end{table}

\paragraph{Cache release.}
\label{app:cache}
The allocator's cached blocks can be returned to the driver (the cache release). The release has 3 parts, each with its own switch: mb, where the last pipeline stage of the trainer, which evaluates the loss, releases after every microbatch; iter, once at each iteration boundary; and producer, where each producer replica releases after each of its first 400 slots. 6 of the 8 cells over the 3 parts on the Qwen3-32B pair at 4096 tokens (Table~\ref{tab:cachesplit}, 1 run each) put 54.6 of the 55.3\,s on the per-microbatch part, 0.85\,s per microbatch; the iteration part adds 1.4\,s, and the producer part, active only in the first 12 iterations, adds 0.0\,s over iterations 8 to 30. The CR column of the run tables records which parts were on: $-$ for none, mb, iter, producer, combinations such as mb+iter, and all for the 3. It changes the pace; the mean KD loss over iterations 8 to 30 differs by at most 0.02 nats across the rows of Table~\ref{tab:cachesplit}. On the seed-43 configuration the release costs at least 0.77\,s per microbatch (42 per DP rank; the CR $-$ run is producer-limited; Table~\ref{tab:rate}). The 200-iteration seed-43 run has it on (Table~\ref{tab:rate}), and the seed-43 pace in Table~\ref{tab:main} is the run with CR $-$.

\subsection{Other Pairs}
\label{app:otherpairs}

\begin{table*}[t]
\centering
\footnotesize
\setlength{\tabcolsep}{3pt}
\begin{tabular}{L{2.7cm}L{1.45cm}rL{1.55cm}rrrL{2.2cm}}
\toprule
Run (30B student) & tp/pp/cp/ep & edp & Dispatcher & LM, it.\ 2 & Grad norm, it.\ 1 & Layer 0 & Outcome \\
\midrule
Torch forward & $-$ & $-$ & $-$ & 2.16 & $-$ & $-$ & reference \\
Disaggregated, 235B & 2/2/1/8 & 2 & flex & 2.11 & 2.2 & $-$ & stable, 20 it. \\
Disaggregated, 235B & 2/2/1/4 & 2 & flex & 2.11 & 1.7 & $-$ & stable, 200 it. \\
Colocated, 4B & 2/2/1/4 & 4 & all-to-all & 2.18 & 29 & 77.4 & finite, 2 it. \\
Colocated, 4B & 1/2/1/8 & 2 & all-to-all & 2.18 & 29 & $-$ & finite, 2 it. \\
Colocated, 4B & 2/2/1/8 & 2 & all-to-all & 2.18 & 29 & $-$ & finite, 2 it. \\
Colocated, 4B & 1/2/1/16 & 1 & all-to-all & 2.70 & $1.2\times10^5$ & 72.9 & above reference \\
Colocated, 4B & 2/1/1/16 & 2 & all-to-all & 3.15 to 3.22 & $1\times10^5$ to $1.4\times10^6$ & 110.9 & losses rise \\
Colocated, 4B & 2/1/1/16 & 2 & all-gather & 2.18 & 29 & $-$ & finite, 2 it. \\
Colocated, 4B, seq. experts & 2/1/1/16 & 2 & all-to-all & 3.11 & $7.8\times10^5$ & $-$ & above reference \\
Colocated, 4B, fp32 router & 2/1/1/16 & 2 & all-to-all & $-$ & $-$ & $-$ & NaN in the forward loss, first iteration \\
Colocated, 235B, 3 runs & 2/2/4/16 & 2 & all-to-all & 2.63 to 2.68 & $1.0$ to $1.7\times10^6$ & $-$ & Inf/NaN, 18 to 58 it. \\
Colocated, 235B & 2/2/4/16 & 2 & all-gather & 2.11 & 12 & $-$ & finite, 60 it. \\
\bottomrule
\end{tabular}
\caption{The Qwen3-30B-A3B student on different grids, 3 to 8 nodes (edp: expert data-parallel size, GPUs divided by ep$\cdot$pp; flex: the flex dispatcher with the DeepEP backend; seq. experts: sequential experts instead of grouped GEMM). Disaggregated: \system{} with the Qwen3-235B teacher; colocated: ModelOpt with the Qwen3-4B teacher (dense, same vocabulary; 4096-token sequences, learning rate $10^{-8}$ for the 2-iteration cells) or the Qwen3-235B teacher (32768 tokens). LM, it.\ 2: the student's LM loss at the second iteration, after 1 optimizer step (at learning rate $10^{-8}$ in the 4B-teacher cells); the Torch forward is the Hugging Face loss on the first iteration's 4096-token batch (the Megatron runs read 2.155 there; the second batch gives 2.178 in the healthy cells), and the 32768-token rows have none. Grad norm, it.\ 1: gradient norm of the first iteration; Layer 0: L2 norm of the first layer's output on the first batch (normalized by the data-parallel size); it.: iterations completed. The checked weight sums of the 3 runs with a layer-0 norm match the Hugging Face checkpoint to $5\times10^{-8}$ relative.}
\label{tab:moeprobe}
\end{table*}

\paragraph{Divergence control on Qwen3-235B.}
The 3 preserved colocated Qwen3-235B runs on 8 nodes (cp 4, ep 16, all-to-all) ended in Inf or NaN after 18, 44, and 58 completed iterations (Table~\ref{tab:moeprobe}). The table compares them with runs of the same student on other grids, with the dense Qwen3-4B teacher at 4096 tokens in the 4B cells. The student's LM loss at the second iteration is within 0.03 of the Hugging Face loss at ep 4 with tp 2 and at ep 8 with tp 1 and tp 2, and above it at ep 16 with tp 1 and tp 2; the gradient norm of the first iteration rises from 29 to $10^5$ and $10^6$. Replacing only the all-to-all token dispatcher by the all-gather dispatcher at the tp2/pp1/ep16 grid gives the loss (2.178) and gradient norm (29) of the ep-4 cell, whereas sequential experts and an fp32 router do not; on the 235B pair it gives a second-iteration loss of 2.11 and a first-iteration gradient norm of 12, against $1.0$ to $1.7\times10^{6}$ with all-to-all. The checked weights match the Hugging Face tensors, and the output of the first layer is 43\% above that of the healthy run at ep 16 in the tp2 run (6\% below at tp 1; layer 0 was not measured in the all-gather cell). At ep 16 the all-to-all dispatcher gives a wrong forward pass and the all-gather dispatcher on the same grid does not, and the tp1/pp2 (edp 1) and tp2/pp1 (edp 2) grids are both above the reference, as is the cp-4 grid with the 235B teacher. With the all-gather dispatcher the colocated 8-node run differs from the learning-rate-$4\times10^{-4}$ run that diverged after 18 iterations in the dispatcher and the learning-rate schedule (the 2 runs share their first 18 iterations, after which the 60-iteration schedule decays the learning rate sooner), and completes its 60 planned iterations with finite losses (the longest all-to-all run failed after 58) (Table~\ref{tab:moeprobe}; memory and pace in Table~\ref{tab:brackets}). Disaggregated on 4+4 nodes with the KD-only loss ($\lambda{=}1$, the LM term skipped), peak learning rate $4\times10^{-4}$, and the 200-iteration schedule, the pair completes the 60 iterations of a run stopped there with no NaN and no skipped iteration, at 105\,s per iteration (Table~\ref{tab:runs}). The disaggregated \sqlk{} runs of Table~\ref{tab:runs} on 4+4 and on 2+1 nodes run 200 iterations each without a NaN.

\begin{table}[t]
\centering
\footnotesize
\setlength{\tabcolsep}{1.5pt}
\begin{tabular}{L{1.0cm}L{1.0cm}L{1.15cm}rL{1.5cm}L{1.8cm}}
\toprule
Layout & Teacher & Grid & edp & Dispatcher & Outcome \\
\midrule
Disagg. & Ultra & 2/2/1/4 & 4 & flex & finite, 200 it. \\
Coloc. & Ultra & 2/4/8/16 & 1 & all-to-all & NaN, gradient norm \\
Coloc. & Ultra & 4/4/4/16 & 1 & all-to-all & NaN, forward loss \\
Coloc. & Ultra & 4/4/4/16 & 1 & all-gather & NaN, gradient norm \\
Coloc. & Nano & 4/4/4/16 & 1 & all-gather & NaN, gradient norm \\
Coloc. & Nano & 4/4/1/16 & 1 & all-gather & NaN, gradient norm \\
Coloc. & Nano & 4/4/4/4 & 4 & all-gather & finite, 3 it. \\
\bottomrule
\end{tabular}
\caption{Nemotron-3 Nano student on 8 nodes (Disagg.: \system{} on 4+4 nodes) with the Ultra teacher or, for the student alone, the Nano checkpoint as its own teacher. Grid: tp/pp/cp/ep; edp: expert data-parallel size. NaN: the first value that is not finite, in the first iteration; finite: iterations completed. The colocated runs share data, global batch 128, 32768 tokens, and learning rate $2\times10^{-5}$, except the 2/4/8/16 row ($4\times10^{-4}$).}
\label{tab:nemoprobe}
\end{table}

\paragraph{NaN on the Nemotron pair.}
Table~\ref{tab:nemoprobe} lists the colocated Nemotron runs next to the disaggregated one. The student gives NaN in its first iteration in each of the 5 colocated ep-16 cells (tp 2 or 4, pp 4, cp 1, 4, or 8), with the all-to-all and all-gather dispatchers and the Ultra and Nano teachers, and stays finite for 3 iterations at ep 4 (edp 4) with the Nano teacher, the all-gather dispatcher, and cp 4; every ep-16 cell has edp 1. The all-to-all dispatcher gives NaN already in the forward loss at tp4/cp4. The Nano-teacher rows run the Nano student with its own initial checkpoint as teacher, so ModelOpt's KL stays in the loss and only the Ultra teacher is removed.

\begin{table*}[t]
\centering
\footnotesize
\setlength{\tabcolsep}{3pt}
\begin{tabular}{L{3.0cm}rrrrcr@{\hspace{8pt}}L{1.8cm}rrrc}
\toprule
\multicolumn{7}{l}{\emph{(a) Llama}} & \multicolumn{5}{l}{\emph{(b) Qwen3-32B to 4B}} \\
\cmidrule(r){1-7}\cmidrule(l){8-12}
Layout (nodes) & Iter. & s/it & GPU-s & Peak & CR & MMLU & Seq.\ $\times$ batch & Microbatches & s/it & Tokens/s/GPU & CR \\
\midrule
Colocated (8) & 5 & 484 & 31.0k & 36.7 & -- & $-$ & $4096\times128$ & 64 & 13.5 & 2426 & $-$ \\
Colocated (16), training & 200 & 301 & 38.5k & 27.5 & -- & 76.98 & $4096\times128$ & 64 & 69.6 & 471 & all \\
\quad pace check & 20 & 305 & 39.0k & 27.4 & -- & $-$ & $8192\times64$ & 32 & 11.9 & 2755 & $-$ \\
Disagg. (ours, 4+2) & 5 & 818 & 39.3k & 68.6 & all & $-$ & $16384\times32$ & 16 & 13.5 & 2427 & $-$ \\
Disagg. (ours, 8+8), training & 400 & 186 & 23.8k & 50.3 & $-$ & 76.41 &  &  &  &  &  \\
\quad pace check & 20 & 235 & 30.1k & 50.3 & all & $-$ &  &  &  &  &  \\
Teacher &  &  &  &  &  & 85.21 &  &  &  &  &  \\
Initial student &  &  &  &  &  & 78.17 &  &  &  &  &  \\
\bottomrule
\end{tabular}
\caption{(a)~Llama runs at 4.19M tokens per iteration: iterations, median iteration time, GPU-seconds per iteration, and peak memory in GiB (colocated: per GPU; disaggregated: the trainer; torch-allocated). Colocated grids are tp32/pp2/cp1/dp1 (8 nodes) and tp32/pp2/cp2/dp1 (16); disaggregated trainers are tp8/pp2/cp2 with dp1 (4+2) and dp2 (8+8); MMLU (\%) is of the trained students, the initial student, and the teacher. The 4+2 median is over iterations 2 to 5. A training run is the full-length run; a pace check is 20 iterations of the same layout, run separately from the training run. (b)~Disaggregated Qwen3-32B to 4B at 0.52M tokens per iteration on 2 nodes (tp2/pp2/dp2, 4 producer replicas); Seq.\ $\times$ batch: sequence length times global batch; Microbatches: per DP replica. CR: the parts of the cache release that were on (Table~\ref{tab:cachesplit}). A dash marks the colocated runs, which do not use it.}
\label{tab:otherpairs}
\end{table*}

\paragraph{Llama pair.}
Table~\ref{tab:otherpairs}a lists the Llama runs. At 16 nodes each layout has a training run and a pace check. The pace of each layout was measured again in a 20-iteration pace check of the same layout. The colocated pace moves by 1.3\% between the 2. The pace check ran with CR all and is 48.7\,s slower, 0.76\,s per microbatch over 64 microbatches per DP rank; the release costs 0.77\,s per microbatch on the primary pair, and no Llama run isolates it. The speedup at 16 nodes is 1.6$\times$ from the training runs, with CR $-$, and 1.3$\times$ from the pace checks. At 8 nodes the colocated layout needs 31.0k GPU-seconds per iteration, 3\% more than the disaggregated pace check at 16 nodes (30.1k) and 30\% more than its training run (23.8k); the disaggregated layout runs on 6 nodes (39.3k GPU-seconds, 5 iterations) where the colocated layout needs 8. Both students score within 1.8 points of the initial student on MMLU (MMLU in Table~\ref{tab:otherpairs}a).

\paragraph{Qwen3-32B pair.}
On Qwen3-32B to 4B the disaggregated run with CR $-$ is 5.1 to 5.2$\times$ faster per GPU than with CR all (Tables~\ref{tab:cachesplit} and~\ref{tab:otherpairs}b), and stays below ModelOpt and KDFlow (Table~\ref{tab:systems}). Table~\ref{tab:otherpairs}b holds the tokens per step fixed at 0.52M and varies the sequence length: with CR $-$ the iteration time is 13.5, 11.9, and 13.5\,s at 64, 32, and 16 microbatches per replica, whereas with CR all the 4096-token configuration pays 0.86\,s per microbatch (0.85\,s for the per-microbatch part alone; Table~\ref{tab:cachesplit}). In the 4096-token CR $-$ run the producers' duty is 0.93 (Table~\ref{tab:rate}).

\section{Artifacts, Licenses, and Compute}
\label{app:artifacts}

The models are Qwen3.5, Qwen3, and Qwen3-4B-Base (Apache License 2.0), Nemotron 3 Nano (NVIDIA Nemotron Open Model License) and Ultra (OpenMDW License 1.1), and Llama 3.1 (Llama 3.1 Community License). The corpus is FineWeb-Edu (ODC-By 1.0) and the benchmark is MMLU (MIT License). The software is Megatron-Bridge, Megatron-Core, Model Optimizer, NeMo-RL, Transformer Engine, and the evaluation tool (Apache License 2.0), and KDFlow and Triton (MIT License). All were used for research distillation, as their licenses permit, and the trained students are not released. FineWeb-Edu is a filtered web corpus and MMLU a multiple-choice benchmark; neither is annotated by us. An anonymized codebase accompanies the submission: it holds the code, the run configuration behind each table, the memory and traffic models with their tests, the scripts that produce the tables and figure data, and a standalone MMLU protocol; it holds no run data, so the measured tables are reproduced by re-running the configurations on the stated hardware. The runs reported in this paper used about 15{,}000 GPU-hours on NVIDIA H100 80\,GB GPUs, including the baselines and the MMLU evaluations but not failed attempts or memory and NaN probes.

\section{Glossary of Symbols}
\label{app:glossary}

Subscripts $t$ and $s$ denote the teacher and the student, and $x\in\{t,s\}$ stands for either. Superscripts $(q)$, $(j)$, and $(q,j)$ denote the cp rank, the tp rank, and the cell of both. Indices start at 0. Quantities in bytes carry the unit explicitly, and GiB is $2^{30}$ bytes.

\newcommand{\gl}[2]{\noindent\parbox[t]{2.2cm}{\raggedright #1}\hspace{4pt}\parbox[t]{4.9cm}{\raggedright #2}\par\vspace{1.5pt}}
\newcommand{\glhead}[1]{\par\medskip\noindent{\raggedright\textbf{#1}\par}\nobreak\smallskip}
\begingroup\footnotesize\renewcommand{\arraystretch}{1.12}

\glhead{Models and sizes (\S\ref{sec:estimation}).}
\gl{$D_x$, $E_x$}{Dense and routed-expert parameters of model $x$; $P_t=D_t{+}E_t$}
\gl{$\ell_x$}{Layers of model $x$}
\gl{$\ell^{\mathrm{att}}_x$, $\ell^{\mathrm{lin}}_x$, $\ell^{\mathrm{moe}}_x$}{Full-attention, linear-attention, and MoE layers of model $x$}
\gl{$d_x$ ($\dt$, $\ds$)}{Hidden width of model $x$}
\gl{$V$, $\Vloc$}{Vocabulary size and its share $V/\mathrm{tp}$ per tp rank}
\gl{$h$, $h_{\mathrm{kv}}$}{Attention heads and key-value heads}
\gl{$d_{\mathrm{head}}$}{Width of one attention head}
\gl{$k_{\mathrm{e}}$}{Routed experts per token}
\gl{$w$}{Per-token width of the linear-attention context exchange (Table~\ref{tab:traffic}); in Table~\ref{tab:parity}, the width in iterations of the moving average}

\glhead{Parallelism and layout (\S\ref{sec:estimation}).}
\gl{tp, pp, cp, dp, ep, etp}{Tensor-, pipeline-, context-, data-, expert-, and expert-tensor-parallel degrees}
\gl{$\mathrm{edp}$}{Data-parallel size of the routed experts, $\mathrm{dp}\,\mathrm{cp}\,\mathrm{tp}/\mathrm{ep}$ at $\mathrm{etp}{=}1$}
\gl{$g$, $\mathcal{G}$, $g^\star$}{A parallelism grid, the set of legal grids, and the node-minimal grid that fits}
\gl{$N(g)$, $N$}{Nodes of grid $g$; number of nodes, each of 8 GPUs}
\gl{$N_{\mathrm{stu}}$, $N_{\mathrm{prod}}$}{Trainer nodes and producer nodes}
\gl{$R$}{Producer replicas of the teacher}
\gl{$\mathrm{tp}_p$, $\mathrm{ep}_p$}{Tensor and expert span of one producer replica}
\gl{$M$}{Per-GPU memory budget, 79.11\,GiB}
\gl{$L$}{Tokens (rows) per sequence; $L=\mathrm{cp}\cdot r$}
\gl{gbs}{Global batch size in sequences}
\gl{$N_{\mathrm{tok}}$}{Tokens per iteration, $\mathrm{gbs}\cdot L$}
\gl{$\tau$}{Tokens handled by one GPU per iteration, $N_{\mathrm{tok}}/(\mathrm{dp}\,\mathrm{cp})$}
\gl{$n_{\mathrm{mb}}$}{Microbatches per DP replica, $\mathrm{gbs}/\mathrm{dp}$}
\gl{$\mu$}{Microbatches in flight across the pipeline}
\gl{$b$}{Bytes per activation element, 2}
\gl{$n$}{Degree of a parallel dimension (Table~\ref{tab:traffic}); number of parts in Lemma~\ref{lem:merge}}
\gl{$\Delta$, $n\Delta$}{Distance between the ranks of a dimension in Megatron's rank order and the number of ranks it spans; the dimension uses NVLink if $n\Delta\le8$, the GPUs of a node, and InfiniBand otherwise}

\glhead{Memory model (\S\ref{sec:memory}).}
\gl{$S$}{Student weights, gradients, and optimizer state per GPU}
\gl{$P_{\mathrm{loc}}$}{Student parameters per GPU}
\gl{$A_{\mathrm{col}}$, $A_{\mathrm{str}}$}{Activation memory in the colocated stack and activation and temporary memory of the streamed trainer (both drawn as $A$ in Figure~\ref{fig:overview})}
\gl{$T_{\mathrm{col}}$}{Teacher weight shard per GPU in the colocated layout}
\gl{$\kappa_{\mathrm{col}}$}{Colocated logits buffers}
\gl{$\kappa_{\mathrm{str}}$}{Memory retained by the streamed loss, the windowed fused loss on streamed hidden states}
\gl{$\sigma_{\mathrm{col}}$, $\sigma_{\mathrm{str}}$}{Non-PyTorch memory (CUDA context, NCCL, allocator) of the colocated stack and the trainer}
\gl{$A_{\mathrm{gc}}$, $A_{\mathrm{bwd}}$, $A_{\mathrm{moe}}$}{Trainer reference-cycle garbage, loss-backward and recompute temporary memory, and expert-layer block}\gl{$\ell_{\mathrm{rc}}$}{Layers recomputed at once}
\gl{$\varphi_p$, $\omega_p$, $\beta_e$}{Producer forward-pass activation and workspace memory, its NCCL buffers, and bytes per routed-expert parameter (1 for fp8, 2 for bf16)}
\gl{$M_{\mathrm{col}}$, $M_{\mathrm{tr}}$, $M_{\mathrm{pr}}$}{Predicted peak memory per GPU of the colocated stack, the trainer, and the producer}
\gl{$\rho_{\mathrm{rep}}$}{Forward-only token rate of one producer replica (measurement per pair in Appendix~\ref{app:pipeline})}
\gl{$T_{\mathrm{iter}}$}{Iteration time}

\glhead{Traffic model (\S\ref{sec:movement}).}
\gl{$\Phi^{\mathrm{tp}}_x$, $\Phi^{\mathrm{pp}}_x$, $\Phi^{\mathrm{cp}}_x$, $\Phi^{\mathrm{ep}}_x$}{Bytes sent per GPU and iteration by TP collectives, pipeline activation transfers (per stage boundary and direction), CP exchanges, and expert dispatch and combine collectives of model $x$}
\gl{$\Phi^{\mathrm{dp}}_s$}{Bytes sent per GPU by the student's gradient reduction and parameter gather}
\gl{$\Phi^{\mathrm{xfer}}$}{Bytes per last-stage rank and iteration of the teacher-to-student transfer}
\gl{$\nu_t$, $\nu_s$}{All-reduces per layer: 2 for the teacher's forward pass, 6 for the student's forward, recompute, and backward passes}

\glhead{Algorithm (\S\ref{sec:algorithm}).}
\gl{$q$, $j$, $(q,j)$}{cp rank, tp rank, and the cell (last-stage rank) with both}
\gl{$i$}{Row (token) index}
\gl{$r$}{Rows per slot, one cp rank's share of a sequence}
\gl{$K$, $C$, $k$}{Windows per slot, rows per window ($r=K\cdot C$), and window index; $K$ also denotes top-$K$ in Sections~\ref{sec:related} and~\ref{sec:experiments}}
\gl{$u$, $v$}{Entry index within a tp rank's vocabulary shard and column index over the full vocabulary}
\gl{$\eta$, $\eta_0$}{Slot index in the producer's stream and the start slot of a resume}
\gl{$n_{\mathrm{groups}}$}{Receiver groups, one per DP replica ($\mathrm{dp}$)}
\gl{$\delta$}{Iterations the producer may run ahead of the trainer}
\gl{$\bm{H}_x^{(q)}$, $\bm{H}_{x,k}^{(q)}$, $\bm{H}_x$}{Final hidden states of model $x$ for the rows of cp rank $q$ (the slot), its window $k$, and all $L$ rows}
\gl{$\bm{W}_x$, $\bm{W}_x^{(j)}$}{lm\_head of model $x$ and its column shard on tp rank $j$}
\gl{$\bm{Z}_x^{(q,j)}$, $\bm{Z}_{x,k}^{(q,j)}$}{Block of the logits of model $x$ held by cell $(q,j)$, and its window $k$}
\gl{$\bm{z}_x$, $\bm{z}_x^{(j)}$, $z_{x,v}$}{Logits of one row, their shard on tp rank $j$, and the entry of column $v$}

\glhead{Loss and statistics (\S\ref{sec:loss}).}
\gl{$p_x$, $p_{x,v}$}{Softmax of $\bm{z}_x$ over the $V$ columns and its entry for column $v$}
\gl{$y_i$, $\bm{e}_{y_i}$}{Label of row $i$ and its one-hot vector}
\gl{$\mathcal{L}$, $\lambda$}{Training loss and the weight of the KD term}
\gl{$\gamma$, $\gamma_{\mathrm{res}}$}{Ratio of the LM loss to the KL loss (ModelOpt's balancer) or to the KD loss (rescaled run) in a batch}
\gl{$\mathcal{L}_{\mathrm{SQ}}$, $\mathcal{L}_{\mathrm{Bal}}$, $\mathcal{L}_{\mathrm{res}}$}{The \sqlk{}, \ballk{}, and rescaled-run losses \eqref{eq:3losses}; $\lambda_\iota$: KD weight at the completed fraction $\iota$ of the iterations}
\gl{$\mathcal{L}_{\mathrm{LM}}$, $\mathcal{L}_{\mathrm{KD}}$}{Language-modeling and KD terms of the loss, $\sum_i\mathrm{CE}_i$ and $\sum_i\mathrm{KD}_i$}
\gl{$\mathrm{KL}_i$, $\bm{e}_{y_i}$, $\omega$}{KL divergence of row $i$, the one-hot label vector, and the KD weight of the target $(1-\omega)\bm{e}_{y_i}+\omega p_t$}
\gl{$\mathcal{L}_{\mathrm{KL}}$, $H_i$}{$\sum_i\mathrm{KL}(p_t\,\|\,p_s)$ and the teacher's entropy of row $i$}
\gl{$\mathrm{CE}_i$, $\mathrm{KD}_i$}{Label loss and distillation loss (soft cross-entropy) of row $i$; $\mathrm{KL}(p_t\,\|\,p_s)$ differs from $\mathrm{KD}_i$ by the teacher's entropy}
\gl{$\mathrm{lse}_s$}{Log-sum-exp of the student's logits of a row}
\gl{$m_x$, $a_x$, $c$}{Row maximum of $\bm{z}_x$, sum of $e^{z-m}$, and teacher-weighted sum of student logits}
\gl{$m_x^{(j)}$, $a_x^{(j)}$, $c^{(j)}$}{The same statistics over the shard of tp rank $j$}
\gl{$m_{x,u}$, $a_{x,u}$, $c_u$}{Scan state after entry $u$ of a shard}
\gl{$m_{x,\mathcal{S}}$, $a_{x,\mathcal{S}}$, $c_{\mathcal{S}}$}{The statistics over a set $\mathcal{S}$ of columns}
\gl{$\mathcal{S}$, $\mathcal{P}_l$, $\mathcal{U}_u$}{A set of columns, its $n$ disjoint parts (index $l$), and the tiles read in one scan ($U$ tiles)}
\gl{$|\mathcal{S}|$}{Number of columns in $\mathcal{S}$}
\gl{$\bm{G}^{(j)}$}{Gradient of $\mathcal{L}$ with respect to the window's student logits on tp rank $j$}
\gl{$\mathbb{R}$, $\top$, $e$}{Real numbers, matrix transpose, and the base of the natural logarithm}
\gl{$[v{=}y_i]$}{1 if $v=y_i$ and 0 otherwise}
\gl{$p$}{$p$-value of the paired McNemar test (Table~\ref{tab:seeds}), not the softmax}

\endgroup

\end{document}